\documentclass{article}
\usepackage[utf8]{inputenc}
\usepackage[a4paper, total={6.5in, 10in}]{geometry}

\usepackage{amsopn}
\usepackage{amsthm}
\usepackage{amsmath}
\usepackage{amssymb}
\usepackage{authblk}

\usepackage{bbm}
\usepackage{bm,nicefrac}
\usepackage{braket,amsfonts,amssymb}

\usepackage{caption}
\usepackage{comment}

\usepackage{makecell}

\usepackage{enumitem}
\usepackage{extarrows}
\usepackage{graphicx}
\usepackage{mathrsfs}  
\usepackage[mathscr]{euscript}

\usepackage{subcaption}
\usepackage{xcolor}
\usepackage[linktoc=all]{hyperref}
\usepackage{titlesec}
\hypersetup{
    colorlinks,
    citecolor=blue,
    filecolor=black,
    linkcolor=blue,
    urlcolor=black
}

\def\c{c}
\def\d{d}

\def\D{D}
\def\M{M}
\def\R{\mathbb{R}}

\def\V{V}
\def\W{W}
\def\y{y}
\def\z{z}
\def\tT{\top}
\def\Dground{\Lambda}
\def\Wstar{\widehat{W}}

\def\tr{\text{tr}}

\newcommand{\cc}{C}
\renewcommand{\l}{\ell}
\renewcommand{\L}{L}

\def\tdisc{{k}}
\def\tcont{{t}}

\newcommand{\diag}{\mathrm{diag}}
\renewcommand{\epsilon}{\varepsilon}

\newcommand{\NN}{{\mathbb{N}}}
\newcommand{\RR}{{\mathbb{R}}}

\renewcommand{\Re}{{\operatorname{Re}}}
\renewcommand{\Im}{{\operatorname{Im}}}
\newcommand{\Id}{{\operatorname{Id}}}
\newcommand{\abs}[2]{\left| #1 \right|_{#2}}
\newcommand{\norm}[2]{\left\| #1 \right\|_{#2}}

\newcommand{\round}[1]{\left( #1 \right)}

\newcommand{\curly}[1]{\left\{ #1 \right\}}
\newcommand{\red}[1]{ {\color{red} #1}}

\newcommand{\HR}[1]{{#1}}
\newcommand{\revision}[1]{{ #1}}

\providecommand{\keywords}[1]{\textbf{\textit{Keywords --- }} #1}

\newtheorem{theorem}{Theorem}[section]

\newtheorem{definition}[theorem]{Definition}
\newtheorem{lemma}[theorem]{Lemma}

\theoremstyle{remark}
\newtheorem{remark}[theorem]{Remark}
\newtheorem{example}[theorem]{Example}

\title{Gradient Descent for Deep Matrix Factorization:\\ Dynamics and Implicit Bias towards Low Rank}
\author[1]{Hung-Hsu Chou}
\author[2]{Carsten Gieshoff}
\author[1]{Johannes Maly}
\author[1]{Holger Rauhut}
\affil[1]{Department of Mathematics, LMU Munich, Germany}
\affil[2]{Chair for Mathematics of Information Processing, RWTH Aachen University, Germany}
\date{\today}

\begin{document}

\maketitle

\begin{abstract}
    %\textcolor{red}{(Ed: proposed modification)
    In deep learning, it is common to use more network parameters than training points. 
    In such scenario of over-parameterization, there are usually multiple networks that achieve zero training error so that the training algorithm induces an implicit bias on the computed solution.
    %However, empirically gradient-based algorithms tend to prefer solutions with certain properties, a phenomenon known as implicit bias and a 
    In practice, (stochastic) gradient descent tends to prefer solutions which generalize well, which provides a
    possible explanation of the success of deep learning. In this paper we analyze 
    %the implicit bias of 
    the dynamics of gradient descent in the simplified setting of linear networks and of an estimation problem.
    \HR{Although we are not in an overparameterized scenario, our analysis nevertheless provides insights into the phenomenon of implicit bias.} In fact, we derive a rigorous analysis of the dynamics of vanilla gradient descent, and characterize the dynamical convergence of the spectrum. We are able to accurately locate time intervals where the effective rank of the iterates is close to the effective rank of a low-rank projection of the ground-truth matrix. In practice, those intervals can be used as criteria for early stopping if a certain regularity is desired. We also provide empirical evidence for implicit bias in more general scenarios, such as matrix sensing and random initialization. \HR{This suggests that deep learning prefers trajectories whose complexity (measured in terms of effective rank) is monotonically increasing, which we believe is a fundamental concept for the theoretical understanding of deep learning.}
\end{abstract}


%\tableofcontents

%-------------------------------------------------------------

\section{Introduction}
\label{sec:Introduction}

%We live in the age of data acquisition and data processing, and the amount of data of which we speak is massive. For example, in \cite{Krizhevsky2012}, computations involved 1.2 million high-resolution images and 60 million parameters to tune.
%and 650,000 neurons.} % Haven't defined neurons yet;)
%It comes as no surprise that we need to rely on machines automatically handling, respectively learning from data if we ever wish to profit from all the information we collect. For decades, a popular approach have been neural networks which approximate complex functions by concatenating simple base units (called \emph{neurons}) in \emph{layers}. Though of merely theoretical interest in the beginning, cf.\ universal approximation theorems in \cite{Baldi1989NeuralNA,Hornik1991}, due to technological improvements in hardware and impressive successes in public challenges, such as Alpha Go \cite{silver2017mastering}, there has been a revival in recent years. 

Deep learning has become the standard machine learning technology in recent years, celebrating breakthroughs in many areas ranging from face recognition over medical imaging to autonomous driving. Despite all its successes it is still mysterious why deep learning works so well. Often deep neural networks have significantly more parameters than the number of examples used in training. As studied systematically via numerical experiments, for instance in \cite{keskar2017,neyshabur2015,zhang2017}, 
(stochastic) gradient descent usually results in zero training error so that the resulting neural networks interpolate the training samples exactly. Nevertheless and somewhat surprisingly, the trained deep networks generalize very well, although classical statistics would suggest that one is in a regime of overfitting.
It was remarked already in \cite{zhang2017} that the employed optimization algorithms induce an implicit bias towards certain solutions. Apparently, those solutions often behave very nicely in realistic situations. Providing an understanding of the nature of such implicit bias 
seems to be a key task for the development and understanding of deep learning in general.
While a general theory for the implicit bias in deep learning seems presently out of reach, first theoretical works \cite{arora2019implicit,geyer2019implicit,gunasekar2019implicit,gunasekar2017implicit,razin2020implicit,soudry2018implicit} concentrate on \textit{linear} networks and suggest that (stochastic) gradient descent converges to a linear network, i.e., a linear function described by a matrix, which is of \textit{low rank}. Nevertheless, even for the linear case, the settings considered in these works are rather restrictive and many open questions remain.

In this article, we consider a matrix estimation problem, see \eqref{eq:MF} below, where the desired matrix $W~\in~\mathbb{R}^{n \times n}$ (describing the linear network) 
is factorized into $N$ matrices $W_1,\hdots,W_N \in \mathbb{R}^{n \times n}$. 
We provide a precise analysis of the dynamics of the gradient descent/flow for each of the individual matrices $W_j$, which are initialized by $\alpha I$, for some suitable small constant $\alpha > 0$. We show that for desired rank $L$ and explicitly given time-intervals the time-dependent product matrix approximates very well the best 
rank $L$ approximation of the ground truth matrix.
In this way, we provide another indication of the role of low rank in the study of implicit bias in deep learning. We are convinced that our proof methods can be extended to study this problem in more general scenarios.

\medskip

%but nevertheless, the
%A core component of awarded methods
%, in regression-type problems and classification \cite{Krizhevsky2012}, 
%is the large number of layers in neural networks. 

Let us describe the setting of our article in more details. The archetype of deep networks is the so-called \emph{feed-forward} neural network of $N$ layers, $N \ge 1$, defined as $h \colon \mathbb{R}^{n_0} \rightarrow \mathbb{R}^{n_N}$ with
\begin{align} \label{eq:NN}
    h(x) = g_N \circ \cdots \circ g_1(x),
\end{align}
where the functions $g_k \colon \mathbb{R}^{n_{k-1}} \rightarrow \mathbb{R}^{n_k}$ are of the form
\begin{align*}
    g_k(x) = \sigma( \W_k x + b_k )
\end{align*}
and model the $N$ layers. They are determined by \emph{weight matrices} $\W_k \in \mathbb{R}^{n_{k}\times n_{k-1}}$, \emph{bias terms} $b_k \in \mathbb{R}^{n_k}$ (which we for simplicity set to $b_k = 0$ from here on), and an \emph{activation function} $\sigma \colon \mathbb{R} \rightarrow \mathbb{R}$ that is in general non-linear and acting component-wise. To guarantee good performance, both the weight matrices and the bias terms are design parameters that are optimized according to the concrete data and application. 
%Neural networks (and especially deep networks) are hence massively over-parametrized models. 
For a fixed activation function $\sigma$ and given training data $(x_i,y_i)_{i=1}^M$, where $x_i \in \mathbb{R}^{n_0}$ and $y_i \in \mathbb{R}^{n_N}$ model training input resp.\ output, the standard approach in supervised learning is to solve
\begin{align} \label{eq:LossMin}
    \min_{\W_1,...,\W_N} \frac{1}{M} \sum_{i = 1}^M \mathcal{L}( h(x_i),y_i),
\end{align}
where $\mathcal{L} \colon \mathbb{R}^{n_N} \times \mathbb{R}^{n_N} \rightarrow \mathbb{R}_+ = \{ x \in \mathbb{R}: x \geq 0\}$ is called \emph{loss function} and should exhibit "distance like" properties (a popular choice being the squared $\ell_2$-norm). The number of free parameters commonly dominates the number of training samples. In this case, we speak of \textit{overparametrization}.

The optimization problem \eqref{eq:LossMin} is normally solved via variants of (stochastic) gradient descent (using back propagation).
%leading
%to decent solutions.
%applying plain vanilla gradient descent via back-propagation often leads to decent solutions.
As already described above, this often leads to decent solutions, even in the overparametrized setting and a recent research hypothesis claims
an implicit bias of gradient descent towards low-complexity solutions -- although no regularization term is added to \eqref{eq:LossMin}.
%Recent research claims in this context an \emph{implicit bias/regularization} of gradient descent towards low-complexity solutions, i.e., though no regularization is added to the procedure, the trajectory of gradient descent stays close to low-dimensional subspaces such that 
It was observed in \cite{Heckel2018,Ulyanov2017} that early stopping may produce beneficial solutions 
%to \eqref{eq:LossMin} 
while omitting unfavorable local and global minima. 
%In certain cases it was even observed that in the limit gradient descent approaches minimizers of low complexity \cite{arora2019implicit}, without much theoretical understanding of this phenomenon. 
Since \eqref{eq:LossMin} is hard to analyze in general due to the non-linear structure of \eqref{eq:NN}, recent theoretical works concentrate on the simplified case of \emph{linear neural networks} \cite{arora2018optimization,arora2019implicit,bah2019learning,bartlett2018gradient,Gissin2019Implicit,li2018algorithmic}, where $\sigma(x) = x$ and $b_k=0$, i.e., \eqref{eq:NN} becomes 
\begin{align} \label{eq:NNlinear}
    h_\text{Linear}(x) = \W_N \cdots \W_1 x.
\end{align}
Choosing $\mathcal{L}$ to be the quadratic loss, equation \eqref{eq:LossMin} then takes the more accessible shape
\begin{align}\label{eq:LossMinLinear}
    \min_{\W_1,...,\W_N} \frac{1}{M} \sum_{i = 1}^M \norm{\W_N \cdots \W_1 x_i - y_i}{2}^2.
\end{align}
This problem reminds of classical formulations of principal component analysis (PCA) in \cite{wold1987principal}. Indeed, if the samples $x_1,\dots,x_M$ span the input space $\mathbb{R}^{n_0}$, any minimizer of \eqref{eq:LossMinLinear} has to satisfy $\W_N \cdots \W_1 = YX^\dagger =: \Wstar \in \mathbb{R}^{n_N \times n_0}$, where $X \in \R^{n_0 \times M}$ and $Y \in \R^{n_N \times M}$ have columns $x_i$ and $y_i$, for $i \in [M]$, and $X^\dagger$ denotes the pseudo-inverse of $X$. The set of minimizers of $\eqref{eq:LossMinLinear}$ thus equals the set of minimizers of the matrix factorization problem
\begin{align}\label{eq:MF}
    \min_{\W_1,...,\W_N} \big\| \W_N \cdots \W_1 - \Wstar \big\|_{F}^2
\end{align}
for $\widehat{W} = YX^\dagger$, where $\norm{\cdot}{F}$ denotes the Frobenius norm. Note that, although the trajectories of gradient descent on \eqref{eq:LossMinLinear} and \eqref{eq:MF} do not agree in general, the iterates are biased towards similar low-dimensional structures. 
%\revision{Note that similar reduction can also be made from commutative matrix sensing problem of the form}
%\begin{align}\label{eq:LossMatrixSensing}
 %   \min_{\W_1,...,\W_N} \frac{1}{M} \sum_{i = 1}^M \norm{\langle \W_N \cdots \W_1, A_i \rangle - y_i}{2}^2,\,
  %  A_i \text{ commute}.
%\end{align}
Indeed, the recent works \cite{arora2019implicit,Bach2019implicit,gunasekar2017implicit,li2018algorithmic} point out that, when applied to factorized problems like \eqref{eq:LossMinLinear}, \eqref{eq:MF}, or matrix sensing and initialized close to zero, gradient descent exhibits an implicit bias towards solutions of low-rank resp.\ low nuclear norm even if the architecture is not imposing any rank constraints on the solution, i.e., $n_0 = n_1 = \cdots = n_N = n \in \mathbb{N}$. It was, moreover, observed that higher order factorizations $\W = \W_N \cdots \W_1$ can strengthen the regularizing effect \cite{arora2019implicit}. We are convinced that understanding this phenomenon in detail will also help to gain theoretical insights into the success of gradient descent in the general training model \eqref{eq:LossMin}. Let us mention that the trajectories of gradient descent on \eqref{eq:LossMinLinear} and \eqref{eq:MF} do not agree in general. We focus on \eqref{eq:MF} in this work to allow a tighter theoretical analysis. For similar reasons, we make the simplifying assumption that all the matrices $W_j$ are initialized by a small constant times the identity matrix although random initializations are often used in practice. We plan to extend this case to more general initializations in future work.

Hence, we are interested in analyzing the discrete dynamics defined by
\begin{align} 
    \W_{j}(\tdisc+1)
    &= \W_{j}(\tdisc) - \eta \nabla_{\W_{j}}%(\tdisc)} 
    \mathcal{L}(\W(\tdisc))
    \label{eq:MatrixFactorization_GradientDescent}\\
    \W_{j}(0)
    &= \alpha \W_0, \label{eq:MatrixFactorization_IdenticalInitialization}
\end{align}
for $j \in [N] := \curly{1,\dots,N}$, where 
\begin{align} \label{eq:L_MatrixFactorization}
    \mathcal{L}(\W) = \frac{1}{2} \big\| \W - \Wstar \big\|_{F}^2 
    \quad
    \text{with}
    \quad
    \nabla_\W \mathcal{L} (\W) = (\W - \Wstar),
\end{align}
$\eta > 0$ is the step-size, $\W_0$ is some initialization matrix, and $\alpha > 0$ is assumed to be small. Note that the gradient of $\mathcal{L}$ with respect to a single factor $\W_j$ is given by
\begin{align} \label{eq:L_gradient}
    \nabla_{\W_j} \mathcal{L}(\W) = (\W_N\cdots \W_{j+1})^{\tT} \nabla_{\W} \mathcal{L}(\W) (\W_{j-1}\cdots \W_{1})^{\tT}.
\end{align}

\subsection{Contribution and Outline}
\label{subsec:ContributionAndOutlines}

%-------------------------------------------------------------

\begin{table}[]
    \centering
    \begin{tabular}{| c | c | c |}
    \hline 
    & Shallow ($N = 2)$ & Deep ($N>2$) \\
    \hline
    Gradient descent & \cite{Bach2019implicit}, \cite{li2018algorithmic} & \textbf{Our work} \\ 
    \hline
    Gradient flow & \cite{gunasekar2017implicit,gunasekar2019implicit} & \textbf{Our work}, \cite{arora2018optimization,arora2019implicit} \\
    \hline
    \end{tabular}
    \caption{Comparison of our work with related literature in terms of which method is analyzed and whether the analysis is restricted to shallow factorizations.}
    \label{tab:Table1}
\end{table}

\begin{table}[]
    \centering
    % \begin{tabular}{| c | c | c | c | c | c | c |}
    % \hline 
    % & This work & \cite{arora2018optimization,arora2019implicit} & \cite{Bach2019implicit} & \cite{gunasekar2017implicit,gunasekar2019implicit} & \cite{geyer2019implicit} & \cite{li2018algorithmic} \\
    % \hline
    % Initialization & Identical & Random & & & &\\ 
    % \hline
    % Ground-truth & Symmetric & PSD & & & & \\
    % \hline
    % Model & \eqref{eq:MF} & \eqref{eq:LossMinLinear} & & & & \\
    % \hline
    % \end{tabular}
    
    \begin{tabular}{| c | c | c | c | }
    \hline 
    & Initialization & Ground-truth & Model\\
    \hline
    \textbf{Our work} & identity & \textbf{symmetric} & \eqref{eq:MF}\\
    \hline
    \cite{Bach2019implicit} & aligned & \makecell{$XX^T$ and $X^TYY^TX$ \\ almost commuting}& \eqref{eq:LossMinLinear}\\
    \hline
    \cite{arora2018optimization,arora2019implicit} & identity & \makecell{symmetric \\ PSD} & \makecell{matrix sensing with \\ commuting measurements} \\
    \hline
    \cite{gunasekar2017implicit,gunasekar2019implicit} & identity & \makecell{symmetric \\ PSD} & \makecell{matrix sensing with \\ commuting measurements}\\
    % \hline
    % \cite{geyer2019implicit} & & &\\
    \hline
    \cite{li2018algorithmic} & orthogonal & \makecell{symmetric \\ PSD} & matrix sensing\\
    \hline
    \end{tabular}
    \caption{Comparison of our work with related literature in terms of considered initialization, ground-truth, and training model. Aligned initialization under the training model \eqref{eq:LossMinLinear} means that $W_N(0)$ has the same singular vectors as $YX^T$ and $W_1(0)$ has the same singular vectors as $XX^T$. Orthogonal initialization means that all $W_j(0)$ are orthogonal matrices. Note that PSD stands for positive semi-definite ground-truths.}
    \label{tab:Table2}
\end{table}

In this paper, we provide a precise analysis of the dynamics in \eqref{eq:MatrixFactorization_GradientDescent}-\eqref{eq:MatrixFactorization_IdenticalInitialization} and their underlying continuous counterparts, for $\W_0 = I$ and symmetric ground-truths $\Wstar \in \mathbb{R}^{n\times n}$. \HR{
Our analysis extends/generalizes the one started in
\cite{arora2019implicit,Gissin2019Implicit,li2018algorithmic},
which consider more sophisticated training models than \eqref{eq:MF} but basically restrict themselves to the gradient flow and gradient descent with $N=2$, cf.\ Tables~\ref{tab:Table1}~and~\ref{tab:Table2}.
In particular, our contributions are the following.}
%Motivated by and in accordance with \cite{arora2019implicit,Gissin2019Implicit,li2018algorithmic}, we extend the theory to a more general setting. In particular,
\begin{enumerate}
    \item 
    \HR{We analyze the gradient descent dynamics for all $N\in\mathbb{N}$ (rather than only for $N=2$ or rather than restricting only to the gradient flow.)}
    %instead of gradient flow or the particular case where $N=2$;
    \item We derive a sharp upper bound on the step size $\eta$ \HR{ensuring convergence}, for all choices of $\alpha$ and $N$.
    \item We only assume symmetry of the ground-truth, which is considerably weaker \HR{than the often used assumption of positive semidefiniteness (PSD)}.
    \item We prove that negative eigenvalues can only be recovered under a suitable perturbation of the initialization. This has not yet been discussed in the standard setting of identity initialization since the prior works \cite{arora2019implicit,Gissin2019Implicit,li2018algorithmic} only consider positive semi-definite ground-truths, cf.\ Table \ref{tab:Table2}.
\end{enumerate}
Our main results consist of three central observations
(note that we \HR{always} use explicit constants in the statements).\\

\textbf{I. Recovery of positive eigenvalues:} Initializing with $\W_0 = I$, as done in \cite{li2018algorithmic} in a matrix sensing framework, the dynamics in \eqref{eq:MatrixFactorization_GradientDescent}-\eqref{eq:MatrixFactorization_IdenticalInitialization} can solely recover non-negative eigenvalues of $\Wstar$. In addition, the following theorem provides a quantitative analysis that shows how fast eigenvalues of $\Wstar$ are approximated depending on their magnitude and other model parameters like $N$, $\eta$, and $\alpha$.
\begin{theorem}
\label{theorem:IdenticalInitialization_ContinuousNonAsymptoticRate}
%\textcolor{red}{(Should we first state the theorem and prove it later, or move it to the later parts so that we can state and prove at the same time?)}
Let $N \geq 2$, $\alpha>0$, $\epsilon>0$. Let $\Wstar = V\Lambda V^\tT \in \mathbb{R}^{n\times n}$ be an eigendecomposition of the symmetric matrix $\Wstar$ with $\Lambda = \diag(\lambda_i : i \in [n])$. Consider $W_1(\tdisc), \hdots, W_N(\tdisc) \in \mathbb{R}^{n\times n}$ and $W(\tdisc) = W_N(\tdisc) \cdots W_1(\tdisc)$ defined by the gradient descent iteration \eqref{eq:MatrixFactorization_GradientDescent} with loss function \eqref{eq:L_MatrixFactorization} and identical initialization \eqref{eq:MatrixFactorization_IdenticalInitialization}. Define $\M = \max(\alpha,\|\Wstar\|^{\frac{1}{N}})$.  Suppose 
\begin{equation}\label{eta:cond:mthm}
    %\alpha \in \left(0, \min_{i \in [n]} 
    %|\lambda_i|^{\frac{1}{N}}\right),
    %\quad
    0 < \eta < \frac{1}{(3N-2) M^{2N-2}}.
    %\in \left(0, \frac{1}{(3N-2) \max_{i \in [n]} |\lambda_i|^{2 - \frac{2}{N}}}\right).
\end{equation}
Then $\W(\tdisc)$ converges to $ V \Lambda_+ V^\tT$ as $k \to \infty$, where $\Lambda_+ = \diag(\max\{\lambda_{i},0\}: i \in [n])$. Moreover, the error $E(\tdisc) = (V^\tT \W(\tdisc) V  - \Lambda_+)$ is a diagonal matrix, whose entries satisfy
\begin{equation}\label{eq:error-estimate}
|E_{ii}(\tdisc)| \leq 
\begin{cases} 
\epsilon  N \lambda_i^{1-\frac{1}{N}} & \mbox{ if } \lambda_i >0,\\
% \epsilon  N\max(\epsilon, \lambda_i^{\frac{1}{N}})^{N-1}  & \mbox{ if } \lambda_i >0,\\
\epsilon^N & \mbox{ if } \lambda_i \leq 0.
\end{cases}
\end{equation}
for all $k \geq T^\Id_N(\lambda_i, \epsilon, \alpha, \eta)$, where $T^\Id_N$ is defined in \eqref{def:TId} below.
%
% Moreover, for fixed $\epsilon>0$ and all 
% $$k \geq T^\Id_N(\lambda_i, \epsilon, \alpha, \eta),$$
% where $T^\Id_N$ is defined in \eqref{def:TId} below and
% the error $E(\tdisc) := (V^\tT \W(\tdisc) V  - \Lambda_+)$ is a diagonal matrix, whose entries satisfy 
% \begin{equation}\label{eq:error-estimate}
% |E_{ii}(\tdisc)| \leq 
% \begin{cases} 
% |\lambda_i - (\lambda_i^{\frac{1}{N}}- \epsilon)^N| \leq \epsilon N \lambda_i^{1 - \frac{1}{N}} & \mbox{ if } \lambda_i > 0,\\
% \epsilon^N & \mbox{ if } \lambda_i \leq 0.
% \end{cases}
% %\quad \mbox{ for  all } k \geq T^\Id_N(\lambda_i, \epsilon, \alpha, \eta)
% \end{equation}
% %is bounded by \textcolor{red}{(possibly more compact expression)} $|\lambda_i - (\lambda_i^{\frac{1}{N}}-\epsilon)^N|$ after $T_{\text{Id}}(\lambda_i,\epsilon)$ iterations, where $T_{\text{Id}}$ has been defined in Theorem \ref{prop:EffectiveRankDiscretePSD}.
\end{theorem}
The proof of Theorem \ref{theorem:IdenticalInitialization_ContinuousNonAsymptoticRate} is given in Section \ref{subsec:IdenticalInitialization}.
\begin{remark}
\label{rem:main-thm}
    The exact form of $T_N^\text{Id}$ involves some additional notation, which will be introduced later, see
    \eqref{def:TId}. Let us nevertheless give some simplified approximate expressions in the relevant case that $N \geq 3$, $0 < \epsilon^N \ll \lambda_i \leq \lambda_1$ and $0 < \alpha^N \ll \lambda_i$, so that the initial matrix $W(0) = \alpha^N \Id$ has small enough spectral norm compared to the $i$-th eigenvalue of the ground truth, which in turn is larger than the desired accuracy $\epsilon^N$. In this case $M = \|\widehat{W}\|^{\frac{1}{N}} = \lambda_1^{\frac{1}{N}}$ and we assume that
    \[
    \eta = \frac{\kappa}{N \lambda_1^{2-\frac{2}{N}}}
    \]
    for some $\kappa \leq \frac{1}{3}$ so that \eqref{eta:cond:mthm} is satisfied. The quantity $T^\Id_N$ then takes the form (the interested reader is referred to the more detailed derivation and discussion in Appendix \ref{sec:Remark_main-thm})
    \begin{equation}\label{TID-form}
    T^\text{Id}_N(\lambda_i,\epsilon,\alpha,\eta) 
    = \frac{1}{\kappa(N-2)} \left(\frac{\lambda_1}{\alpha^N}\right)^{1-\frac{2}{N}} \frac{\lambda_1}{\lambda_i}
    + \mathcal{O} \big( \log(1/\epsilon) \big),
    \end{equation}
    where we ignore the additional term $s_N(\lambda_i,\alpha)$ appearing in \eqref{def:TId} since it is neglectable for small $\kappa$ and large $N$ (and probably resembles a proof artefact). As \eqref{TID-form} shows, $T^\text{Id}_N$ consists of two main terms. One depends on the ratio between $\lambda_i$ and $\lambda_1$ and one on the desired accuracy $\epsilon$. $T^\text{Id}_N$ increases for smaller $\lambda_i$ and $\epsilon$. 
\end{remark}
\begin{figure}
\begin{subfigure}[c]{0.45\textwidth}
\includegraphics[width=\textwidth]{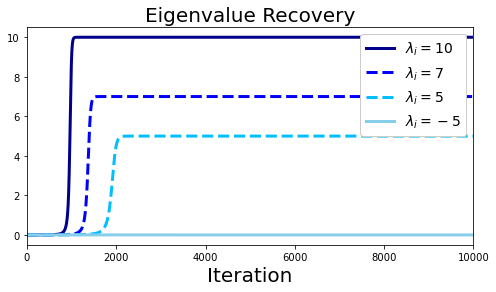}
\subcaption{Convergence of eigenvalues with identical initialization.}
\label{fig:IdenticalVSPerturbedA}
\end{subfigure} \quad
\begin{subfigure}[c]{0.45\textwidth}
\includegraphics[width=\textwidth]{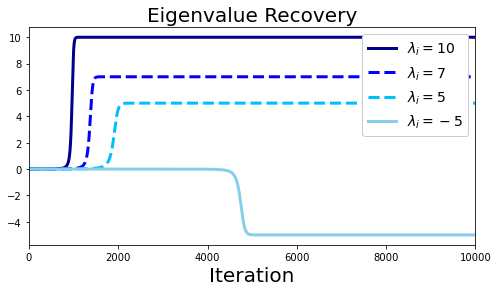}
\subcaption{Convergence of eigenvalues with perturbed identical initialization.}
\label{fig:IdenticalVSPerturbedB}
\end{subfigure}
\caption{Depicted is a comparison of gradient descent dynamics for different types of initialization. Depicted is the evolution of the entries $(V^\tT \W V)_{ii}$, where $V$ diagonalizes $\Wstar$. In the case of identical initialization, only positive eigenvalues can be recovered, while in the case of perturbed identical initialization, the full spectrum is recovered. In addition, the recovering rate is positively correlated with the magnitude of eigenvalues. Here $N=3$, $\alpha = 10^{-1}$, $\beta = \frac{\alpha}{2}$, and $\eta = 10^{-3}$. }
\label{fig:IdenticalVSPerturbed}
\end{figure}
\begin{figure}
\begin{subfigure}[c]{0.45\textwidth}
\includegraphics[width=\textwidth]{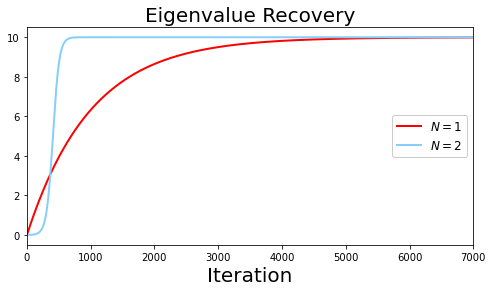}
\subcaption{The dynamics between $N=1$ and $N=2$ are quite different, in the sense that the transition for $N=2$ is more concentrated.}
\label{fig:IdenticalwithDifferentNA}
\end{subfigure} \quad
\begin{subfigure}[c]{0.45\textwidth}
\includegraphics[width=\textwidth]{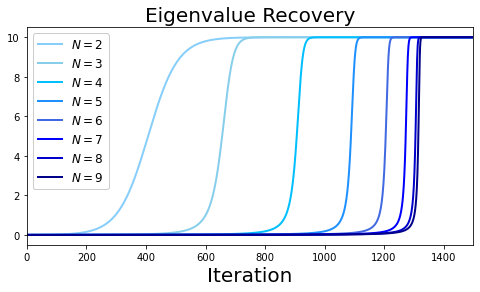}
\subcaption{For $N\geq2$, the dynamics are more similar. As $N$ increases, the recovery rate becomes slower initially, but sharper once the transition occurs.}
\label{fig:IdenticalwithDifferentNB}
\end{subfigure}
\caption{Depicted is a comparison of eigenvalue recovery rate for different number of layers $N$ with identical initialization. To make the comparison fair we fixed $\alpha^N= 3\cdot 10^{-3}$. Here $\lambda=10$ and $\eta = 10^{-3}$. }
\label{fig:IdenticalwithDifferentN}
\end{figure}
%

%\paragraph{} The above observation that gradient descent only factorizes positive semi-definite matrices in a loss-less way leads to the question whether this is a general short-coming of the algorithm when applied in the symmetric setting or due to the specific initialization. Our next result provides an answer by changing

\textbf{II. Recovery of arbitrary eigenvalues:} Perturbing the initialization slightly, i.e., setting 
\begin{equation}\label{eq:PerturbedInitialization_MatrixInitialization}
    \W_j(0)
    = \begin{cases}
    (\alpha - \beta)I & \text{ if }j=1\\
    \alpha I & \text{ otherwise},
    \end{cases}
\end{equation}
allows the dynamics in \eqref{eq:MatrixFactorization_GradientDescent} to recover the whole spectrum of $\Wstar$, cf.\ Figure \ref{fig:IdenticalVSPerturbedB}. This suggests that the spectral cut-off observed in Theorem \ref{theorem:IdenticalInitialization_ContinuousNonAsymptoticRate} is a pathological case and thus hardly observed in practice. As before, the following theorem also quantifies the approximation rates of eigenvalues of $\Wstar$ in terms of their magnitude, their sign, $N$, $\eta$, $\alpha$, and $\beta$.
\begin{theorem}\label{theorem:PerturbedInitialization_NonAsymptoticConvergence}
Let $N\geq 2$ and let $\Wstar = V\Lambda V^\tT \in \mathbb{R}^{n \times n}$ be an eigendecomposition of the symmetric matrix $\Wstar$ with $\Lambda = \diag(\lambda_i : i \in [n])$. Consider $W_1(\tdisc), \hdots, W_N(\tdisc) \in \mathbb{R}^{n\times n}$ and $W(\tdisc) = W_N(\tdisc) \cdots W_1(\tdisc)$ defined by the gradient descent iteration defined in
\eqref{eq:MatrixFactorization_GradientDescent} with loss function \eqref{eq:L_MatrixFactorization} and perturbed identical initialization \eqref{eq:PerturbedInitialization_MatrixInitialization}, where $0 < \frac{\beta}{\c-1} < \alpha $, and $\c\in(1,2)$ is the maximal real solution to the polynomial equation $1=(\c-1)\c^{N-1}$. Define $\M = \max(\alpha,\|\Wstar\|^{\frac{1}{N}})$. If
\begin{equation}\label{cond:stepsize:thm}
    0 < \eta <  %\min\left(\frac{2^{\frac{1}{N}}-1}{2},\frac{1}{9N}\right)\frac{1}{(\c M)^{2N-2}} = \textcolor{blue}{
    \frac{1}{9N (\c M)^{2N-2}},
\end{equation}
then $\W(\tdisc)$ converges to $\Wstar$. Moreover, the error $E(\tdisc) = (V^\tT \W(\tdisc) V  - \Lambda)$ is a diagonal matrix, whose entries satisfy
\begin{equation}
    |E_{ii}(\tdisc)| \leq \left\{
    \begin{aligned}
        \epsilon N|\lambda_i|^{1-\frac{1}{N}} %\epsilon N |\lambda_i|^{1 - \frac{1}{N}} \phantom{B}  
        % |\lambda_i| - (|\lambda_i|^{\frac{1}{N}} - \epsilon)^N 
        &\mbox{ for all } \tdisc \geq T^{\operatorname{P}}_N(\lambda_i,\epsilon,\alpha,\beta,\eta)
        &\text{ if }|\lambda_i|\geq\alpha^N,\\
        \alpha^N \phantom{BLA} 
        &\mbox{ for all } \tdisc\in\mathbb{N}_0
        &\text{ if }0 \leq \lambda_i<\alpha^N, \\
        2 \alpha^N \phantom{BLA} 
        &\mbox{ for all } \tdisc\in\mathbb{N}_0
        &\text{ if } - \alpha^N \leq \lambda_i<0. \\
    \end{aligned}
    \right.
\end{equation}
for all $\epsilon\in(0,|\lambda_i|^{\frac{1}{N}})$, and
\begin{equation}
    T_N^{\operatorname{P}}(\lambda,\epsilon,\alpha,\beta, \eta) :=
    \begin{cases}
    T_N^{\Id}(\lambda,\epsilon,\alpha,\eta)
    &\text{ if } \lambda \geq \alpha^N,\\
    \displaystyle{
    T^\Id_N\left(|\lambda|, \epsilon, \beta, \eta\right) + \frac{\alpha}{\eta \left( \frac{9N-2(c-1)}{9N}\beta|\lambda|^{\frac{1}{N}}\right)^{N-1} \left(|\lambda|^{\frac{1}{N}}-\frac{\beta}{c-1}\right)}}
    &\text{ if } \lambda \leq -\alpha^N,
    % T_N^{\Id}(\lambda,\frac{3}{4}\beta,\alpha,\eta) + \text{Work in Progress}
    % &\text{ if } \lambda < 0,
    \end{cases}
\end{equation}
where $T^{\Id}_N$ is defined in \eqref{def:TId}.
%in Theorem \ref{prop:EffectiveRankDiscretePSD}.
\end{theorem}

The proof of Theorem \ref{theorem:PerturbedInitialization_NonAsymptoticConvergence} is given in Section \ref{subsec:PerturbedInitialization}. 
\begin{remark}
    Apart from again providing precise bounds on the number of iterations sufficient to obtain a predefined approximation accuracy, Theorem \ref{theorem:PerturbedInitialization_NonAsymptoticConvergence} is instructive in pointing out two qualitatively different regimes of the gradient descent dynamics. If $\lambda_i$ is positive, the dynamics behave as in Theorem \ref{theorem:IdenticalInitialization_ContinuousNonAsymptoticRate}. If $\lambda_i$ is negative, however, gradient descent first approximates $(\lambda_i)_+ = 0$ up to perturbation level $\beta$. As soon as this level is reached, the dynamics force the $i$-th eigenvalue of $\W_1$ to become negative while all others remain positive, as illustrated in Figure \ref{fig:IdenticalVSPerturbed} and Figure \ref{fig:Perturbed} below. Afterwards the sign does not change and hence the negative eigenvalue is recovered. \\
    Interestingly enough, in practice this phenomenon may also happen when running simulations with identical initialization. Although in theory the eigenvalues of $W(k)$ should remain non-negative, it is possible that the $i$-th eigenvalue of some factor $W_j(k)$, which corresponds to a negative ground-truth eigenvalue, becomes negative due to numerical errors, i.e., when it reaches machine precision. As a consequence the $i$-th eigenvalue of $W(k)$ becomes negative and converges to the ground-truth eigenvalue. \\
    \revision{Note that the upper bound on $\eta$ in \eqref{cond:stepsize:thm} decays in $N$ like $\Omega(N^{-3})$ if $\alpha < 1$. This can be seen as follows: the defining equation of $c$, which can be written as $c^N - c^{N-1} = 1$, implies that $c^\ell - c^{\ell-1} \le 1$, for any $\ell = 1,\dots,N$ since $c > 1$. Hence, $$c^N - 1 = \sum_{\ell=1}^N (c^\ell - c^{\ell-1}) \le N, $$
    such that $c^{2N-2} \le c^{2N} \le (N+1)^2$. Consequently, we obtain for $\alpha < 1$ that
    $$ \frac{1}{9N (\c M)^{2N-2}} \ge \frac{1}{9 N(N+1)^2 \max\{1, \| \Wstar \|^{\frac{2N-2}{N}}\} } = \Omega(N^{-3}).$$}
    Let us finally mention that the condition \revision{$0 < \frac{\beta}{\c-1} < \alpha $} is used to simplify parts of the argument. Numerical simulations suggest that it is an artifact of the proof\revision{; $0 < \beta < \alpha $ is empirically sufficient.} 
\end{remark}

\textbf{III. Implicit bias towards low-rank:} Theorems \ref{theorem:IdenticalInitialization_ContinuousNonAsymptoticRate} and \ref{theorem:PerturbedInitialization_NonAsymptoticConvergence} suggest an implicit rank regularization of the gradient descent iterates $\W(\tdisc)$ if we stop at some appropriate finite $k$, since dominant eigenvalues will be approximated faster than the rest of the spectrum. The discussion in Section \ref{sec:EffectiveRank} --- in particular, Theorems \ref{prop:EffectiveRankContinuousPSD} (gradient flow, $N = 2$) and \ref{prop:EffectiveRankDiscretePSD} (gradient descent, $N \ge 2$) --- makes this precise by showing that the effective rank (a generalized notion of rank) of the iterates $\W(\tdisc)$ first drops to one and then monotonously increases, plateauing on the effective rank levels of various low-rank approximations of $\Wstar$, cf.~Figure~\ref{fig:EffRank_Intro}. Theorems~\ref{prop:EffectiveRankContinuousPSD}~and~\ref{prop:EffectiveRankDiscretePSD} explicitly characterize the time intervals during which the effective rank of $\W(\tdisc)$ remains approximately constant.

\begin{figure}[!htb]
\centering
\begin{subfigure}[c]{0.45\textwidth}
\includegraphics[width=\textwidth]{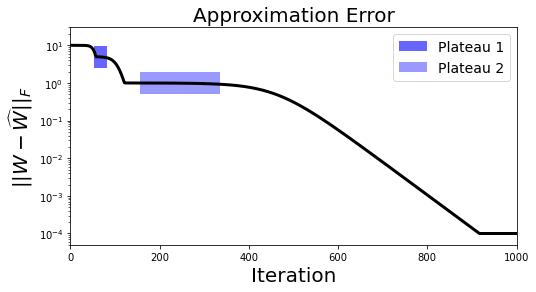}
%\subcaption{During the gradient descent there are multiple plateaus.}
%\label{fig:EffRank_IntroA}
\end{subfigure}
\quad
\begin{subfigure}[c]{0.45\textwidth}
\includegraphics[width=\textwidth]{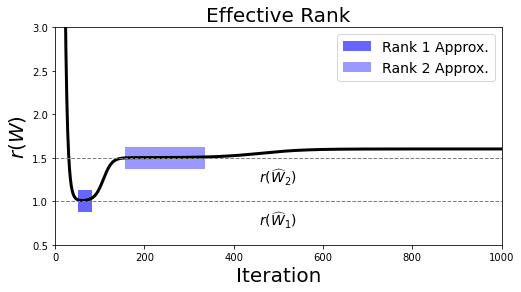}
%\subcaption{During the gradient descent the matrix $\W$ first becomes a good rank one approximation of $\Wstar$, then rank two approximation, and so on.}
%\label{fig:EffRank_IntroB}
\end{subfigure}
\caption{Dynamics of gradient descent \eqref{eq:MatrixFactorization_GradientDescent}-\eqref{eq:MatrixFactorization_IdenticalInitialization} with $N=2$, $n=200$, $\alpha = 10^{-2}$, $\eta = 10^{-2}$. The matrix $\Wstar\in\mathbb{R}^{n\times n}$ is symmetric of rank $3$ with nonzero eigenvalues $(\lambda_1,\lambda_2,\lambda_3) = (10,5,1)$. The shaded regions are the predictions of Theorem \ref{prop:EffectiveRankContinuousPSD} on where the best rank $L$ approximation $\Wstar_L$ of $\Wstar$ lies. Remarkably, those predictions correspond to the plateaus of gradient descent during training. The effective rank is defined as $r(\W)=\|\W\|_*/\|\W\|$, and all other parameters are chosen such that the prediction error (the upper bound in Theorem \ref{prop:EffectiveRankContinuousPSD}) is less than $10^{-1}$.}
\label{fig:EffRank_Intro}
\end{figure}

In addition to those highlights, we provide numerical evidence supporting the theory and simulations in more general settings suggesting that our observations are not restricted to matrix factorization with symmetric ground truths. The organization of the paper is as follows: Section \ref{sec:AnalysisOfDynamics} contains the core analysis including a quantitative description of the dynamics in \eqref{eq:MatrixFactorization_GradientDescent}-\eqref{eq:MatrixFactorization_IdenticalInitialization}, both for the identical and the perturbed initialization. Building upon those results, Section \ref{sec:EffectiveRank} then deduces an implicit low-rank bias of gradient descent and compares theoretical predictions to actual numerical outcomes. Finally, we present in Section \ref{sec:Numerics} additional numerical simulations in more general settings and discuss future work in Section \ref{sec:Discussion}.

\subsection{Related Work}
\label{subsec:RelatedWork}

%-------------------------------------------------------------

%Starting long before the recent success of deep neural network, back in the 80's, linear networks have been seen as an intermediate step of understanding the complicated dynamics of general networks~\cite{Baldi1989NeuralNA}. Although great parts of the theory are still under development, several approaches stand out due to their tractability and connection to other well-established fields of study, such as compressed sensing and phase retrieval. %\textcolor{red}{Due to the complicated nature of gradient descent, most works are based on either the dynamics of gradient flow, or the asymptotic behavior of gradient descent.}

Linear multilayer neural networks and related optimization problems have been investigated in several works \cite{arora2018optimization,arora2019implicit,bah2019learning,Baldi1989NeuralNA,geyer2019implicit,kawaguchi2016deep,li2018algorithmic,razin2020implicit}. In particular, it has been shown in \cite{bah2019learning} (extending \cite{arora2018optimization,arora2019implicit,kawaguchi2016deep}) that the gradient flow minimizing
$L(W_1,\hdots,W_N) =\frac{1}{2} \|Y - W_N \cdots W_1 X\|_F^2$, i.e., learning deep linear networks, converges to a global minimizer for almost all initializations.

%One popular choice are single layer linear networks. 
In \cite{li2018algorithmic} the authors consider the problem of recovering a symmetric, positive matrix $X$ of low rank from incomplete linear measurements $y=\mathcal{A}(X)$. They are able to show that gradient descent on the factorized problem $L(W_1) = \frac{1}{2} \| y - \mathcal{A}(W_1 W^T_1)\|_2^2$ converges to the ground truth if a restricted isometry assumption holds for $\mathcal{A}$. While this seems to suggest a bias of gradient descent towards low-rank solutions, the conclusion is questionable because restricting the linear system $y = \mathcal{A}(W)$ to positive semidefinite matrices $W$ often means that $X$ is the unique solution if $y = \mathcal{A}(X)$ for a low rank matrix $X$ \cite{geyer2019implicit,kakurate16}. %\textcolor{blue}{(unless $m < n^2/2$, cf.\ \cite[Proposition 1]{slawski2015regularization}).}
%\textcolor{red}{We have to be careful here. PSD constraints do not imply solution uniqueness if $m < n^2/2$, cf.\ \cite[Proposition 1]{slawski2015regularization}.}
%\textcolor{blue}{HOLGER: Well, \cite[Proposition 1]{slawski2015regularization} is not a contradiction to the statement that I made because in that Proposition it is not stated that the matrix is of rank smaller than $n$ (which would be a contradiction to the results in \cite{geyer2019implicit,kakurate16}). I would propose not to go into more details here.}

%analyze error bounds of gradient descent on such networks for positive definite ground-truths and orthonormal initialization. The work \cite{bah2019learning} shows that the gradient flow on the network converges to a global minimum regardless of the initialization, if the ground truth has full rank. The authors also conjecture that such statement holds for general multi-layer network.

%Multi-layer linear networks are more complicated to analyze due to the non-uniqueness of global minima, yet they are of interest to many. 

%It is shown in \cite{kawaguchi2016deep} that, despite non-uniqueness, the gradient flow on linear networks escapes local extrema and converges to global minima. 

It was observed empirically in a number of works, see e.g.~\cite{keskar2017,neyshabur2015,zhang2017}, that training deep (nonlinear) overparameterized neural networks via (stochastic) gradient descent imposes an implicit bias towards networks that generalize well on unseen data. In fact, it came as a surprise that increasing the number of parameters may decrease the generalization error despite the fact that the training error is always zero. First works towards a theoretical explanation of this phenomenon include \cite{arora2019implicit,Bach2019implicit,geyer2019implicit,gunasekar2019implicit,gunasekar2017implicit,neyshabur2017geometry,razin2020implicit,soudry2018implicit}. As a first step, most of these works concentrate on linear networks. In   \cite{soudry2018implicit,gunasekar2019implicit} the authors study classification problems and convolution networks. In particular, they show for gradient descent an implicit bias towards filters whose Fourier transform minimizes the $\ell_p$-norm, where $p=2/N$ (assuming that the loss function converges to zero for the gradient descent iterates).
Matrix sensing problems are studied in \cite{arora2019implicit,gunasekar2017implicit}, where the authors concentrate on recovering a matrix
$X \in \mathbb{R}^{n \times n}$
%minimizing $L(W) = \frac{1}{2} \|y - \mathcal{A}(X)\|$, 
from linear measurements $y_i =
\mathcal{A}(X)_i = \tr(A_i^T X)$, $i=1,\hdots,m$ via gradient descent on the functional
$\widehat{L}(W_1,\hdots,W_N) = L(W_N \cdots W_1)$ with $L(W)= \frac{1}{2} \|y - \mathcal{A}(W) \|_2^2$ where $W_j \in \mathbb{R}^{n \times n}$. 
Assuming that the measurement matrices $A_i$ commute and $X$ is positive semidefinite these works show that if  $L(W_1(k),\hdots,W_N(k)) \to 0$ as $k \to \infty$, then the limit
$W_\infty = \lim_{k \to \infty} W_N(k)\cdots W_1(k)$ minimizes the nuclear norm $\|W\|_*$ among all positive semidefinite matrices $W$ satisfying $y = \mathcal{A}(W)$. Whereas this is interesting, the restriction to commuting measurements and positive semidefinite ground truth seems very limiting. In particular, there may be at most $n$ linearly independent commuting measurement matrices $A_i$ (usually one requires at least $C r n$ measurements to recover a matrix of rank $n$) while the restriction to positive semidefinite matrices may result in uniqueness of the solution to $y = \mathcal{A}(W)$, see also \cite{geyer2019implicit}. The article \cite{razin2020implicit} indicates that implicit bias of gradient descent in deep matrix factorization may not be explainable by a norm (such as the nuclear norm) but rather via low rank. Our contribution seems to support this conjecture, but rigorously showing it for matrix recovery is still open. \HR{Only shortly before finishing our manuscript, we became aware of \cite{Gissin2019Implicit}, 
which is probably conceptually closest to our work. 
The authors study the gradient flow and gradient descent dynamics in the same setting as we do and focus 
%The authors of \cite{Gissin2019Implicit} focused 
on the question under which choices of $\alpha$ the dynamics of pairwise different eigenvalues of the product $\W_N \cdots \W_1$ are distinguishable.} Whereas parts of our technical results appeared in less general form in \cite{Gissin2019Implicit} before, cf.\ Section \ref{subsec:IdenticalInitialization}, our proof techniques and main results are fundamentally different. For further discussion on \cite{Gissin2019Implicit}, we refer the reader to Section \ref{sec:Comparison}.

Early stopping of gradient descent in deep learning has been investigated in a number of contributions, see e.g.~\cite{bartlett2018gradient,Jacot2018NTK,Yao2007EarlyStopping}. It may be  interpreted as bias-variance trade-off \cite{Yao2007EarlyStopping}.
In the context of neural tangent kernels, \cite{Jacot2018NTK} shows that when the width of network becomes infinite, the convergence rate of gradient flow is faster for eigenspaces with larger eigenvalues. Hence, early stopping may seem appealing for applications where only few major features are required. In fact, early stopping is intertwined with the idea of implicit bias, as we will discuss later in our paper.

\subsection{Notation}
\label{subsec:Notation}

%-------------------------------------------------------------

%In Section \ref{sec:Setting} we describe the problems and the model. In Section \ref{sec:EffectiveRank} we state our main theorems that demonstrate the power of early stopping. The core analysis of the dynamics is in Section \ref{sec:AnalysisOfDynamics}. We then discuss the potential of more general results in Section \ref{sec:MatrixSensing}. In the end, we present several numerical experiments based on the MNIST dataset.

%Notation-wise, 
We abbreviate $[n] := \curly{1,...n}$.
We denote matrices by uppercase letters and scalars by lowercase letters. Norms that frequently appear are the operator norm (spectral norm) $\|A\| = \sup_{\|x\|_2 = 1} \|Ax\|_2$, the Frobenius norm $\|A\|_F = (\sum_{j,k} |A_{j,k}|^2)^{1/2} = \sqrt{\tr(A^TA)}$, and the nuclear norm $\|\cdot\|_* = \sum_{j} \sigma(A)_j$, where $\sigma(A)_j$ are the singular values of $A$. 
%To distinguish the discrete variables from continuous variables, we denote the continuous variables with the "tilde symbol" on top, e.g if $\tdisc\in\mathbb{N}_0$ is the discrete time variable, then $\tcont\in\mathbb{R}_+$ is the continuous time variable. 
Throughout the paper, $\alpha$ represents the initialization, $\eta$ the step size, and $N$ the depth of the matrix factorization.
For a real number $a$, we denote $a_+ = \max\{0,a\}$.

%-------------------------------------------------------------

%\section{Problem Setting}
%\label{sec:Setting}
%\input{sec_Setting}

%-------------------------------------------------------------

\section{The Dynamics of Gradient Descent}
\label{sec:AnalysisOfDynamics}

%The characterization of the effective rank of gradient flow and gradient descent iterates $\tilde{\W}(\tcont)$ resp.\ $\W(\tdisc)$ mainly relies on the now presented analysis of their respective dynamics. Since of greater interest to a broad community, we focus our presentation in Section \ref{subsec:Eigenvalu Dynamics} on the dynamics of gradient descent. They, however, implicitly build upon corresponding results on gradient flow which can be found in Section \ref{subsec:IdenticalInitialization} and \ref{subsec:PerturbedInitialization}, together with the omitted proofs from Section \ref{subsec:Eigenvalu Dynamics}. Recall that we restrict ourselves in this section to symmetric ground-truths $\Wstar$ in the setting of matrix factorization with loss function \eqref{eq:L_MatrixFactorization}.

The goal of this section is to derive precise bounds on the full trajectory 
of the gradient descent iterations $W(\tdisc)$ defined in \eqref{eq:MatrixFactorization_GradientDescent} for the quadratic loss function \eqref{eq:L_MatrixFactorization} where $\widehat{W}$ is assumed to be a symmetric ground truth matrix. We first choose a small positive multiple of the identity as initialization of all factor matrices. However, we will see that we cannot recover negative eigenvalues of $\widehat{W}$ with such initialization. To mend this, we will consider then a slightly modified initialization, which guarantees recovery of the ground truth $\widehat{W}$ in the limit as $\tdisc \to \infty$, and characterize also the dynamics in this case.

%-------------------------------------------------------------
%
%\subsection{Eigenvalue Dynamics}
%\label{subsec:Eigenvalu Dynamics}
%
\subsection{Identical Initialization}
\label{subsec:IdenticalInitialization}

%-------------------------------------------------------------

We start by observing that the dynamics of the different eigenvalues decouple when initializing with the same multiple of the identity matrix. A similar result and proof has already appeared for the underlying gradient flow in \cite[Section E.1]{Gissin2019Implicit}.
%, which is why we omit the proof here.

%\red{Since similar result and proof already appear in \cite{Gissin2019Implicit}, we provide our statement for notation completeness but not the proof.}
% 
%We first analyze the case with identical initialization in Theorem \ref{theorem:IdenticalInitialization_ContinuousNonAsymptoticRate}, which is simple but has certain disadvantages like lacking the ability to recover negative eigenvalues. %Then we will study the more complicated case with perturbed initialization that can be used to recover the full spectrum.
%
\begin{lemma}\label{lemma:IdenticalInitialization_MatrixDynamics}
Let $(\W_j(\tdisc))_{j=1}^N$ be the solution to the gradient descent \eqref{eq:MatrixFactorization_GradientDescent} with identical initialization \eqref{eq:MatrixFactorization_IdenticalInitialization}
and let $\widehat{W} = \V \Dground \V^T$ be an eigenvalue decomposition of the symmetric ground truth matrix (where $V$ is orthogonal).
Then the matrices $\D_j(\tdisc):= \V^\tT\W_j(\tdisc)\V$ are real, diagonal and identical, i.e., $\D_j(\tdisc) = \D(\tdisc)$ for all $j$ for some $\D(\tdisc)$, and follow the dynamics
\begin{equation}
\label{eq:IdenticalInitialization_MatrixDynamics}
    \D(\tdisc+1) = \D(\tdisc) - \eta \D(\tdisc)^{N-1}(\D(\tdisc)^N-\Dground), \quad k \in \mathbb{N}_0.
\end{equation}
\end{lemma}
\revision{
\begin{proof}
Let $D(0) = \alpha I$ and define $D(\tdisc)$ for $\tdisc \in \mathbb{N}$ recursively via \eqref{eq:IdenticalInitialization_MatrixDynamics}.
We prove the claim that the matrices $D_j(\tdisc)$
are real, diagonal and that $D_j(\tdisc) = D(\tdisc)$ for all $j$ by induction. For $k=0$ this is clearly true since $W_j(0) = \alpha I$.
%by induction. It is trivially true for $\tdisc=0$ (apart from \eqref{eq:IdenticalInitialization_MatrixDynamics} which is undefined for $k=0$\textcolor{red}{(underlined but no comment)}). 
Now suppose this claim holds for some $\tdisc\in\mathbb{N}_0$. Then by \eqref{eq:L_gradient}, we have, for all $j=1,\hdots,N$ and $\tdisc \in \mathbb{N}$,
\begin{align*}
    \W_{j}(\tdisc+1)
    &= \W_{j}(\tdisc) - \eta (\W_N(\tdisc)\cdots \W_{j+1}(\tdisc))^{\tT} (\W(\tdisc)-\Wstar)  (\W_{j-1}(\tdisc)\cdots \W_{1}(\tdisc))^{\tT}\\
    &= \V\D(\tdisc)\V^\tT- \eta (\V\D(\tdisc)^{j-1}\V^\tT)^{\tT} \V(\D(\tdisc)^{N}-\Dground)\V^\tT  (\V\D(\tdisc)^{N-j}\V^\tT)^{\tT}\\
    &= \V[\D(\tdisc) - \eta \D(\tdisc)^{N-1} (\D(\tdisc)^N-\Dground)]\V^\tT
    = \V\D(\tdisc+1)\V^\tT.
\end{align*}
By orthogonality of $V$ this shows that
$D_j(k+1) = D(k+1)$ so that the induction step is completed.
%This completes the proof.
\end{proof}
}
%Lemma \ref{lemma:IdenticalInitialization_MatrixDynamics} implies that the matrix dynamics can be decoupled into $n$ independent scalar dynamics. 
Due to the previous decoupling lemma it suffices to analyze the dynamics of each diagonal entry of $D(\tdisc)$ separately.
Denoting by $\lambda$ an eigenvalue of the ground truth matrix $\widehat{W}$ the corresponding diagonal element $\d(\tdisc)$ of $\D(\tdisc)$ evolves according to the equation
%Hence, it suffices to analyze the dynamics of the solution of
\begin{equation}\label{eq:IdenticalInitialization_ScalarDynamics}
    \d(\tdisc+1) = \d(\tdisc) - \eta \d(\tdisc)^{N-1}(\d(\tdisc)^N-\lambda),
    \quad \d(0) = \alpha >0.
\end{equation}
The following lemma describes the convergence of $\d$ in \eqref{eq:IdenticalInitialization_ScalarDynamics}. It extends \cite[Lemma 1]{Gissin2019Implicit} to negative choices of $\lambda$ and to the case $\lambda \le \alpha$. Note that the proof is fundamentally different from the one presented in \cite{Gissin2019Implicit}.
\begin{lemma}\label{lemma:IdenticalInitialization_Convergence}
Let $\d$ be the solution of \eqref{eq:IdenticalInitialization_ScalarDynamics} for some $\lambda \in \RR$ and $\alpha > 0$.
\begin{itemize}
    \item If $N = 1$ and $\eta\in (0,1)$, then $\d$ converges to $\lambda$ linearly, i.e.,
    \begin{equation*}
        |\d(\tdisc) - \lambda| \leq (1-\eta)^\tdisc |\alpha- \lambda|\quad \mbox{ for all } \tdisc \in \mathbb{N}.
    \end{equation*}
    \item If $N\geq 2$ and
    \begin{align*}
        0 < \eta < 
        \begin{cases}
         %\frac{1}{
         \left( N \max \left\{ \alpha,|\lambda|^{\frac{1}{N}} \right\}^{2N-2} \right)^{-1} & \mbox{ if } \lambda > 0,\\
  %        \alpha^{-2N+2} & \mbox{ if } \lambda = 0,\\
        \alpha^{-2N+2} & \mbox{ if } \lambda \leq 0 \mbox{ \rm{and} } \alpha \geq |\lambda|^{\frac{1}{N}}, \\
         \left( (3N-2) |\lambda|^{2-\frac{2}{N}}\right)^{-1} & \mbox{ if } \lambda < 0 \mbox{ \rm{and} } 0 < \alpha < |\lambda|^{\frac{1}{N}}, \\
        \end{cases}
    \end{align*}
    then $\lim_{\tdisc \to \infty} \d(\tdisc) = \lambda_+^{\frac{1}{N}} =\max\{\lambda, 0\}^{\frac{1}{N}}$. Moreover, the error $|\d(\tdisc)^N-\lambda_+|$ is monotonically decreasing and, for all $\tdisc \ge 0$, one has that $\d(\tdisc) \in [\alpha,\lambda^\frac{1}{N}]$ if $\lambda \geq \alpha^N$, and $\d(\tdisc) \in [\lambda_+^{\frac{1}{N}},\alpha]$ if $\lambda < \alpha^N$.
\end{itemize}
\end{lemma}

\revision{
\begin{remark}
    Note that the proof of Lemma \ref{lemma:IdenticalInitialization_Convergence} shows that the sequence $d(\tdisc)$ is monotonically increasing for $\lambda > \alpha$ and monotonically decreasing for $\lambda < \alpha$. We will repeatedly make use of this observation in the following.
\end{remark}
}
\begin{proof}
%%% N=1
For $N=1$ the computation is straight-forward and follows by induction. For $N\geq 2$ we need to make a case distinction with four cases that depend on the sign and the magnitude of $\lambda$. This is due to the fact that the sign of $\lambda$ determines the limit of $\d$, while the magnitude of $\lambda$ determines whether $\d$ is increasing or decreasing in time.
% For $N=1$, we have $\d(\tdisc+1) - \lambda = (1-\eta)(\d(\tdisc)-\lambda)$.
%\revision{For $N=1$ by induction,} $\d(\tdisc) - \lambda = (1-\eta)^{\tdisc}(\alpha-\lambda)$. This proves the first claim.

%%% N=2, lambda = 0
%If $N \geq 2$ and $\lambda = 0$ then $\d(\tdisc + 1) = f(\d(\tdisc))$ for $f(x) = x - \eta x^{2N-1}$. 
%For $0 \leq x \leq \alpha$, 
%it holds $f(x) \leq x \leq \alpha$ and 
%$$
%f(x) = x(1-\eta x^{2N-2}) \geq x (1- \eta \alpha^{2N-2}) \geq %0
%$$
%by assumption on $\eta$. Hence, $f([0,\alpha]) \subset [0,\alpha]$ and the sequence $(\d(\tdisc))_{\tdisc \geq 0}$ is monotonically decreasing. Again by the monotone convergence theorem, it follows that $\d(\tdisc)$ converges to the unique fix point $0$ of $f$ as $\tdisc \to \infty$.

%%% N=2, |lambda|^(1/N) > alpha
Let $ N\ge 2$. Before diving into the case distinction, let us define, for $|\lambda|^\frac{1}{N} > \alpha$, the function 
$$
g(x) = x - \eta x^{N-1}(x^N - \lambda)
$$ and observe that $\d(\tdisc+1) = g(\d(\tdisc))$. For $x \in [0,|\lambda|^{\frac{1}{N}}]$,
$z = x/|\lambda|^{1/N} \in [0,1]$ and $\sigma = \operatorname{sign}(\lambda)$ its derivative satisfies
\begin{align}
g'(x) & = 1 - \eta ((N-1)x^{N-2}(x^N-\lambda) + N x^{2N-2}) \notag \\
& = 1- \eta |\lambda|^{2-\frac{2}{N}}\left((N-1)z^{N-2}(z^N- \sigma) + Nz^{2N-2}\right) \notag\\
& \geq \left\{ \begin{array}{ll} 
1  - \eta |\lambda|^{2-\frac{2}{N}} N & \mbox{ if } \lambda > 0, \\
 1 - \eta |\lambda|^{2-\frac{2}{N}} (3N-2)  & \mbox{ if } \lambda < 0,
 \end{array}\right. \label{ineq:g-deriv}
\end{align}
where we used that $z \in [0,1]$.
By the assumption on $\eta$ it follows that
$g'(x) \geq 0$ for all $x \in [0, |\lambda|^{\frac{1}{N}}]$ and, hence, $g$ is monotonically increasing on that interval. We can now distinguish the four cases defined by $\lambda \ge 0$/$\lambda < 0$ and $|\lambda|^\frac{1}{N} > \alpha$/$|\lambda|^\frac{1}{N} \le \alpha$.

For $\lambda > 0$ and $|\lambda|^\frac{1}{N} > \alpha$, $g$ has its fixed points at $0$ and $\lambda^{\frac{1}{N}}$ so that by monotonicity $g$ maps the interval $[0, \lambda^{\frac{1}{N}}]$ to itself. Hence, $\lambda^\frac{1}{N} \ge g(x) > x$ for all $x \in (0, \lambda^{\frac{1}{N}})$ so that $\lambda^\frac{1}{N} \ge \d(\tdisc + 1) > \d(\tdisc)$ for all $\tdisc \in \NN$. \revision{This means that the error $|\d(\tdisc)^N-\lambda|$ is monotonically decreasing.} Hence, $\{\d(\tdisc)\}_{\tdisc\geq 0}$ is an increasing bounded sequence. By the monotone convergence theorem, it converges to the unique fixed point $\lambda^{\frac{1}{N}}$ of $g$ on $[\alpha, \lambda^{1/N}]$.

For $\lambda < 0$ and $|\lambda|^\frac{1}{N} > \alpha$, since $g$ is monotonically increasing on $[0,|\lambda|^{\frac{1}{N}}]$ and $g(0) = 0$ we have $g(x) \geq 0$ for all $[0,|\lambda|^{\frac{1}{N}}]$. Note that $g(x) < x$ if $\lambda \leq 0$ and $x > 0$. Since $\d(0)= \alpha >0$, it follows that $0 < \d(\tdisc+1) < \d(\tdisc)$ for all $\tdisc \in \NN$. \revision{This means that the error $|\d(\tdisc)^N-0|$ is monotonically decreasing.} As $g$ has the unique fixed point $0$ in the interval $[0,\alpha]$, the sequence $\{\d(\tdisc)\}_{\tdisc\geq 0}$ converges to $0$ by the monotone convergence theorem.

% For $\lambda \geq 0$ and $\alpha > \lambda^\frac{1}{N}$
For $\lambda \geq 0$ and $|\lambda|^\frac{1}{N}\leq\alpha$, we show by induction that $\d(\tdisc)$ remains in the interval $[\lambda^{\frac{1}{N}}, \alpha]$ and is monotonically decreasing in $\tdisc$, \revision{which implies that the error $|\d(\tdisc)^N-\lambda|$ is monotonically decreasing.} For $\tdisc = 0$, the claim is trivially fulfilled. If $\d(\tdisc)\in[\lambda^{\frac{1}{N}}, \alpha]$, then $d(\tdisc)^{N} - \lambda \geq 0$ so that $\d(\tdisc+1) \leq \d(\tdisc) \leq \alpha$. Moreover, if $\lambda > 0$
\begin{align*}
    \d(\tdisc+1) - \lambda^{\frac{1}{N}}
    &= \d(\tdisc) - \eta \d(\tdisc)^{N-1}(\d(\tdisc)^N - \lambda) - \lambda^{\frac{1}{N}}\\
    &= (\d(\tdisc) - \lambda^{\frac{1}{N}})(1- \eta \d(\tdisc)^{N-1}\sum_{k=1}^{N} \d(\tdisc)^{k-1}\lambda^{\frac{N-k}{N}})\\
    &\geq (\d(\tdisc) - \lambda^{\frac{1}{N}})(1- \eta N\alpha^{2N-2})
    \geq 0,
\end{align*}
since $d(\tdisc) \geq \lambda^{\frac{1}{N}}$ and by assumption on $\eta$. If $\lambda = 0$, then 
$$
d(\tdisc+1) = \d(\tdisc) - \eta \d(\tdisc)^{2N-1} = \d(\tdisc)(1 - \eta \d(\tdisc)^{2N-2}) \leq \d(\tdisc)
$$
since $\d(\tdisc) \geq 0$ and by assumption on $\eta$. Hence $\d(\tdisc+1)\in[\lambda^{\frac{1}{N}}, \alpha]$. \revision{Here we have a bounded decreasing sequence, and hence it must converges to the only fixed point in the domain, which is $\lambda^{\frac{1}{N}}$.}

%%% N=2, |lambda|^(1/N) \le alpha

Finally, consider $\lambda < 0$ with $|\lambda|^{\frac{1}{N}}\leq \alpha$. If $\d(\tdisc)\in [0,\alpha] $, then $\d(\tdisc+1)\leq \d(\tdisc) \leq \alpha$. Moreover,
\begin{equation*}
    \d(\tdisc+1)
    = \d(\tdisc)(1-\eta \d(\tdisc)^{N-2}(\d(\tdisc)^N - \lambda))
    \geq \d(\tdisc)(1- 2\eta \alpha^{2N-2})
    > 0
\end{equation*}
by the assumption on $\eta$. Hence $\d(\tdisc+1)\in [0,\alpha]$. \revision{This means that the error $|\d(\tdisc)^N-0|$ is monotonically decreasing.} \revision{Here we have a bounded decreasing sequence, and hence it must converges to the only fixed point in the domain, which is $0$.}
\end{proof}
Lemma \ref{lemma:IdenticalInitialization_Convergence} shows that in the presence of matrix factorization, i.e., $N\geq 2$, gradient descent with identical initialization loses its ability to recover negative eigenvalues and the condition on the constant $\alpha$ in the initialization becomes more restrictive. (The condition on $\eta$ becomes either more or less restrictive depending on $|\lambda|$.) 
Note that Lemma \ref{lemma:IdenticalInitialization_Convergence} only provides a sufficient condition on the stepsize $\eta$ for convergence. However, this condition is basically necessary up to the constant, see Lemma \ref{lemma:IdenticalInitialization_Convergence_Necessary} in Appendix \ref{sec:Appendix_Optimality}. 

Having settled convergence, we will now analyze the number of iterations that are needed in order to reach an $\epsilon$-neighborhood of $\lambda$. By Lemma \ref{lemma:IdenticalInitialization_MatrixDynamics}, this forms the basis for analyzing the implicit bias of gradient descent \eqref{eq:MatrixFactorization_GradientDescent} on the matrix factorized problem with $N \geq 2$. 
In order to state our theorem, we need to introduce a few quantities corresponding to certain numbers of iterations that are important in our subsequent analysis. For $\lambda > 0$ and $\mu > 0$ we define
\begin{align}
    U^-_N(\mu) &: = %\frac{1}{|\lambda|}\cdot
    \begin{cases}
    - \ln(\mu) 
    &\text{ if } N=2,\\
    \frac{1}{N-2}\left(\frac{1}{\mu^{N-2}}\right)
    %-\frac{1}{\alpha^{N-2}}\right)
    &\text{ if } N\geq 3,
    \end{cases} \notag \displaybreak[2]\\
%\end{align*}
%and for $N=2$ and $\lambda > %0$
%\begin{align*}
    U^+_N(\lambda,\mu) &:=
    \begin{cases}
      - \frac{1}{2\lambda} \ln\left(\frac{\lambda }{\mu^2} -  1\right), 
      %- \ln\left(\frac{\lambda}{\mu^2} - 1\right)\right) 
    &\mbox{ if } N=2,\\
    \frac{\lambda^{\frac{2}{N}-2}}{N} \sum_{\ell=1}^{N} \Re \left( e^{\frac{4\pi i\ell}{N}} \left( \ln\left(e^{\frac{2\pi i \ell}{N}}- \mu \lambda^{-\frac{1}{N}}\right)\right)\right) - \frac{1}{\lambda(N-2)\mu^{N-2}}
    &\mbox{ if } N\geq 3,
    \end{cases} \label{def:Uplus}
\end{align}
%(\textcolor{red}{H: Still need to add the cases $\alpha \in [(c_N \lambda)^{\frac{1}{N}}, \lambda^{\frac{1}{N}}]$ and $\epsilon > \lambda^{1/N} - (c_N \lambda)^{\frac{1}{N}}$.})
where $\ln$ denotes the principle branch\footnote{Any other branch of the complex logarithm can be taken as well as long as we always choose the same branch.} of the complex logarithm, i.e., $\Im (\ln(z)) \in (-\pi,\pi]$ for all $z \in \mathbb{C} \setminus\{0\}$. Note that we can express $U_N^+$ also in terms of purely real expressions, see Remark~\ref{rem:ExplicitForm} and Appendix \ref{sec:Appendix_ContinuousDynamics} for details.
Then, for $\alpha, \eta > 0$ we set
\begin{align}
\label{def:T+-}
T^{-}_N(\mu,\alpha) & := U^{-}_N(\mu) - U^{-}_N(\alpha),%\displaybreak[2]\\
\qquad T^{+}_N(\lambda,\mu,\alpha) := U^+_N(\lambda,\mu) - U^+_N(\lambda,\alpha),\displaybreak[2]
%\qquad 
%\widehat{T}_N(\lambda,\mu,\eta) & := 
%\frac{\ln(\epsilon) - \ln \left(\lambda^{1/N} \right) - \ln \left( 1-c_N^{1/N}\right)}{\ln(1- \eta N  \lambda^{2-\frac{2}{N}})}, \quad \mbox{ where } c_N  := \frac{N-1}{2N-1},
%:=\frac{\ln(\mu) - \ln(\lambda^{\frac{1}{N}})}{\ln(1 - \eta \lambda^{2-\frac{2}{N}})},
\end{align}
and further define
% \begin{align*}
% %    c_N  := \frac{N-1}{2N-1}, \quad
%     s_N(\lambda,\alpha)  
%      %\frac{\lambda^{2-\frac{1}{N}}}{\alpha^{N-1}(\lambda - \alpha^N)},
%     & := \left\lceil \left( \frac{N-1}{2N-1} \right)^{1-\frac{1}{N}} \left(\frac{\lambda^{\frac{1}{N}}}{\alpha}\right)^{N-1} \right\rceil,\\
%     c_N  & := \frac{N-1}{2N-1},
%     \quad
%     a_N  := \left|\ln \left( 1-c_N^{1/N}\right)\right|,
%     \quad 
%     b_N := \left| \ln\left(\frac{2N-1}{2N-2} - \left(\frac{N-1}{2N-1}\right)^{\frac{1}{N}}\right)\right|.
%     %\gamma_N(\lambda,\eta) := \frac{\ln\left(1-\eta N \lambda^{2-\frac{2}{N}}\right)}{\ln\left(1-\eta c_N N \lambda^{2 - \frac{2}{N}}\right)}.
% \end{align*}
\begin{align*}
    &c_N := \frac{N-1}{2N-1},
    \quad
    s_N(\lambda,\alpha) := \left\lceil c_N^{1-\frac{1}{N}} \left(\frac{\lambda^{\frac{1}{N}}}{\alpha}\right)^{N-1} \right\rceil,\\
    &a_N  := \left|\ln \left( 1-c_N^{\frac{1}{N}}\right)\right|,
    \quad 
    b_N := \left| \ln\left(\frac{1}{2c_N} - c_N^{\frac{1}{N}}\right)\right|.
\end{align*}
Finally, for $\lambda \in \RR$, we define $m_{\alpha,\epsilon} = \max(\alpha,\epsilon)$ and
\begin{align}
& T^{\Id}_N(\lambda,\epsilon,\alpha,\eta) \notag\\
\label{def:TId}
& := \begin{cases}
\frac{1}{\eta|\lambda|}T^-_N(\epsilon,\alpha)& \mbox{if } \lambda < 0 \; ,\\
0 & \mbox{if } 0\leq\lambda<\epsilon^N,\\
\frac{1}{\eta} T_N^+(\lambda,\lambda^\frac{1}{N} + \epsilon,\alpha) & \mbox{if } \epsilon^N \leq \lambda < \alpha^N,  \\
\displaystyle{\frac{\ln((\lambda^{\frac{1}{N}}-\alpha)/\epsilon)}{\left|\ln(1- \eta N (c_N \lambda)^{2-\frac{2}{N}})\right|}}
& \mbox{if } m_{\alpha,\epsilon}^N \leq \lambda < \frac{\alpha^N}{c_N}, \\
\frac{1}{\eta}T_N^+(\lambda, \lambda^{\frac{1}{N}} - \epsilon,\alpha)+s_N(\lambda,\alpha) & \mbox{if }\frac{m_{\alpha,\epsilon}^N}{c_N}\leq \lambda < \big(\frac{\epsilon}{1-\sqrt[N]{c_N}}\big)^N,\\
% & \revision{\mbox{else if }\frac{\alpha^N}{c_N}\leq \lambda < (1-c_N^{\frac{1}{N}})^{-N}\epsilon^N},\\
\displaystyle{\frac{1}{\eta}T^+_N\left(\lambda, (c_N \lambda)^{\frac{1}{N}},\alpha\right) + s_N(\lambda,\alpha) 
+ \frac{\ln(\lambda^{\frac{1}{N}}/\epsilon)  - a_N}{\left|\ln(1- \eta N (c_N \lambda)^{2-\frac{2}{N}})\right|}}
& \mbox{otherwise.}
% \revision{\frac{\alpha^N}{c_N}\leq \lambda, \lambda > (1-c_N^{\frac{1}{N}})^{-N}\epsilon^N}. 
\end{cases}
\end{align}
% I copy the original below for record and comparison.
% \begin{align*}
% & T^{\Id}_N(\lambda,\epsilon,\alpha,\eta) \notag\\
% & := \begin{cases}
% \frac{1}{\eta|\lambda|}T^-_N(\epsilon,\alpha)& \mbox{if } \lambda < 0 \; ,\\
% \frac{1}{\eta} T_N^+(\lambda,\lambda^\frac{1}{N} + \epsilon,\alpha) & \mbox{if } 0 \le \lambda < \alpha^N,  \\
% \displaystyle{\frac{\ln((\lambda^{\frac{1}{N}}-\alpha)/\epsilon)}{\left|\ln(1- \eta N (c_N \lambda)^{2-\frac{2}{N}})\right|}}
% & \mbox{if } 0 < c_N \lambda < \alpha^N \leq \lambda,\\
% \displaystyle{\frac{1}{\eta}T^+_N\left(\lambda, (c_N \lambda)^{\frac{1}{N}},\alpha\right) + s_N(\lambda,\alpha) 
% + \frac{\ln(\lambda^{\frac{1}{N}}/\epsilon)  - a_N}{\left|\ln(1- \eta N (c_N \lambda)^{2-\frac{2}{N}})\right|}}
% & \mbox{if } c_N \lambda \ge \alpha^N,\epsilon < (1-c_N^{\frac{1}{N}}) \lambda^{\frac{1}{N}}, \\
% \frac{1}{\eta}T_N^+(\lambda, \lambda^{\frac{1}{N}} - \epsilon,\alpha)+s_N(\lambda,\alpha) & \mbox{if } c_N \lambda \geq \alpha^N, 
% \epsilon \geq (1-c_N^{\frac{1}{N}}) \lambda^{\frac{1}{N}}. 
% \end{cases}
% \end{align*}
% 
The quantity $T^{\Id}_N(\lambda,\epsilon,\alpha,\eta)$ estimates the required number of gradient descent iterations to reach a certain accuracy 
$\epsilon \in (0, |\lambda_+^{\frac{1}{N}}-\alpha|)$ when
starting with identical initialization with parameter $\alpha > 0$ and using step size $\eta$. In particular, it
illustrates that fine approximation of positive eigenvalues dominating $\alpha$ (fourth case) happens in two stages: to obtain a rough approximation a fixed number of iterations is necessary ($T_N^+$ does not depend on $\epsilon$) while to obtain approximation of accuracy $\epsilon \ll 1$ one needs an additional number of iterations depending on $\epsilon$, cf.~Figure \ref{fig:Tfigure}. Note that Remark~\ref{rem:main-thm} already commented on the behaviour of $T^\Id_N$ in the most relevant fourth case, see also Lemma~\ref{lem:approximate:TId}.

\begin{figure}[!htb]
\begin{subfigure}[c]{0.45\textwidth}
\includegraphics[width=\textwidth]{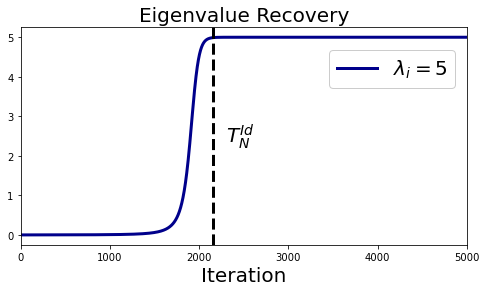}
\subcaption{Approximation of positive eigenvalue and our prediction.}
\label{fig:TfigureA}
\end{subfigure}
\quad
\begin{subfigure}[c]{0.45\textwidth}
\includegraphics[width=\textwidth]{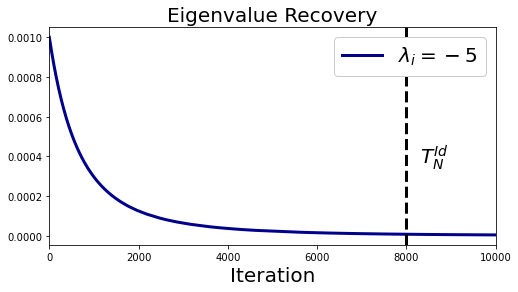}
\subcaption{Approximation of negative eigenvalue and our prediction.}
\label{fig:TfigureB}
\end{subfigure}
\caption{We see enlarged parts of Figure \ref{fig:IdenticalVSPerturbedA}. Our theoretical prediction $T^{\Id}_N$ is quite sharp in estimating the necessary number of iterations to reach certain accuracy. Again due to the limitation of identical initialization, only positive eigenvalues are recovered. Here $N=3$, $\alpha = 10^{-1}$, $\eta = 10^{-3}$. } \label{fig:Tfigure}
\end{figure}

With these definitions at hand we are ready to state the first core result. 
%on the number of iterations required to reach an $\epsilon$-neighborhood of $\lambda_+$.
Note that the third case $c_N \lambda < \alpha^N \leq \lambda$ together with the general assumption $\epsilon \in (0, |\lambda_+^{\frac{1}{N}}-\alpha|)$ implies that $\epsilon < (1-c_N^{\frac{1}{N}}) \lambda^{\frac{1}{N}}$.
\begin{theorem}\label{thm:IdenticalInitialization_ConvergenceRate}
Let $N \geq 2$, $\lambda \in \RR$, $\alpha > 0$, $\eta > 0$ and $\d$ be the solution of \eqref{eq:IdenticalInitialization_ScalarDynamics}. Define $\M = \max(\alpha,|\lambda|^{\frac{1}{N}})$.
Suppose $\eta$ satisfies
\begin{align} \label{eq:cond:stepsize}
     0 < \eta <
     \begin{cases} %\frac{1}{
     \frac{1}{2N\M^{2N-2}} & \mbox{ if } \lambda \geq 0 \\
     %\frac{1}{
     \frac{1}{(3N-2)\M^{2N-2}} & \mbox{ if } \lambda < 0.
     \end{cases}
\end{align}
% \begin{align} \label{eq:cond:stepsize}
%      0 < \eta <
%      \begin{cases} %\frac{1}{
%      \left(\max \left\{ (2N-2)\alpha^{2N-2}, N |\lambda|^{2 - \frac{2}{N}} \right\}\right)^{-1} & \mbox{ if } \lambda \geq 0 \\
%      %\frac{1}{
%      \left( (3N-2) \max\left\{ \alpha, |\lambda|^{\frac{1}{N}}\right\}^{2N-2}\right)^{-1} & \mbox{ if } \lambda < 0.
%      \end{cases}
% \end{align}
Further, let
$\epsilon\in(0, |\alpha - \lambda_+^{\frac{1}{N}}|)$ be the desired error and $T = \min\{ \tdisc: |\d(\tdisc) - \lambda_+^{\frac{1}{N}}|\leq \epsilon\}$ be the %sufficient 
minimal number of iterations to achieve such error bound. Then
\begin{equation}
%\label{eq:IdenticalInitialization_ContinuousNonAsymptoticRate}
    T 
    \leq T^\Id_N(\lambda,\epsilon,\alpha,\eta).
\end{equation}
Moreover, in the case $\lambda > \alpha^N$ we have the lower bound 
%and $\epsilon < \lambda^{\frac{1}{N}} - \alpha$ (otherwise $T=0$),
\begin{equation}
    T \geq \begin{cases} \displaystyle{\frac{1}{\eta} T_{N}^+(\lambda,(c_N \lambda)^{\frac{1}{N}},\alpha) + \frac{\ln\left( \lambda^{\frac{1}{N}}/ \epsilon\right) - b_N}{\left| \ln(1-\eta N \lambda^{2 - \frac{2}{N}})\right|} } 
    %\widehat{T}_N (\lambda,\epsilon,\eta) \leq
    & \mbox{ if } \alpha^N < c_N \lambda,\\
    \displaystyle{\frac{\ln((\lambda^{\frac{1}{N}}-\alpha)/\epsilon)}{\left|\ln(1- \eta N (c_N \lambda)^{2-\frac{2}{N}})\right|}}
    % \ln\left( \frac{\lambda^{1/N} - \alpha}{\epsilon}\right)/\left|\ln\left(1- \eta N \lambda^{2-\frac{2}{N}}\right) \right|
    & \mbox{ if } \alpha^N \geq c_N \lambda.
    \end{cases}
\end{equation}

% \begin{equation} \label{eq:IdenticalInitialization_ContinuousNonAsymptoticRate}
%     T_{\text{Id}}(\lambda,\epsilon):=
%     \begin{cases}
%         T_0(\lambda,\epsilon)
%         &\text{if } \lambda \leq 0\\
%         T_1(\lambda,\zeta) + s(\lambda) + \gamma(\lambda) T_2(\lambda,\epsilon) 
%         &\text{if } \lambda > 0
%     \end{cases},
% \end{equation}
% and
% \begin{align} \label{eq:Tdef}
% \begin{split}
%     T_0(\lambda,\epsilon) &:= \frac{1}{\eta(-\lambda)}\cdot
%     \begin{cases}
    % \ln(\alpha) - \ln(\epsilon) 
    % &\text{ if } N=2\\
    % \frac{1}{N-2}\left(\frac{1}{\epsilon^{N-2}}-\frac{1}{\alpha^{N-2}}\right)
    % &\text{ if } N\geq 3.
    % \end{cases} \\
    % T_1(\lambda,\epsilon) &:=
    % \frac{1}{\eta\lambda}\cdot\begin{cases}
    % c - \frac{1}{2}\log|\epsilon^2 + \lambda| + \log|\epsilon|
    % &\text{ if } N=2\\
    % c - \frac{1}{N}\sum_{\ell = 1}^N c_{N, \ell} \log(\epsilon- r_{N, \ell}) - \frac{1}{(N-2)\epsilon^{N-2}}
    % &\text{ if } N \geq 3
    % \end{cases}\\
    % T_2(\lambda,\epsilon) &:=\frac{\ln(\epsilon) - \ln(\lambda^{\frac{1}{N}})}{\ln(1 - \eta \lambda^{2-\frac{2}{N}})}\\
%     s(\lambda) &= \frac{\lambda^{2-\frac{1}{N}}}{ \alpha^{N-1}(-\alpha^N+\lambda)}\\
%     \gamma(\lambda) &= \frac{\ln(1 - \eta N\lambda^{2-\frac{2}{N}})}{\ln\left(1 - \eta c_2 N\lambda^{2-\frac{2}{N}}\right)}.
% \end{split}
% \end{align}
% The constants are given by $c_2 = ((N-1)/(2N-1))^{2 - \frac{2}{N}}$, $\zeta = ((N-1)\lambda/(2N-1))^{\frac{1}{N}}$, the constants $c,c_{N,l},r_{N,l}$ are defined in Lemma \ref{lemma:IdenticalInitialization_ContinuousScalarDynamics}.
\end{theorem}

%\begin{remark} \label{rem:IdenticalInitialization_ConvergenceRate}
%   If $\lambda > \alpha^N$, our theoretical considerations even give a lower bound on $T$. Note that the gap between upper and lower bound (quantified by $s_N$ and $\gamma_N$) shrinks, for $\eta \rightarrow 0$.
%\end{remark}

%We also use the gradient flow dynamics as an intermediate step for our analysis of gradient descent. The following lemma describes the continuous dynamics of gradient flow.

%%%%%%% To be completed %%%%%%%
% \begin{remark}\label{rem:U-real-expression}
%   In the case $N\geq 3$, the quantity $T_N^+(\lambda,\nu,\alpha)$ can by expressed 
%   alternatively by defined as $T_N^+(\lambda,\nu,\alpha) =  \widetilde{U}_N^+(\lambda,\mu) - \widetilde{U}_N^+(\lambda,\alpha)$, where $\widetilde{U}_{N}^+$ is defined by purely real-valued expressions,
%   $$
%   \widetilde{U}_N^+(\lambda,\mu) = 
%   $$
%   \red{TO BE COMPLETED}.
% \end{remark}
%%%%%%% To be completed %%%%%%%

The proof of Theorem \ref{thm:IdenticalInitialization_ConvergenceRate} uses the following two lemmas. The first analyzes the continuous analog of \eqref{eq:IdenticalInitialization_ScalarDynamics}, namely the gradient flow following the differential equation obtained by letting the step size $\eta$ tend to zero, i.e.,
\begin{equation}\label{eq:IdenticalInitialization_ContinuousScalarDynamics}
    \y'(\tcont) = - \y(\tcont)^{N-1}(y(\tcont)^N-\lambda),
    \quad \y(0) = \alpha>0.
\end{equation}
%\textcolor{red}{
\HR{The statement below significantly  extends \cite[Theorem 1]{Gissin2019Implicit}, which only covers the cases $N=1,2$ and the limit $N \to \infty$,
while we cover arbitrary $N \in \mathbb{N}$.}

%because it includes the cases for all $N\in\mathbb{N}$.}
%
\begin{lemma}\label{lemma:IdenticalInitialization_ContinuousScalarDynamics}
%Consider the iteration \eqref{eq:IdenticalInitialization_ScalarDynamics}. Its continuous analog
The gradient flow defined by \eqref{eq:IdenticalInitialization_ContinuousScalarDynamics} has the following solution.
If $\lambda = 0$ then
\begin{align}\label{eq:IdenticalInitialization_LambdaZero}
    \y(\tcont) = \begin{cases}
    \alpha e^{-\tcont}, & N = 1, \\
    \revision{\round{ (2N - 2)\tcont + \alpha^{-(2N-2)}}^{-\frac{1}{2N-2} }}, & N \ge 2.
    % \round{ (2N - 2)\tcont + \alpha^{\frac{2}{N}-2}}^{-\frac{N}{2N-2} }, & N \ge 2
    \end{cases}
\end{align}
% \noteE{In response to comment 2.2, yes the solution should read like
% \begin{align*}
%     y'(t) = -y(t)^{2N-1}
%     &\implies  y^{1-2N}dy = -dt\\
%     &\implies  \frac{1}{2-2N}y^{2-2N} = -t + C\\
%     &\implies  y = ((2N-2)t + \alpha^{-(2N-2)})^{-\frac{1}{2N-2}}
% \end{align*}
% }
%if $\lambda = 0$, and 
If $\lambda > 0$ then the solution is given implicitly by
\begin{equation}\label{eq:IdenticalInitialization_ContinuousScalarDynamicsSolution}
    \tcont=\begin{cases}
     \ln\left|\frac{\lambda-\alpha}{\lambda-\y(\tcont)}\right|
    &\text{ if } N=1\\
    U^+_N(\lambda,\y(\tcont)) - U^+_N(\lambda,\alpha)
    %\frac{1}{2\lambda}\left( \ln\left(\frac{\lambda }{\alpha^2} -  1\right) - \ln\left(\frac{\lambda}{\y(\tcont)^2} - 1\right)\right) 
    &\text{ if } N\geq 2,
    % \\
    % \frac{1}{\lambda} \left(c - \frac{1}{N}\sum_{\ell = 1}^N c_{N, \ell} \log(\y(\tcont)- r_{N, \ell}) -\frac{1}{(N-2)\y(\tcont)^{N-2}}\right)
    % &\text{ if } N \geq 3
    \end{cases}
\end{equation}    
where $U^+_N$ is defined in \eqref{def:Uplus}.
% and for $N \geq 3$ by
% \begin{align}
%  \tcont & = \frac{\lambda^{\frac{2}{N}-2}}{N} \sum_{\ell=1}^{N} \Re \left( e^{\frac{4\pi i\ell}{N}} \left( \ln\left(\lambda^{\frac{1}{N}} e^{\frac{2\pi i \ell}{N}}- \y(\tcont)\right) - \ln\left( \lambda^{\frac{1}{N}} e^{\frac{2\pi i \ell}{N}} - \alpha \right)\right)\right) \notag\\
% & \quad +   \frac{1}{\lambda(N-2)} \left( \frac{1}{\alpha^{N-2}} - \frac{1}{\y(\tcont)^{N-2}} \right), \label{eq:continuous-N3} %
% \end{align}
% where $\ln$ denotes the principle complex logarithm, i.e., $\Im \ln(z) \in (-\pi,\pi]$ for all $z \in \mathbb{C} \setminus\{0\}$.
%if $\lambda > 0$, where $ r_{N, \ell} := \lambda^{\frac{1}{N}} e^{\frac{2\pi i \ell}{N}}$, $c_{N, \ell} = r_{N, \ell}^{2-N}$, and $c$ are constants such that $\y(0) = \alpha$.
\end{lemma}
\begin{remark} \label{rem:ExplicitForm}
For $N = 1,2$ and $\lambda \neq 0$ we can also give the solution of \eqref{eq:IdenticalInitialization_ContinuousScalarDynamics} in explicit form, 
\[
\y(t) = \left\{ \begin{array}{ll} \lambda - (\lambda - \alpha) \exp(-t) & \mbox{ if } N = 1,\\
\sqrt{\frac{\lambda}{1+\left(\frac{\lambda}{\alpha^2}-1\right)\exp(-2\lambda t)}} & \mbox{ if } N =2.
\end{array}\right.
\]
The solution for $N=2$ has been derived already in \cite{Bach2019implicit}. Let us also mention that, for $N \ge 3$, the solution can be expressed in an alternative way avoiding the use of complex logarithms. For details, see Appendix \ref{sec:Appendix_ContinuousDynamics}.
\end{remark}

\begin{proof}
% The continuous analog of \eqref{eq:IdenticalInitialization_ScalarDynamics} is given by
% \begin{align*}
%     \tilde{\d}'(\tcont) := \lim_{\eta \to 0} \frac{\tilde{\d}(\tdisc+1) - \tilde{\d}(\tdisc)}{\eta} = -\tilde{\d}(\tcont)^{N-1}(\tilde{\d}(\tcont)^N-\lambda).
% \end{align*}
If $\lambda = 0$, or $\lambda\neq 0$ and $N = 1$, it is easy to verify the stated solutions.

Let now $\lambda > 0$ and $N \geq 2$. The differential equation \eqref{eq:IdenticalInitialization_ContinuousScalarDynamics} is separable, hence its solution $\y(t)$ satisfies
\[
t = \int_0^t dt = - \int_0^{\y(t)} \frac{1}{\z^{N-1}(\z^N-\lambda)} dz.
\]
We denote by $$
r_{N,\ell} = \lambda^{\frac{1}{N}} e^{\frac{2\pi i \ell}{N}}, \quad \ell = 1,\hdots,N,
$$
the complex roots of $\lambda$. 
A partial fraction decomposition gives
\begin{equation}\label{eq:partial-fraction}
\frac{1}{\z^{N-1}(\z^N-\lambda)}
    = \frac{1}{\lambda}\left( \frac{\z}{\z^{N}-\lambda} - \frac{1}{\z^{N-1}}  \right)
    = \frac{1}{\lambda}\left( \frac{1}{N}\sum_{\ell = 1}^N\frac{c_{N, \ell}}{\z- r_{N, \ell}} - \frac{1}{\z^{N-1}} \right).
\end{equation}
% The residue method yields
% \begin{align*}
%     \frac{1}{\tilde{\d}^{N-1}(\tilde{\d}^N-\lambda)}
%     &= \frac{1}{\lambda}\left( \frac{\tilde{\d}}{\tilde{\d}^{N}-\lambda} - \frac{1}{\tilde{\d}^{N-1}}  \right)
%     % = \frac{1}{\lambda}\left( \frac{\tilde{\d}}{\Pi_{\ell=1}^N(\tilde{\d}- r_{N, \ell})} - \frac{1}{\tilde{\d}^{N-1}}  \right)
%     % \\
%     = \frac{1}{\lambda}\left( \frac{1}{N}\sum_{\ell = 1}^N\frac{c_{N, \ell}}{\tilde{\d}- r_{N, \ell}} - \frac{1}{\tilde{\d}^{N-1}} \right).
% \end{align*}
The coefficients $c_{N,\ell}$ can be computed by the residue method,
$$
c_{N,\ell} = N \lim_{\y \to r_{N,\ell}} (\y - r_{N,\ell})
\frac{\y}{\y^{N}-\lambda} = N \frac{r_{N,\ell}}{\frac{d}{d\y}(\y^N-\lambda)_{\vert \y = r_{N,\ell}}} = r_{N,\ell}^{2-N} %\frac{r_{N,\ell}}{r_{N,\ell}^{N-1}} 
= \lambda^{\frac{2}{N}-1} e^{4\pi i \ell/N}  .
$$
\revision{Note that $\y(\tcont) > 0$, for all $\tcont \ge 0$, since $\y(0) = \alpha > 0$. Indeed, if there were a time $\tcont_*$ with $\y(\tcont_*) = 0$, then both the trajectory starting at $\y(0) = \alpha$ and the (constant) trajectory starting at $\y(0) = 0$ would lead to $\y(\tcont_*) = 0$. Since the right-hand side of \eqref{eq:IdenticalInitialization_ContinuousScalarDynamics} is locally Lipschitz continuous in $\y$, this however contradicts the local uniqueness of the trajectory guaranteed by Picard–Lindel{\"o}f.
%, since it is unique and continuous by the theorem of Picard-Lindelöf and $\y(t) \leq 0$ for some $t$ would imply that $\y(t') = 0$ for some $t'$. But  
%in particular, $\y(t)$ is real-valued, 
We thus} only need to consider 
$z \in \RR_+ = \{z \in \RR: z > 0\}$ and therefore
\[
\frac{1}{\z^{N-1}(\z^N - \lambda)} = \frac{1}{\lambda} \left(\frac{1}{N} \sum_{\ell=1}^N \Re\left(\frac{c_{N,\ell}}{z-r_{N,\ell}}\right) - \frac{1}{\z^{N-1}}\right).
\]
Since the integration interval is a subset of $\RR_+$ we have 
\begin{align*}
- \int_\alpha^{\y} \Re\left( \frac{c_{N,\ell}}{z-r_{N,\ell}}\right) dz & = \Re\left( c_{N,\ell} \ln(r_{N,\ell}- \y)\right) - \Re\left( c_{N,\ell} \ln(r_{N,\ell}-\alpha)\right) \\ 
& = \lambda^{\frac{2}{N}-1} \Re\left(e^{4\pi i \ell/N} \left(\ln\left(\lambda^{\frac{1}{N}} e^{2\pi i \ell/N} - \y\right) - \ln\left(\lambda^{\frac{1}{N}} e^{2\pi i \ell/N} - \alpha\right)\right)\right)\\
& = \lambda^{\frac{2}{N}-1} \Re\left(e^{4\pi i \ell/N} \left(\ln\left(e^{2\pi i \ell/N} - \frac{\y}{\lambda^{\frac{1}{N}}}\right) - \ln\left(e^{2\pi i \ell/N} - \frac{\alpha}{\lambda^{\frac{1}{N}}}\right)\right)\right)
\end{align*}
%Note that there is no issue with taking the right branch of the complex logarithm in last equality since we only consider the real part.
Hence, for $N\geq 3$ the solution $\y(t)$ 
%of \eqref{} 
satisfies
\begin{align*}
t & = \frac{\lambda^{\frac{2}{N}-2}}{N} \sum_{\ell=1}^{N} \Re \left( e^{\frac{4\pi i\ell}{N}} \left( \ln\left( e^{\frac{2\pi i \ell}{N}}- \y(t) \lambda^{-\frac{1}{N}} \right) - \ln\left(  e^{\frac{2\pi i \ell}{N}} - \alpha \lambda^{-\frac{1}{N}} \right)\right)\right)\\
& \quad +   \frac{1}{\lambda(N-2)} \left( \frac{1}{\alpha^{N-2}} - \frac{1}{\y(t)^{N-2}} \right).
\end{align*}
If $N=2$ then the roots $r_{2,1} = - \lambda^{1/2}$, $r_{2,2} = \lambda^{1/2}$ are real and %a short computation shows that 
the solution $\y(t)$ satisfies
\[
 t = \frac{1}{\lambda}\left( \frac{1}{2} \ln\left(\frac{\alpha + \lambda^{1/2}}{\y + \lambda^{1/2}}\right) + \frac{1}{2} \ln\left(\frac{\alpha - \lambda^{1/2}}{\y - \lambda^{1/2}}\right) + \ln\left(\frac{\y}{\alpha}\right)\right)
= \frac{1}{2\lambda} \ln \left(\frac{\frac{\lambda}{\alpha^2} - 1}{\frac{\lambda}{\y(t)^2} - 1}\right).
\]
By the definition of $U^+_N$ this completes the proof.
\end{proof}
% Observe that $r_{N,N-\ell} = \overline{r_{N,\ell}}$ and $c_{N,N-\ell} = \overline{c_{N,N-\ell}}$ and hence
% $$
% f(y):= \frac{c_{N,\ell}}{\y - r_{N,\ell}} + \frac{c_{N,N-\ell}}{\y - r_{N,N-\ell}} = \frac{c_{N,\ell}(\y - \overline{r_{N,\ell}}) + \overline{c_{N,\ell}}(\y - r_{N,\ell})}{|\y-r_{N,\ell}|^2} =
% 2 \frac{\Re(c_{N,\ell}) \y - \Re(c_{N,\ell} \overline{r_{N,\ell}})}
% {\y^2 - 2 \Re(r_{N,\ell}) \y + |r_{N,\ell}|^2}.
% $$
% Finding the solution of the separable differential equation \eqref{eq:IdenticalInitialization_ContinuousScalarDynamics} amounts to integrating the right hand side of \eqref{eq:IdenticalInitialization_ContinuousScalarDynamics} with respect to $\y$. 

The second lemma required for the proof of Theorem~\ref{thm:IdenticalInitialization_ConvergenceRate} links the continuous dynamics in \eqref{eq:IdenticalInitialization_ContinuousScalarDynamics} to the discrete dynamics in \eqref{eq:IdenticalInitialization_ScalarDynamics}. Although a result of this form might exist already, we include the full proof for the reader's convenience.
\begin{lemma}\label{lemma:ODE_Increasing_ConvexSolution}
    For a continuously differentiable function $f\colon \mathbb{R} \rightarrow \mathbb{R}$, consider the (unique) solution $\y(\tcont)$ of the differential equation $\y'(\tcont) = f(\y(\tcont)), \y(0) =\alpha > 0$. Let $I \subset \mathbb{R}$ be an interval with $\alpha \in I$ and $T > 0$ such that $\y(t) \in I$ for all $t \in [0,T]$ (here $T = \infty$ is allowed). Assume that
    $$
    |f(x)| \leq K_1,  \quad |f'(x)| \leq K_2, \quad \mbox{ and } f'(x) f(x) \geq 0 \quad \mbox{ for all } x \in I.
    $$
    %Suppose $|f| < K_1 $ and $|f'| < K_2 $. 
    Fix $0<\eta < \frac{1}{K_2}$. For $\tdisc\in\mathbb{N}_0$, define
    \begin{align*}
        \d^*(\tdisc) &:= \y(\eta\tdisc),\\
        \d(\tdisc+1) &:= \d(\tdisc) + \eta f(\d(\tdisc)), \quad \d(0) = \alpha.
    \end{align*}
    %with $\d^*(0)=\d(0)=\alpha$. 
    %If $\tilde{\d}''(\tcont) = f'(\tilde{\d}(\tcont)) f(\tilde{\d}(\tcont)) > 0$ for all $\tcont \geq 0$, then
    Then, for all $0 \leq \tdisc \leq T/\eta$,  
    \begin{equation*}
        \d(\tdisc)\leq \d^*(\tdisc).
        %, \quad \mbox{ for all } \tdisc \in \NN_0 \mbox{ with }\; \eta \tdisc \leq T.
    \end{equation*}
    Let $s = \left\lceil \frac{K_1}{f(\alpha)} \right\rceil$. If, in addition $f(x) \geq 0$, for all $x \in [\alpha,\sup_{\tdisc \in \NN} \d(\tdisc))$, then
    \begin{equation*}
        \d^*(\tdisc)\leq \d(\tdisc+s) \quad \mbox{ for all $\tdisc$ such that } \d^*(\tdisc) \in I. 
        %\mbox{ such that } \d(\tdisc + s) \in I.
    \end{equation*}
    %where $s = \left\lceil \frac{K_1}{f(\alpha)} \right\rceil$.
\end{lemma}
\begin{proof}
We first show by induction that $\d(\tdisc)\leq \d^*(\tdisc)$ for $0 \leq \tdisc \leq T/\eta$. The claim clearly holds for $\tdisc = 0$. Now assume that it holds for some $\tdisc$ such that $\eta(\tdisc + 1) \leq T$. We aim at proving the claim for $\tdisc + 1$.
Since $f(x) f'(x) \geq 0$ for $x \in I$ and $\y (\tcont) \in I$ for $\tcont \in [0,T]$ we have
\[
\y''(\tcont) = \frac{d}{d\tcont} f(\y(t)) = f'(\y(t)) f(\y(t)) \geq 0. 
\]
Hence, $\y$ is convex on $[0,T]$ and therefore,
$\y(\eta (k+1)) \geq \y(\eta k) + \eta \y'(\eta k) = \y(\eta k) + \eta f(\y(\eta k))$ and $\y(\eta k) \geq \y(\eta(k+1)) - \eta f(\y(\eta(k+1)))$ so that by definition of $\d^*$
\begin{align}
        \eta f(\d^*(\tdisc)) \leq \d^*(\tdisc+1) - \d^*(\tdisc) \leq \eta f(\d^*(\tdisc+1))
        \label{eq:lemma_IdenticalInitialization_PositiveConvergenceRate_Integral}.
\end{align}
%Since $|f'(x)| \leq K_2$ for $x \in I$, the function $f$ is Lipschitz continuous on $I$ with Lipschitz constant $K_2$. 
The definition of $\d$ and the mean-value theorem together with \eqref{eq:lemma_IdenticalInitialization_PositiveConvergenceRate_Integral} 
%and the definition of $\d$ this implies
and $|f'(x)|\leq K_2$ for all $x \in I$
imply that, for some $\xi$ between $\d(\tdisc)$ and $\d^*(\tdisc)$,
 \begin{align*}
         \d^*(\tdisc+1) - \d(\tdisc+1)
        &\geq \d^*(\tdisc)-\d(\tdisc)  + \eta (f(\d^*(\tdisc)) - f(\d(\tdisc)) )\\
        & =  \d^*(\tdisc)-\d(\tdisc) +\eta f'(\xi)(\d^*(\tdisc)) - \d(\tdisc)) 
        = (1+\eta f'(\xi))(\d^*(\tdisc)) - \d(\tdisc))\\
        & \geq (1 - \eta K_2)(\d^*(\tdisc)) - \d(\tdisc)) \geq 0.
        %\geq (\d^*(\tdisc) -\d(\tdisc)) (1 - \eta K_2)
        %\geq 0.
    \end{align*}
In the second inequality we have used the induction hypothesis that $\d^*(\tdisc)-\d(\tdisc)\geq 0$. This proves the first part of the lemma.

\revision{As in the proof of Lemma \ref{lemma:IdenticalInitialization_ContinuousScalarDynamics}, we can argue by Picard–Lindel{\"o}f that $f(d^*(\tdisc)) > 0$ for any $\tdisc \ge 0$ since $\y(0) = \alpha > 0$. Indeed, if there were a time $\tcont_* = \eta \tdisc_*$ with $f(d^*(\eta\tdisc_*)) = 0$, then both the trajectory starting at $\y(0) = \alpha$ and the (constant) trajectory starting at $\y(0) = 0$ would lead to $f(d^*(\tcont_*)) = \y(\tcont_*) = 0$. Since $f$ is locally Lipschitz continuous in $\y$ by assumption, this however contradicts the local uniqueness of the trajectory guaranteed by Picard–Lindel{\"o}f.} 
Let us now assume that in addition $f(x)> 0$ for all $x \in (0,\sup_{\tdisc \in \NN} \d(\tdisc))$ and show the second claim that $\d^*(\tdisc) \leq \d(\tdisc + s)$, for all $\tdisc$ with $\d^*(\tdisc) \in I$. First, consider the case $\d(\tdisc+s) \in I$. We show the stronger statement that, for all $\tdisc$ with $\d^*(\tdisc) \in I$,
\begin{equation}\label{eq:stronger-claim}
\d^*(\tdisc) \leq \d(\tdisc + s) - \eta K_1
\end{equation}
via induction. Note that $f$ is monotonously increasing on $I$ because $f(x) > 0$ and 
$f'(x)f(x) \ge 0$ by assumption. Since  $\d(s) \ge d(0) + s \eta f(\alpha) \ge d(0) + \eta K_1$ (by the monotonicity of $f$ as well as the definition of $\d$ and $s$), we have
\begin{align*}
    \d^*(0) = \d(0) \le \d(0 + s) - \eta K_1.
\end{align*}
Let us now assume the claim \eqref{eq:stronger-claim} holds for $\tdisc$ with $\d(\tdisc+s) \in I$ and $\d^*(\tdisc) \in I$. If $\d^*(\tdisc+1), d(\tdisc + s + 1) \in I$,
%since $\y$ is convex and $\d^*(\tdisc) \le \d(\tdisc+s) - \eta K_1$,
the right hand inequality in \eqref{eq:lemma_IdenticalInitialization_PositiveConvergenceRate_Integral} implies that
\begin{align*}
    \alpha \le \d^*(\tdisc+1) 
    \le \d^*(\tdisc) + \eta f(\d^*(\tdisc+1))
    \le \d^*(\tdisc) + \eta K_1
    \le \d(\tdisc+s),
\end{align*}
which implies $\d^*(\tdisc) + \eta K_1 \in I$. %\textcolor{red}{(Do we need $I$ to be an interval to conclude this?)}
%so that
% monotonicity of $f$ holds. 
%\]
Hence, by the monotonicity of $f$ on $I$ we obtain
\begin{align*}
    \d^*(\tdisc+1) 
    %&= \y(\eta \tdisc + \eta) \le \y(\eta \tdisc) + \eta \y'(\eta \tdisc + \eta) =
    & \leq \d^*(\tdisc) + \eta f(\d^*(\tdisc+1)) 
     \le \d(\tdisc+s) + \eta f(\d^*(\tdisc+1)) - \eta K_1\\ 
    & \le \d(\tdisc+s) + \eta f(\d^*(\tdisc) + \eta K_1) - \eta K_1 
    \le \d(\tdisc+s) + \eta f(\d(\tdisc+s)) - \eta K_1\\ 
    & = \d((\tdisc+1)+s) - \eta K_1.
\end{align*}
%where we used in the third and fourth line 
%that $d^*(\tdisc+1) \in I$ by assumption 
This completes the induction step and, hence, proves the inequality $\d^*(\tdisc) \leq \d(\tdisc+s)$ whenever $\d(\tdisc+s) \in I$. 

Finally, consider the case that $\tdisc$ is such that $\d^*(\tdisc) \in I$ but $\d(\tdisc + s)$ is not contained in $I$. Then we must have $\d(\tdisc+s) \geq \sup I$ since $f(x) \geq 0$ for all $x \in [\alpha,\sup_{\tdisc \in \NN} \d(\tdisc))$. 
It follows that $\d^*(\tdisc) \leq \d(\tdisc+s)$ also in this case.
\end{proof}
\begin{proof}[Proof of Theorem~\ref{thm:IdenticalInitialization_ConvergenceRate}]
We first consider the case that $\lambda < 0$.
%Then the condition \eqref{eq:cond:stepsize} on the stepsize $\eta$ implies that $\eta < \left( (3N-2) \max\{\alpha,|\lambda|^{1/N}\}^{2N-2} \right)^{-1}$ in both the cases $\alpha < |\lambda|^{\frac{1}{N}}$ and $\alpha \geq |\lambda|^{\frac{1}{N}}$. 
By our assumption \eqref{eq:cond:stepsize} on the stepsize $\eta$ the conditions in Lemma~\ref{lemma:IdenticalInitialization_Convergence} are satisfied so that $\d(\tdisc)$ is monotonically decreasing and $\d(\tdisc) \in [0,\alpha]$, for all $\tdisc \ge 0$.
Let $g_{-}(x) = x + \eta \lambda x^{N-1}$ and define $\d_-(\tdisc)$ as 
\begin{equation*}
    \d_-(\tdisc+1)
    = g_{-}(\d_-(\tdisc)) = \d_-(\tdisc) + \eta\lambda\d_-(\tdisc)^{N-1},
    \quad \d_-(0) = \alpha.
\end{equation*}
Note that for $x \geq 0$ and $\lambda < 0$ it holds $g(x) := x - \eta x^{N-1}(x^N - \lambda) \leq g_{-}(x)$. Recalling that $d(\tdisc + 1) = g(\d(\tdisc))$, we inductively conclude that $\d(\tdisc) \leq \d_-(\tdisc)$, for all $\tdisc \ge 0$. Let $\y_-$ be the continuous analog of $\d_-$, i.e., the solution of the differential equation
\begin{align*}
    \y_-'(\tcont) = \lambda \y_- (\tcont)^{N-1}, \quad \y_-(0) = \alpha,
\end{align*} 
which is explicitly given by
\begin{equation}\label{y-minus-explicit}
    \y_-(\tcont) =
    \begin{cases}
    \alpha e^{\lambda \tcont}
    &\text{ if }N = 2\\
    [(2-N)\lambda \tcont + \alpha^{-(N-2)}]^{-\frac{1}{N-2}}
    &\text{ if }N\geq 3.
    \end{cases}
\end{equation}
Define the discrete variable $\d_-^*(\tdisc) := \y_-(\eta\tdisc)$. We intend to apply Lemma~\ref{lemma:ODE_Increasing_ConvexSolution} for $I = [0,\alpha]$, $f(x) = \lambda x^{N-1}$. Note that $f(x) f'(x) = \lambda^2 (N-1) x^{2N-1} \geq 0$, $|f(x)| \leq |\lambda| \alpha^{N-1} =: K_1$ and $|f'(x)| \leq |\lambda| (N-1) \alpha^{N-2} =: K_2$ for $x \in I$. 
The assumption \eqref{eq:cond:stepsize} on the stepsize implies that $\eta < 1/K_2$
%$\lambda < 0$, the stepsize $\eta$ satisfies by assumption
%\[
%\eta < \frac{1}{(3N-2) \alpha^{2N-2}\}} 
%\]
so that Lemma~\ref{lemma:ODE_Increasing_ConvexSolution}
yields %$\d(\tdisc) \leq 
$\d_-(\tdisc) \leq \d_-^*(\tdisc)$, for all $\tdisc \ge 1$. Hence, $\d_-(\tdisc) \leq \d_-^*(\tdisc) = \y_-(\eta \tdisc) \leq \epsilon$ for $\tdisc \geq \frac{1}{\eta |\lambda| } T_N^-(\epsilon,\alpha)$ by the definition of $T_N^-(\epsilon,\alpha)$ and a short calculation using \eqref{y-minus-explicit}.
We conclude that $T \leq \frac{1}{\eta |\lambda| } T_N^-(\epsilon,\alpha)$.

We now consider the case $0 \leq \lambda < \alpha^N$. 
By Lemma~\ref{lemma:IdenticalInitialization_Convergence}, $\d(\tdisc)$ is monotonically decreasing and $\d(\tdisc) \in [\lambda^\frac{1}{N},\alpha]$, for all $\tdisc \ge 0$. 
For $x \in [\lambda^\frac{1}{N},\alpha]$, the function $f(x) = -x^{N-1}(x^N - \lambda)$ satisfies 
$f(x) f'(x) = (N-1)x^{2N-3}(x^N-\lambda)^2 + N x^{3N-3}(x^N - \lambda) \geq 0$  and
$|f'(x)| = (N-1) x^{N-2}(x^N - \lambda) + N x^{2N-2} \leq (2N-1) \alpha^{2N-2} =: K_2$. By assumption \eqref{eq:cond:stepsize}, the stepsize satisfies $\eta < 1/K_2$.
%Applying Lemma \ref{lemma:ODE_Increasing_ConvexSolution} for $I = [\lambda^\frac{1}{N},\alpha]$, $f(x) = -x^{N-1}(x^N - \lambda)$, and $\eta$ as assumed, thus 
Hence, we can apply Lemma~\ref{lemma:ODE_Increasing_ConvexSolution} 
which gives $\d(\tdisc) \leq \d^*(\tdisc) = \y(\eta \tdisc)$, for all $\tdisc \ge 0$, where $\y$ is defined by \eqref{eq:IdenticalInitialization_ContinuousScalarDynamics}. 
By Lemma \ref{lemma:IdenticalInitialization_ContinuousScalarDynamics} and the observation that $\y$ is monotonically decreasing for $\lambda \ge 0$ and $\lambda < \alpha^N$, we obtain that $\y(\tcont) \le \lambda^\frac{1}{N} + \epsilon$ for all $\tcont \ge %\frac{1}{\eta} 
T_N^+(\lambda,\lambda^\frac{1}{N} + \epsilon,\alpha)$ since $T_N^+(\lambda,\lambda^\frac{1}{N} + \epsilon,\alpha) = U_N^+(\lambda,\lambda^\frac{1}{N} + \epsilon) - U_N^+(\lambda,\alpha)$ and $y$ satisfies \eqref{eq:IdenticalInitialization_ContinuousScalarDynamics}. Consequently, $\d^*(\tdisc) \leq \lambda^{\frac{1}{N}} + \epsilon$ for all $\tdisc \leq \frac{1}{\eta} T_N^+(\lambda,\lambda^\frac{1}{N} + \epsilon,\alpha)$, which proves the claim for the case $0 \leq \lambda < \alpha^N$.

Finally, consider $\lambda > \alpha^N$. The proof distinguishes two phases of the dynamics. In the first phase we use the associated continuous flow
$\y$ in \eqref{eq:IdenticalInitialization_ContinuousScalarDynamics} for the time where it is convex. For the following second phase, we directly work with the discrete dynamics.
In order to make this distinction we use again the function $f(x) = -x^{N-1}(x^N - \lambda)$ so that
$\y'(\tcont) = f(\y(t))$ and $\d(\tdisc+1) = \d(\tdisc) + \eta f(\tdisc)$.
Note that the function
$$h(x) = f(x) f'(x) = (N-1)x^{2N-3}(x^N-\lambda)^2 + N x^{3N-3}(x^N - \lambda)$$ satisfies $h(x) \geq 0$ for $x \in [0,\zeta]$ and $h(x) \leq 0$ for $x \in [\zeta, \lambda^{\frac{1}{N}}]$, where 
\[
\zeta = \left(\frac{\lambda(N-1)}{2N-1}\right)^{\frac{1}{N}} = (c_N \lambda)^\frac{1}{N} 
\]
This implies that $\y$ is convex as long as $\y(\tcont) \in [0,\zeta]$ and concave when $\y(\tcont) \in [\zeta, \lambda^{\frac{1}{N}}]$. % i.e., $\zeta$ is the inflection point of $\y$.  
If $\epsilon > \lambda^{\frac{1}{N}} - \zeta$ then $\alpha < \zeta$ by the assumption $\epsilon \in (0, |\alpha - \lambda_+^{1/N}|)$ and the desired accuracy $\epsilon$ is reached while the dynamics is still in the convex phase, i.e., $\y(\tcont) \in [\alpha, \zeta] \subset [0, \zeta]$. In this case, we define
$$
T_1 := \min\{ \tdisc \in \mathbb{N}_0 :  \lambda^{\frac{1}{N}} - \d(\tdisc) \leq  \epsilon\}, \qquad T_2 := 0.
$$
Now assume $\epsilon \leq \lambda^{\frac{1}{N}} - \zeta$. If $\alpha < \zeta$ then the dynamics starts in the convex phase, while it starts in the concave phase if $\alpha \geq \zeta$. Accordingly, we define
\begin{align*}
    T_1&:= \begin{cases} 
    \min\{\tdisc \in \mathbb{N}_0 : \d(\tdisc) \geq \zeta\} & \mbox{ if } \alpha < \zeta\\ 0 & \mbox{ if } \alpha \geq \zeta\end{cases} \\
    T_2 &:=\min\{\tdisc \in \mathbb{N}_0 : \lambda^{\frac{1}{N}} - \d(\tdisc + T_1) \leq  \epsilon\}.
\end{align*}
%If $\epsilon \leq \lambda^{\frac{1}{N}} - \zeta$ and $\alpha \geq \zeta$ then $\y(t)$ is never in the convex phase and we define

%where $\zeta = (c_N \lambda)^\frac{1}{N} = ((N-1)\lambda/(2N-1))^{\frac{1}{N}}$ is the inflection point of $\y$. 
%We will bound $T_1$ and $T_2$
%(if needed) separately.
%and conclude with $T = T_1 + T_2$.

We start with bounding $T_1$.
If $\alpha \geq \zeta$, then $T_1=0$ and we are done. %(\textcolor{red}{H: We still need to take care about this case in the statement of the Theorem}). 
In the case $\alpha < \zeta$, we intend to apply Lemma \ref{lemma:ODE_Increasing_ConvexSolution} for $I = [\alpha,\zeta]$, and $f(x) = -x^{N-1}(x^N -\lambda)$. Then
$f(x) f'(x)\geq 0$ for $x \in I$ as already noted above
and $|f'(x)| = |(N-1) x^{N-2}(\lambda - x^N) - N x^{2N-1}| \leq \lambda^{2 - \frac{2}{N}} N =: K_2$, for all $x \in [0, \lambda^{\frac{1}{N}}] \supset I$, where the inequality follows similarly as in \eqref{ineq:g-deriv}.
By the assumption \eqref{eq:cond:stepsize} on the stepsize we have $\eta < 1/K_2$.
Hence, by Lemma~\ref{lemma:ODE_Increasing_ConvexSolution} we have 
$
\d(\tdisc) \leq \d^*(\tdisc) \leq \d(\tdisc+s)$, where $$
s = \left\lceil \frac{\max_{x \in [\alpha,\zeta]}|f(x)|}{f(\alpha)} \right\rceil
\leq  \left\lceil \frac{ \zeta^{N-1}|\lambda - \alpha^N|}{\alpha^{N-1}|\lambda - \alpha^N|}\right\rceil
= \left\lceil \left( \frac{N-1}{2N-1} \right)^{1-\frac{1}{N}} \left(\frac{\lambda^{\frac{1}{N}}}{\alpha}\right)^{N-1} \right\rceil = s_N(\lambda, \alpha).
$$ 
Since $\y$ is monotonically increasing for $\lambda > \alpha^N$, Lemma~\ref{lemma:IdenticalInitialization_ContinuousScalarDynamics} implies that $T_1$ is lower bounded by $T_N^+(\lambda, \min\{\zeta, \lambda^{\frac{1}{N}} - \epsilon\},\alpha)$ and upper bounded by $T_N^+(\lambda,\min\{\zeta, \lambda^{\frac{1}{N}} - \epsilon\},\alpha)+s_N(\lambda,\alpha)$.

%TO BE CONTINUED

%\begin{enumerate}
    % \item 
    % Applying Lemma \ref{lemma:ODE_Increasing_ConvexSolution} for $I = [\alpha,\zeta]$, $f(x) = -x^{N-1}(x^N - \lambda)$, and $\eta$ as assumed, $\d(\tdisc) \leq \d^*(\tdisc) \leq \d(\tdisc+s)$, where $s = \left\lceil \frac{\max|f|}{f(\alpha)} \right\rceil$. Since $\y$ is monotonously increasing for $\lambda > \alpha^N$, $T_1$ is lower bounded by
    % $T_N^+(\lambda,\zeta,\alpha)$ and upper bounded by $T_N^+(\lambda,\zeta,\alpha)+s_N(\lambda,\alpha)$.
    %\item 

Now we consider the second phase where $\d(\tdisc) \in[\zeta,\lambda^{\frac{1}{N}}]$ and define $\Delta(\tdisc) = \lambda^{\frac{1}{N}} - \d(\tdisc)$ to be the difference to the limit $\lambda^{\frac{1}{N}}$. 
If $\epsilon \geq \lambda^{\frac{1}{N}} (1- c_N^{1/N}) = \lambda^{\frac{1}{N}} - \zeta$ %then $\Delta(T_1) \leq \epsilon$ so that 
then $T_2 = 0$. Therefore, we assume $\epsilon \geq \lambda^{\frac{1}{N}} (1- c_N^{1/N})$ from now on.
By Lemma \ref{lemma:IdenticalInitialization_Convergence}, $\d(\tdisc)$ is increasing and remains inside the interval $[\zeta,\lambda^{\frac{1}{N}}]$ for all $\tdisc \geq T_1$.
A direct computation gives %$\frac{\Delta(\tdisc+1)}{\Delta(\tdisc)} = 1 - \eta g(\d(\tdisc))$, where
    \begin{equation} \label{eq:Delta-relation}
        \Delta(\tdisc+1) = \left(1 - \eta g(\d(\tdisc))\right) \Delta ( \tdisc), \quad \mbox{ with } \quad
        g(x)
        = \frac{x^{N-1}(x^N - \lambda)}{x - \lambda^{\frac{1}{N}}}
        = x^{N-1}\left(\sum_{j = 1}^{N}\lambda^{1-\frac{j}{N}}x^{j-1}\right).
    \end{equation}
    %which implies that error is decaying exponentially. 
    The last equality follows from a straightforward calculation.
    Note that 
    \[
    g'(x) = \sum_{j=1}^N \lambda^{1-\frac{j}{N}} (N+j-2) x^{N+j-3} \geq  0\quad \mbox{ for } x \geq 0.   
    \]
    In particular, $g$ is increasing on $[\zeta, \lambda^{\frac{1}{N}}]$. Note that, since $c_N < 1$,
    \[
    g(\zeta) = (c_N \lambda)^{\frac{N-1}{N}} \sum_{j=1} \lambda^{1-\frac{j}{N}} (c_N \lambda)^{\frac{j-1}{N}}
    = \lambda^{2-\frac{2}{N}} c_N^{1-\frac{1}{N}} \sum_{j=1}^N c_N^{\frac{j-1}{N}}
    \geq \lambda^{2-\frac{2}{N}} c_N^{1-\frac{1}{N}} N c_N^{\frac{N-1}{N}}
    =
    N \lambda^{2-\frac{2}{N}} c_N^{2-\frac{2}{N}}.
    \]
    Thus
    \begin{equation}  \label{eq:g-inclusion} g(\d(\tdisc))\in
       [ g(\zeta), g(\lambda^{\frac{1}{N}})] 
       \subset \left[
       \left(\frac{N-1}{2N-1}\right)^{2 - \frac{2}{N}}N\lambda^{2-\frac{2}{N}}, 
       N\lambda^{2-\frac{2}{N}}\right]
        =: [r_1,r_2].
    \end{equation}
    It follows that 
    %(1 - \eta r_2)^{\tdisc} \Delta(T_1)
    %\lambda^{\frac{1}{N}} 
    %\leq 
    \begin{equation}\label{DeltaT1:bound}
    \Delta(\tdisc + T_1) \leq (1 - \eta r_1)^{\tdisc}
    \Delta(T_1) \leq (1-\eta r_1)^k (\lambda^{\frac{1}{N}} - \zeta) = \left(1-\eta N (c_N \lambda)^{2- \frac{2}{N}}\right)^k \lambda^{\frac{1}{N}} \left(1 - c_N^{1/N} \right).   %\lambda^{\frac{1}{N}}.
    \end{equation}
    Note that $1 - \eta N (c_N \lambda)^{2 - \frac{2}{N}} \geq 1 - \eta N \lambda^{2 - \frac{2}{N}} > 0$ by the assumption on $\eta$. Hence, $\Delta(\tdisc + T_1) \leq \epsilon$ if 
    $$
    \tdisc \geq \frac{\ln(\lambda^{\frac{1}{N}}/\epsilon)  + \ln \left( 1-c_N^{1/N}\right)}{\left|\ln(1- \eta N (c_N \lambda)^{2-\frac{2}{N}})\right|} =\frac{\ln(\lambda^{\frac{1}{N}}/\epsilon) - a_N}{\left|\ln(1- \eta N (c_N \lambda)^{2-\frac{2}{N}})\right|}
    %= \gamma_N(\lambda,\eta) \widehat{T}_N(\lambda,\epsilon,\eta),
    $$ 
    so that $T_2$ is bounded from above by the right hand side. In the case that $\alpha > \zeta$, we have $T_1=0$ and the inequality \eqref{DeltaT1:bound} can be improved to
    $$
    \Delta(k) \leq \left(1-\eta r_1\right)^k \Delta(0) = \left(1-\eta N (c_N \lambda)^{2-\frac{2}{N}}\right)^k( \lambda^{\frac{1}{N}} - \alpha),  
    $$
    which leads to 
    $$
    T_2 \leq \frac{\ln\left(\frac{\lambda^{\frac{1}{N}}-\alpha}{\epsilon}\right)}{\left|\ln(1- \eta N (c_N \lambda)^{2-\frac{2}{N}})\right|}. 
    $$
    %\leq  \gamma_N(\lambda,\eta) \widehat{T}_N(\lambda,\epsilon,\eta)$.

    For the lower bound, assume first that $\alpha < \zeta$. Then $T_1 \geq 1$. 
    and since $\alpha < \zeta$ it follows that $\d(T_1 - 1) < \zeta$ so that by the assumption \eqref{eq:cond:stepsize} on the stepsize $\eta$,
    \begin{align*}
    \d(T_1) & = \d(T_1-1) + \eta f(d(T_1-1)) \leq \zeta + \eta \max_{x \in [0, \lambda^{\frac{1}{N}}]} |x^{N-1}(x^N-\lambda)| 
    \leq \zeta + \eta \lambda^{\frac{N-1}{N}}\lambda \\
    & \leq (c_N \lambda)^{\frac{1}{N}} + \frac{\lambda^{2-\frac{1}{N}}}{(2N-1) \lambda^{2-\frac{2}{N}}}
    %= (c_N \lambda)^{\frac{1}{N}} +  \frac{\lambda^{\frac{1}{N}}}{2N-1} 
    = \left(c_N^{1/N} + \frac{1}{2N-2}\right) \lambda^{\frac{1}{N}},
    \end{align*}
    Thus, $$
    \Delta(T_1) = \lambda^{1/N} - \d(T_1) \geq \left( 1- c_N^{1/N} - \frac{1}{2N-2}\right) \lambda^{1/N} = \left(\frac{2N-1}{2N-2} - \left(\frac{N-1}{2N-1}\right)^{\frac{1}{N}}\right) \lambda^{\frac{1}{N}}.
    $$ 
    Observe that by \eqref{eq:Delta-relation} and \eqref{eq:g-inclusion},
    for all $\tdisc \geq 0$,
    $$
    \Delta(\tdisc + T_1) \geq (1 - \eta r_2)^{\tdisc} \Delta(T_1)
    \geq (1 - \eta r_2)^{\tdisc} \geq (1 - \eta r_2)^{\tdisc} \left(\frac{2N-1}{2N-2} - \left(\frac{N-1}{2N-1}\right)^{\frac{1}{N}}\right) \lambda^{\frac{1}{N}}.
    $$
    Hence, $\Delta(\tdisc + T_1) \geq \epsilon$ for all 
    \[
    \tdisc \leq \frac{\ln\left(\lambda^{\frac{1}{N}}/\epsilon\right) + \ln\left(\frac{2N-1}{2N-2} - \left(\frac{N-1}{2N-1}\right)^{\frac{1}{N}}\right)}{\left|\ln(1- \eta N \lambda^{2 - \frac{2}{N}})\right|} =
    \frac{\ln\left(\lambda^{\frac{1}{N}}/\epsilon\right) - b_N}{\left|\ln(1- \eta N \lambda^{2 - \frac{2}{N}})\right|}.
    %\widehat{T}_N(\lambda,\epsilon, \eta).
    \]
    This implies that $T_2$
    is lower bounded by the right hand side above if
    %\geq \widehat{T}_N(\lambda,\epsilon, \eta)$ 
    if $\alpha < \zeta$.
    
    If $\alpha > \zeta$ then 
    $T_1 = 0$ and
    $
    \Delta(\tdisc) \geq (1-\eta r_2)^k \Delta(0) =  (1-\eta r_2)^k (\lambda^{\frac{1}{N}} - \alpha).$
    Hence, $\Delta(\tdisc) \geq \epsilon$ for all 
    \[
    k \leq \frac{\ln\left(\frac{\lambda^{1/N} - \alpha}{\epsilon}\right)}{\left| \ln\left(1-\eta N \lambda^{2-\frac{2}{N}}\right)\right|}. %= \widehat{T}_N( ).
    \]
    Hence, $T_2$ is bounded from below by the right hand side of the above inequality.
    
    Noting that $T=T_1 + T_2$, collecting all the cases and comparing with the definition of $T^\Id_N(\lambda,\epsilon,\alpha,\eta)$ and $s_N(\lambda,\alpha)$ completes the proof.
    %\geq \widehat{T}_N()$.
    %and if $\alpha \leq \zeta$ 
    %this implies
    % $
    %(1 - \eta r_2)^{\tdisc+1}
    %(\lambda^{1/N} - \zeta)
    %\leq \Delta(\tdisc + T_1).
    %$
    %Hence, for $\epsilon < \lambda^{1/N} - c_N^{1/N}$, the requirement
    %  $\Delta(T_1 + \tdisc) \leq \epsilon$ necessarily implies that
    % \[
    % \tdisc + 1 \geq \frac{\ln(\epsilon) - \ln \left(\lambda^{1/N} \right) - \ln \left( 1-c_N^{1/N}\right)}{\ln(1- \eta N  \lambda^{2-\frac{2}{N}})} =  \widehat{T}_N(\lambda,\epsilon,\eta).
    % \]
    % Therefore,
    % $T_2$ is lower bounded by
    %$T_2 :=\frac{\ln(\epsilon) - \ln(\lambda^{\frac{1}{N}})}{\ln(1 - \eta r_2)}$ 
    % $\widehat{T}_N(\lambda,\epsilon,\eta) - 1$ and upper bounded by %$\frac{\ln(\epsilon) - \ln(\lambda^{\frac{1}{N}})}{\ln(1 - \eta r_1)} = \frac{\ln(1 - \eta r_2)}{\ln(1 - \eta r_1)} T_2(\lambda,\epsilon) =\gamma(\lambda) T_2(\lambda,\epsilon)$
    % $\gamma_N(\lambda,\eta) \widehat{T}_N(\lambda,\epsilon,\eta)$.
%\end{enumerate}
\end{proof}

%%%%%%% To be completed %%%%%%%
% \textcolor{red}{H: We might slightly improve the upper bound by using the exact expression
% $$g(\zeta) = \lambda^{2-\frac{2}{N}} \frac{c_N^{1/N}}{1-c_N^{1/N}} \frac{N}{2N-1}.$$}
%%%%%%% To be completed %%%%%%%

%

%Unfortunately, Lemma \ref{lemma:IdenticalInitialization_ConvergenceRate} implies that negative eigenvalues cannot be recovered via identical initialization. Therefore we will now march towards the perturbed initialization setting.

\subsubsection*{Proof of Theorem \ref{theorem:IdenticalInitialization_ContinuousNonAsymptoticRate}}

%Theorem~\ref{theorem:IdenticalInitialization_ContinuousNonAsymptoticRate} 
The statement is an immediate consequence of Lemma~\ref{lemma:IdenticalInitialization_MatrixDynamics} and Theorem~\ref{thm:IdenticalInitialization_ConvergenceRate} in Section \ref{subsec:IdenticalInitialization}, also noting that $\|\Wstar \| = \max_{i \in [n]} |\lambda_i|$. In particular, Theorem \ref{thm:IdenticalInitialization_ConvergenceRate} yields that, for $\epsilon\in(0,|\lambda_i|^{\frac{1}{N}})$,
\begin{equation*}
|E_{ii}(\tdisc)| \leq 
\begin{cases} 
|\lambda_i - (\lambda_i^{\frac{1}{N}}- \epsilon)^N|  & \mbox{ if } \lambda_i > 0,\\
\epsilon^N & \mbox{ if } \lambda_i \leq 0.
\end{cases}
\end{equation*}
%\textcolor{red}{}. 
The case $\lambda_i > 0$ in \eqref{eq:error-estimate} follows from the mean-value theorem applied to the function $h(\epsilon) = (\lambda_i^{\frac{1}{N}} - \varepsilon)^N$. To be precise, there exists $\xi \in (0,\varepsilon)$ such that
\begin{equation*}
    |\lambda_i - (\lambda_i^{\frac{1}{N}} - \epsilon)^N|
    = |h(0)-h(\epsilon)|
    = |\epsilon h'(\xi)|
    = |-\epsilon N(\lambda_i^{\frac{1}{N}} - \xi)^{N-1}|
    \leq \epsilon N\lambda_i^{1-\frac{1}{N}}.
\end{equation*}
% where we used that $h'(\xi) = -N(\lambda_i^{\frac{1}{N}} - \xi)^{N-1}$ and $|\lambda_i^{\frac{1}{N}} - \xi| \leq \max (\lambda_i^{\frac{1}{N}}, \epsilon)$.

% In fact for $\xi \in (0,\epsilon)$ we have $h'(\xi) = -N(\lambda_i^{\frac{1}{N}} - \xi)^{N-1}$ so that $|h'(\xi)| \leq N (\lambda_i^{\frac{1}{N}})^{N-1}$ and 
% \[
% |\lambda_i - (\lambda_i^{\frac{1}{N}} - \epsilon)^N| = |h(0)-h(\epsilon)| = |\epsilon h'(\xi)| \leq \epsilon N \lambda_i^{1- \frac{1}{N}}.
% \]

%-------------------------------------------------------------

\subsection{Perturbed Identical Initialization} 
\label{subsec:PerturbedInitialization}

For $N\geq 2$, we have seen in the previous section that we cannot recover negative eigenvalues $\lambda$ of the ground truth matrix with gradient descent when identically initializing all matrices as $W_1(0) =  \cdots = W_N(0) = \alpha I$. 
As already mentioned before, this problem can be overcome 
by slightly perturbing the constant $\alpha$ at one of the $N$ matrices.
%perturbing the initialization allows in contrast to the findings of Section \ref{subsec:IdenticalInitialization} to recover the full spectrum of the ground truth, while still enjoying the benefit of implicit bias. 
Instead of \eqref{eq:MatrixFactorization_IdenticalInitialization} we thus consider the mildly perturbed initialization
\begin{equation}
\label{eq:PerturbedInit}
    \W_j(0)
    = \begin{cases}
    (\alpha - \beta)I & \text{ if }j=1\\
    \alpha I & \text{ otherwise},
    \end{cases}
\end{equation}
for some $ 0 < \beta < \alpha$ (where one could think of $\beta$ much smaller than $\alpha$). 
We will call \eqref{eq:PerturbedInitialization_MatrixInitialization} \emph{perturbed identical initialization}. The choice of $j=1$ for perturbing the constant $\alpha$ to $(\alpha - \beta)$ is generic. In light of the fact that slightly perturbing a single factor $\W_j(0)$ suffices to recover the full spectrum, the spectral cut-off phenomenon in Theorem \ref{theorem:IdenticalInitialization_ContinuousNonAsymptoticRate} appears to be a pathological case. 

We can now turn to the proof of Theorem \ref{theorem:PerturbedInitialization_NonAsymptoticConvergence}. Before stating the formal argument, let us provide a rough intuition on why a slight perturbation like in \eqref{eq:PerturbedInit} makes such a difference: all squared singular values of all factors $W_j$ converge to the $N$-th root of the squares of the respective ground-truth eigenvalues, i.e., their limits coincide in absolute value. At the same time the gradient descent dynamics induce a repelling effect between the eigenvalues of differently initialized factors if the corresponding ground-truth eigenvalue is negative. The only way all factors can converge in this situation is that the perturbed factor converges to the negative $N$-th root of the ground-truth eigenvalue while the eigenvalues of all remaining factors stay positive.
We begin by stating a modified version of the decoupling in Lemma~\ref{lemma:IdenticalInitialization_MatrixDynamics}.

% \eqref{eq:PerturbedInitialization_MatrixInitialization} \textcolor{blue}{(Maybe we can just move the equation \eqref{eq:PerturbedInitialization_MatrixInitialization} to here?)}, i.e., for $ 0 < \beta < \alpha$ we set
% $W_1(0) = (\alpha- \beta) \Id$ and $W_2(0) = \cdots = W_N(0) = \alpha \Id$. We call this perturbed identical initialization.

\begin{lemma}\label{lemma:PerturbedInitialization_MatrixDynamics}
Let $N \geq 2$ and $\W_j(\tdisc), \tdisc \in \mathbb{N}_0, j=1,\hdots,N$ 
be the solution to the gradient descent \eqref{eq:MatrixFactorization_GradientDescent} with perturbed initialization \eqref{eq:PerturbedInitialization_MatrixInitialization}. Let $\widehat{W} = \V \Dground \V^T$ be an eigendecomposition of $\Wstar$. Then the matrices $\D_j(\tdisc):= \V^\tT\W_j(\tdisc)\V$, $j=1,\hdots,N$ are real, diagonal, and follow the coupled  dynamics
\begin{align}
\label{eq:PerturbedInitialization_MatrixDynamics}
  %  \begin{cases}
  \begin{split}
    \D_1(\tdisc+1) & = \D_1(\tdisc) - \eta \D_2(\tdisc)^{N-1}(\D_1(\tdisc)\D_2(\tdisc)^{N-1}-\Dground), \quad \tdisc \in \mathbb{N}_0, \\
    \D_2(\tdisc+1) & = \D_2(\tdisc) - \eta \D_1\D_2(\tdisc)^{N-2}(\D_1(\tdisc)\D_2(\tdisc)^{N-1}-\Dground), \quad ~\tdisc \in \mathbb{N}_0, \\
    \D_j(\tdisc + 1) & = \D_2(\tdisc + 1) \qquad  \mbox{ for } j = 3, \hdots, N, ~\tdisc \in \mathbb{N}_0.
    \end{split}
   % \end{cases}
\end{align}
The last case is empty for $N=2$.
\end{lemma}
The proof of Lemma \ref{lemma:PerturbedInitialization_MatrixDynamics} follows the lines of Lemma \ref{lemma:IdenticalInitialization_MatrixDynamics} and is thus omitted. Similar to (but not quite the same as) the case of identical initialization, the system can be reduced to the scalar dynamics
\begin{equation}\label{eq:PerturbedInitialization_ScalarDynamics}
    \begin{split}
        \d_1(\tdisc+1) & = \d_1(\tdisc) - \eta \d_2(\tdisc)^{N-1}(\d_1(\tdisc)\d_2(\tdisc)^{N-1}-\lambda),\\
        \d_2(\tdisc+1) & = \d_2(\tdisc) - \eta \d_1(\tdisc)\d_2(\tdisc)^{N-2}(\d_1(\tdisc)\d_2(\tdisc)^{N-1}-\lambda), \quad \tdisc \in \mathbb{N}_0
    \end{split}
\end{equation}
with the perturbed initialization
\begin{equation}\label{eq:PerturbedInitialization_ScalarInitialization}
    \d_1(0) = \alpha - \beta >0, \quad \d_2(0) = \alpha>0.
\end{equation}
To further analyze the perturbed setting, we concentrate on three quantities: the difference $\Delta_1(\tdisc)$, the  difference of squares $\Delta_2(\tdisc)$, and the rate factor $\kappa(\tdisc)$, defined as
\begin{equation}\label{def:Diff_SquareDiff_RateFactor}
    \Delta_1(\tdisc) := \d_2(\tdisc) - \d_1(\tdisc),
    \quad \Delta_2(\tdisc) := \d_2(\tdisc)^2 - \d_1(\tdisc)^2,
    \quad \kappa(\tdisc) := \d_2(\tdisc)^{N-2}(\d_1(\tdisc)\d_2(\tdisc)^{N-1} - \lambda)
\end{equation}
so that
%(and the equations below)
\begin{align}
    \d_1(\tdisc+1)&=\d_1(\tdisc)-\eta\d_2(\tdisc)\kappa(\tdisc) \label{d1:kappa}\\
    \d_2(\tdisc+1)&=\d_2(\tdisc)-\eta\d_1(\tdisc)\kappa(\tdisc).\label{d2:kappa}
\end{align}
\begin{lemma}\label{lemma:DeltaContraction}
    Let $\eta > 0$,  $\d_1(\tdisc),\d_2(\tdisc)$ be defined by \eqref{eq:PerturbedInitialization_ScalarDynamics} with the perturbed identical initialization \eqref{eq:PerturbedInitialization_ScalarInitialization}, and let $\Delta_1,\Delta_2$, and $\kappa$ be the quantities defined in \eqref{def:Diff_SquareDiff_RateFactor}. %Suppose 
    Then
    \begin{equation}\label{Delta1_Delta2}
         \Delta_1(\tdisc+1) = 
     [ 1 + \eta \kappa(\tdisc)] \Delta_1(\tdisc), \quad \mbox{ and } \quad \Delta_2(\tdisc+1) = [ 1 - \eta^2 \kappa^2(\tdisc)] \Delta_2(\tdisc).
    \end{equation}
    %that $\eta>0$ 
    %is sufficiently small such that
    Consequently, if for some $\tdisc \in \mathbb{N}_0$ it holds that $\eta|\kappa(\tdisc)|<1$, %for all $\tdisc \in \mathbb{N}_0$. 
    then $\operatorname{sign}(\Delta_1(\tdisc+1)) = \operatorname{sign}(\Delta_1(\tdisc))$, $\operatorname{sign}(\Delta_2(\tdisc+1)) = \operatorname{sign}(\Delta_2(\tdisc))$ %preserve their signs for all $\tdisc \in \mathbb{N}_0$, 
    and $|\Delta_2(\tdisc+1)| \leq |\Delta_2(\tdisc)|$. %is decreasing in $\tdisc$. 
    Moreover, if additionally $\kappa(\tdisc) <0$, 
    %for all $\tdisc \in \mathbb{N}_0$, 
    then $|\Delta_1(\tdisc+1)| < |\Delta_1(\tdisc)|$; 
    %is decreasing; 
    if $\kappa(\tdisc) >0$,
    %for %all $\tdisc \in \mathbb{N}_0$, 
    then $|\Delta_1(\tdisc+1)| > |\Delta(\tdisc)|$.
    %is increasing in $\tdisc$.
\end{lemma}
\begin{proof}
A simple calculation gives
\begin{align*}
    \Delta_1(\tdisc+1) &= 
    \d_2(\tdisc) - \d_1(\tdisc) - \eta \left(\d_1(\tdisc) \d_2(\tdisc)^{N-2} - \d_2(\tdisc)^{N-1}\right)\left(\d_1(\tdisc) \d_2(\tdisc)^{N-1} - \lambda\right) \\
    &= 
    [ 1 + \eta \kappa(\tdisc)] \Delta_1(\tdisc),
\end{align*}
and similarly
%\begin{align*}
    $\Delta_2(\tdisc+1) = [ 1 - \eta^2 \kappa^2(\tdisc)] \Delta_2(\tdisc).$
%\end{align*}
This completes the proof.
%These relations easily imply the claims of the lemma.
\end{proof}

\begin{remark}
    Observe that $\Delta_2(0) = \alpha^2 - (\alpha - \beta)^2 = \beta(\beta + 2 \alpha)$ so that
    $\Delta_2(\tdisc)$ remains small
    for all $\tdisc \in \mathbb{N}$ if $\beta$ is small and $|\eta \kappa(k)| < 1$ for all $\tdisc \in \mathbb{N}$ (we see below that the second condition holds). 
    This property is sometimes referred to as the "balancedness condition" \cite{bah2019learning,Du2018AlgorithmicRI}, especially when analyzing gradient flow. In contrast, $\Delta_1$ does not necessarily remain small due to its dependence on the sign of $\kappa$. In fact, it is possible that $\lim_{\tdisc\to\infty}\d_1 = -\lim_{\tdisc\to\infty}\d_2$, cf.\ Figure \ref{fig:Perturbed}. Compared to the case of identical initialization, such flexibility (in signs) allows to recover the full spectrum of the matrix instead of just the positive eigenvalues.
\end{remark}

\begin{figure}
\begin{subfigure}[c]{0.45\textwidth}
\includegraphics[width=\textwidth]{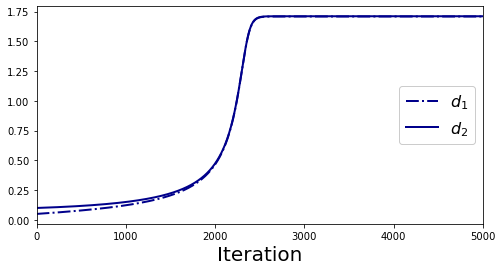}
\subcaption{Approximation of positive eigenvalue ($\lambda~=~5$).}
\label{fig:PerturbedA}
\end{subfigure} \quad
\begin{subfigure}[c]{0.45\textwidth}
\includegraphics[width=\textwidth]{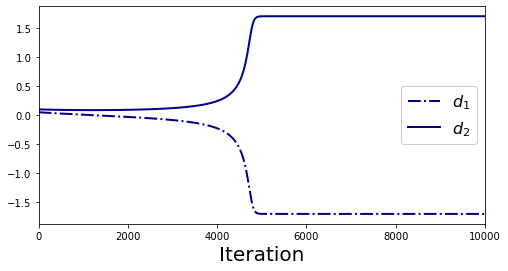}
\subcaption{Approximation of negative eigenvalue ($\lambda~=~-5$).}
\label{fig:PerturbedB}
\end{subfigure}
\caption{Comparison of the gradient descent dynamics under perturbed initialization. Depicted is the evolution of $\d_1$ and $\d_2$ for positive and negative $\lambda$. At the point $k_0$, defined in \eqref{def:PahseTransitionTime}, the dashed line in Figure \ref{fig:PerturbedB} becomes negative and the solid line changes its behavior. Here $N=3$, $\alpha = 10^{-1}$, $\beta = \frac{\alpha}{2}$, $\eta = 10^{-3}$.}
\label{fig:Perturbed}
\end{figure}

Due to the coupling of $d_1$ and $d_2$, we are not able to derive limits as previously done in Lemma \ref{lemma:IdenticalInitialization_Convergence}. However, we can show in Lemma \ref{lemma:PerturbedInitialization_Convergence_AllLambda} below that the product $d_1(\tdisc)d_2(\tdisc)^{N-1}$ converges to $\lambda$ regardless of its sign. In order to keep the presentation concise, parts of the proof (treating positive $\lambda$) are deferred to Appendix \ref{sec:Appendix_Supplement} in form of Lemma \ref{lemma:PerturbedInitialization_Convergence_PositiveLargeLambda} and \ref{lemma:PerturbedInitialization_Convergence_PositiveSmallLambda}.

For $\lambda<0$, which cannot be recovered with identical initialization, the key is the ``phase transition" time $\tdisc_0$, defined by
\begin{equation}\label{def:PahseTransitionTime}
    \tdisc_0 = \inf \{\tdisc\in\mathbb{N}_0:\d_1(\tdisc)<0\},
\end{equation}
where we use the convention that $\inf \emptyset = \infty$.
\begin{lemma}\label{lemma:PerturbedInitialization_PhaseTransition}
    Let $N \geq 2$ and $\lambda < 0$. Let $\d_1,\d_2$ be defined by \eqref{eq:PerturbedInitialization_ScalarDynamics} with the perturbed identical initialization \eqref{eq:PerturbedInitialization_ScalarInitialization} and define $M=\max(\alpha,|\lambda|^{\frac{1}{N}})$. If
    \begin{equation}\label{assume:eta}
        0<\eta<\frac{1}{(3N-2)M^{2N-2}},
    \end{equation}
    then $\tdisc_0$ defined in \eqref{def:PahseTransitionTime} is finite. 
\end{lemma}
\begin{proof}
    Let the difference $\Delta_1$ and the factor $\kappa$ be defined as in \eqref{def:Diff_SquareDiff_RateFactor}. 
    We define the two auxiliary sequences
    \begin{align}
        a(\tdisc+1) &= a(\tdisc) - \eta a(\tdisc)^{N-1}(a(\tdisc)^N-\lambda),
        \quad a(0) = \alpha -\beta >0,\label{def:a_IdenticalInitialization}\\
        b(\tdisc+1) &= b(\tdisc) - \eta b(\tdisc)^{N-1}(b(\tdisc)^N-\lambda),
        \quad b(0) = \alpha >0, \label{def:b_IdenticalInitialization}
    \end{align}
    whose behavior is well-understood by Section \ref{subsec:IdenticalInitialization}, and we make the auxiliary claim that
    \begin{equation}\label{claim-lambda-neg1}
        0 \leq \d_1(\tdisc) \leq a(\tdisc) \leq \alpha-\beta
        \quad \mbox{ and } \quad
        0 < b(\tdisc) \leq \d_2(\tdisc) \leq \alpha, \quad \mbox{ for } \tdisc < \tdisc_0.
    \end{equation}
    Assume that $\tdisc < \tdisc_0$ for the moment. By the definition of $\tdisc_0$, we have $0\leq\d_1(\tdisc)$. Furthermore, Lemma \ref{lemma:IdenticalInitialization_Convergence} implies that $a(\tdisc)\leq\alpha-\beta$ and $0<b(\tdisc)$. Note that due to $\lambda<0$ and $\d_1(\tdisc) \geq 0$, $\kappa(\tdisc) \geq 0$ as long as $\d_2(\tdisc) > 0$. Since $\d_2(0) = \alpha > 0$, $\Delta_1(0) = \beta > 0$ and $\d_2(\tdisc) = \d_1(\tdisc) + \Delta_1(\tdisc)$ it follows by induction from \eqref{Delta1_Delta2} in Lemma~\ref{lemma:DeltaContraction} that $\Delta_1(\tdisc) > 0$, $\d_2(\tdisc) > \d_1(\tdisc) \geq 0$ and $\kappa(\tdisc) > 0$ for all $\tdisc < \tdisc_0$.
    Therefore, $\d_1(\tdisc)$ and $\d_2(\tdisc)$ are monotonically decreasing in $\tdisc$ for $\tdisc<\tdisc_0$ by \eqref{eq:PerturbedInitialization_ScalarDynamics}. Hence, $\d_1(\tdisc)\leq \d_1(0) = \alpha-\beta$ and $\d_2(\tdisc)\leq \d_2(0) = \alpha$, for $\tdisc < \tdisc_0$. 
    %Note that $\kappa(\tdisc)>0$ due to $\lambda<0$ and $\d_1,\d_2>0$. 
    %By Lemma \ref{lemma:DeltaContraction}, $\Delta_1(\tdisc) = \d_2(\tdisc) - \d_1(\tdisc)$ is monotonically increasing for $\tdisc < \tdisc_0$.
    % since (recall that $\lambda < 0$) 
    % \begin{align} \label{eq:justification}
    %     0 
    %     < \eta \kappa(\tdisc)
    %     = \eta \d_2(\tdisc)^{N-2} (\d_1(\tdisc) \d_2(\tdisc)^{N-1} - \lambda)
    %     \le \eta (\alpha^{2N-2} + |\lambda| \alpha^{N-2})
    %     \le 2 \eta \M^{2N-2} 
    %     < 1,
    % \end{align}
    % by assumption on $\eta$.
    %Since $\d_2(0) - \d_1(0)>0$, we \textcolor{red}{obtain} $\d_2(\tdisc) - \d_1(\tdisc)>0$ for $\tdisc < \tdisc_0$. 
    In order to fully prove \eqref{claim-lambda-neg1}, we will show next that $\d_1(\tdisc) \leq a(\tdisc)$ and $b(\tdisc) \leq \d_2(\tdisc)$ by induction. 
    
    By construction $\d_1(0)=\alpha-\beta=a(0)$ and $\d_2(0)=\alpha=b(0)$. Assume that $\d_1(\tdisc) \leq a(\tdisc)$ and $b(\tdisc) \leq \d_2(\tdisc)$ for some $\tdisc < \tdisc_0-1$. Define $g(x) = x-\eta x^{N-1}(x^N-\lambda)$ as in the proof of Lemma \ref{lemma:IdenticalInitialization_Convergence}. By a direct computation as in \eqref{ineq:g-deriv}, $g'(x)\geq 0$ for \revision{$x\in(0,M)$} because $\eta$ satisfies \eqref{assume:eta}. Thus $g$ is monotonically increasing on \revision{$(0,M)$} and \revision{since $\d_1(\tdisc),\d_2(\tdisc),a(\tdisc),b(\tdisc) \in (0,M)$ }
    % Recall that $g'(x) = 1 - \eta x^{N-2} ( (2N-1) x^N - (N-1) \lambda)$ such that, by , the function $g$ is monotonously increasing on $(0,\alpha)$. Since $\d_1(\tdisc) < \d_2(\tdisc)$ and $\lambda<0$, we obtain
    \begin{align*}
        \d_1(\tdisc+1)
        &< \d_1(\tdisc) - \eta\d_1(\tdisc)^{N-1}(\d_1(\tdisc)^{N}-\lambda)
        = g(\d_1(k)) \leq g(a(\tdisc))
        = a(\tdisc+1) \quad \text{and} \\
        \d_2(\tdisc+1) 
        &> \d_2(\tdisc) - \eta\d_2(\tdisc)^{N-1}(\d_2(\tdisc)^{N}-\lambda)
        = g(\d_2(\tdisc)) \geq g(b(\tdisc))  
        = b(\tdisc + 1).
        \end{align*}
    This completes the induction step and shows \eqref{claim-lambda-neg1}.
    
    We now prove the lemma by contradiction. Suppose that $\tdisc_0=\infty$ so that \eqref{claim-lambda-neg1} holds for all $\tdisc\in\mathbb{N}_0$. Lemma~\ref{lemma:IdenticalInitialization_Convergence} implies that
    $ \lim_{\tdisc \to \infty} a(\tdisc) = \lim_{\tdisc \to \infty} b(\tdisc) = 0$. Since $0 \leq \d_1 \leq a$, we obtain that $\lim_{\tdisc \to \infty} \d_1(\tdisc) = 0$. Note that
    \begin{align} \label{eq:justification}
        0 
        < \eta \kappa(\tdisc)
        = \eta \d_2(\tdisc)^{N-2} (\d_1(\tdisc) \d_2(\tdisc)^{N-1} - \lambda)
        \le \eta (\alpha^{2N-2} + |\lambda| \alpha^{N-2})
        \le 2 \eta \M^{2N-2} 
        < 1
    \end{align}
    by assumption on $\eta$. By Lemma \ref{lemma:DeltaContraction},
    % (as observed in \eqref{eq:justification} above $0 < \eta \kappa(\tdisc) < 1$),
    $\Delta_1(\tdisc)=\d_2(\tdisc)-\d_1(\tdisc)$ is monotonically increasing while $\Delta_2(\tdisc)=\d_2(\tdisc)^2-\d_1(\tdisc)^2$ is monotonically decreasing. Consequently, we have
    $$
    \d_2(\tdisc) \geq \d_2(\tdisc) - \d_1(\tdisc) = \Delta_1(\tdisc) \geq \Delta_1(0) \geq \beta > 0 \quad \mbox{ for all } \tdisc \in \mathbb{N}_0.
    $$
    Since $\d_1(\tdisc), \d_2(\tdisc) \geq 0$ and $\lambda < 0$ this gives
    $$
    \kappa(\tdisc)= \d_2(\tdisc)^{N-2}(\d_1(\tdisc) \d_2(\tdisc)^{N-1} - \lambda) \geq \beta^{N-2} |\lambda|  > 0.
    $$
    Since $\eta \kappa(\tdisc) < 1$, we get
    $$
    | 1- \eta^2 \kappa(\tdisc)^2| \leq 1 - \eta^2 \beta^{2(N-2)} |\lambda|^2 =: \rho < 1.
    $$
    By \eqref{Delta1_Delta2}, this yields
    \[
    |\Delta_2(\tdisc)| \leq \rho^k |\Delta_2(0)|
    \]
    and thus $\lim_{\tdisc \to \infty} \Delta_2(\tdisc) = 0$. Since also $\lim_{\tdisc \to \infty} \d_1(\tdisc) = 0$, we conclude that $\lim_{\tdisc \to 0} \d_2(\tdisc) = 0$ which contradicts $\d_2(\tdisc) \geq \beta > 0$ for all $\tdisc \in \mathbb{N}_0$. Hence, the set $\{\tdisc\in\mathbb{N}_0:\d_1(\tdisc)<0\}$ is not empty and $\tdisc_0$ is finite. 
\end{proof}
\begin{lemma}\label{lemma:PerturbedInitialization_Convergence_AllLambda}
    Let $N \geq 2$ and $\lambda \in\mathbb{R}$. Let $\d_1,\d_2$ be defined by \eqref{eq:PerturbedInitialization_ScalarDynamics} with the perturbed identical initialization \eqref{eq:PerturbedInitialization_ScalarInitialization}, where $0<\beta<\alpha$. Define $M = \max(\alpha,|\lambda|^{\frac{1}{N}})$ and let $\c\in(1,2)$ be the maximal real solution to the polynomial equation $1=(\c-1)\c^{N-1}$. If
    \begin{equation}\label{cond:stepsize}
        %\textcolor{blue}{
        0 < \eta <  %\min\left\{\frac{2^{\frac{1}{N}}-1}{2}  ,\frac{1}{9N} \right\} \frac{1}{(\c M)^{2N-2}}=
        \frac{1}{9N(\c M)^{2N-2}},
        %}
        %, \frac{1}{2^N} \right),
        % \min\left(2^{\frac{1}{N}}-1, (4N)^{-1} \right)\cdot\frac{(c-1)^2}{2M^{2N-2}},
    \end{equation}
    then $\lim_{\tdisc\to\infty}\d_1(\tdisc)\d_2(\tdisc)^{N-1}=\lambda$.
\end{lemma}
\begin{proof}
    If $\lambda\geq 0$, then the claim immediately follows from Lemmas \ref{lemma:PerturbedInitialization_Convergence_PositiveLargeLambda} and \ref{lemma:PerturbedInitialization_Convergence_PositiveSmallLambda}. Thus it suffices to prove the claim for $\lambda<0$. By Lemma \ref{lemma:PerturbedInitialization_PhaseTransition}, $\tdisc_0 := \inf \{\tdisc\in\mathbb{N}_0 \colon \d_1(\tdisc)<0\}$ is finite. For $\tdisc \geq \tdisc_0$ we are in a second phase of the dynamics where $\d_1(\tdisc)$ stays negative. In order to understand this phase we set $\d_0(\tdisc) := -\d_1(\tdisc)$ and observe that $\d_0(\tdisc)$ and $\d_2(\tdisc)$ satisfy
    \begin{align} \label{eq:d0d2}
    \begin{split}
        \d_0(\tdisc+1) 
        &= \d_0(\tdisc) - \eta\d_2(\tdisc)^{N-1}(\d_0(\tdisc)\d_2(\tdisc)^{N-1}-|\lambda|)\\
        \d_2(\tdisc+1) 
        &= \d_2(\tdisc) - \eta\d_0(\tdisc)\d_2(\tdisc)^{N-2}(\d_0(\tdisc)\d_2(\tdisc)^{N-1}-|\lambda|).
    \end{split}
    \end{align}
    Setting \revision{$\alpha_0 := \d_2(\tdisc_0)$} and $\gamma_0 := \d_0(\tdisc_0) = - \d_1(\tdisc_0) >0$, and interpreting $\alpha_0,\gamma_0$ as new initial conditions, the above system has the same form as \eqref{eq:PerturbedInitialization_ScalarDynamics} with $\lambda$ replaced by $|\lambda| >0$. According to Lemmas \ref{lemma:PerturbedInitialization_Convergence_PositiveLargeLambda} and \ref{lemma:PerturbedInitialization_Convergence_PositiveSmallLambda},
    \begin{equation*}
        \lim_{\tdisc\to\infty}\d_1(\tdisc)\d_2(\tdisc)^{N-1}
        =-\lim_{\tdisc\to\infty}\d_0(\tdisc)\d_2(\tdisc)^{N-1}
        =-|\lambda|
        =\lambda
    \end{equation*}
    if $0<\gamma_0<\alpha_0\leq\M$. It remains to verify the latter condition. \revision{Since $d_1(\tdisc_0-1)$ and $d_2(\tdisc_0-1)$ are positive, we have $d_2(\tdisc_0)\leq d_2(\tdisc_0-1)$ according to the dynamics in \eqref{eq:d0d2}. Together with the fact that $d_2$ is monotonically decreasing before this time, we deduce that $\alpha_0 \leq \alpha \leq M$.}
    %$\alpha_0<\alpha\leq\M$ simply comes from the fact that $\d_2$ is decreasing for $\tdisc\leq\tdisc_0$. 
    Let the difference sequence $\Delta_1$ and the factor $\kappa$ be defined as in \eqref{def:Diff_SquareDiff_RateFactor}. Since $0 < \eta \kappa(\tdisc) < 1$, for $\tdisc < \tdisc_0$, cf.\ Equation~\eqref{eq:justification}, Lemma~\ref{lemma:DeltaContraction} states that $\Delta_1(\tdisc)>0$ is monotonically increasing, for $\tdisc < \tdisc_0$. In particular, $\d_2(\tdisc_0-1) - \d_1(\tdisc_0-1) = \Delta_1(\tdisc_0-1) > 0$
    and we obtain
    %Deducing $0\leq\d_1(\tdisc_0-1)<\d_2(\tdisc_0-1)$, we get that
    \begin{align}
        \alpha_0 - \gamma_0
        &= \d_2(\tdisc_0) + \d_1(\tdisc_0)\nonumber\\
        &= \d_2(\tdisc_0-1) - \eta\d_1(\tdisc_0-1)\d_2(\tdisc_0-1)^{N-2}(\d_1(\tdisc_0-1)\d_2(\tdisc_0-1)^{N-1}-\lambda)\nonumber\\
        &\quad + \d_1(\tdisc_0-1) - \eta\d_2(\tdisc_0-1)^{N-1}(\d_1(\tdisc_0-1)\d_2(\tdisc_0-1)^{N-1}-\lambda)\nonumber\\
        &= (\d_1(\tdisc_0-1)+\d_2(\tdisc_0-1))(1-\eta\kappa(\tdisc_0-1))\label{ineq:alpha0geqgamma0}
        \geq 0.
    \end{align}
    Hence the condition $0<\gamma_0<\alpha_0\leq\M$ is satisfied.
\end{proof}
From a less technical point of view, Lemma \ref{lemma:PerturbedInitialization_Convergence_AllLambda} and its proof show that if $\lambda >0$, the dynamics is similar to the case of identical initialization. 
If $\lambda < 0$, the dynamics is only similar up to the point where one of the components changes sign. Then, they start to behave as if $\lambda = |\lambda|$ and follow a mirrored trajectory of the identical initialization setting. 
%Hence, the rate for the perturbed initialization is closely related to the rate for the identical initialization.
We have all tools at hand to finally prove Theorem \ref{theorem:PerturbedInitialization_NonAsymptoticConvergence}.

\subsubsection*{Proof of Theorem \ref{theorem:PerturbedInitialization_NonAsymptoticConvergence}}

The convergence of $\W(\tdisc)$ to $\Wstar$ directly follows from Lemma \ref{lemma:PerturbedInitialization_MatrixDynamics} and \ref{lemma:PerturbedInitialization_Convergence_AllLambda}. For the rate of convergence, we will use Theorem \ref{thm:IdenticalInitialization_ConvergenceRate} and Lemmas
\ref{lemma:DeltaContraction}, \ref{lemma:PerturbedInitialization_PhaseTransition}, \ref{lemma:PerturbedInitialization_Convergence_AllLambda}, \ref{lemma:PerturbedInitialization_Convergence_PositiveLargeLambda}, \ref{lemma:PerturbedInitialization_Convergence_PositiveSmallLambda}.

With $\lambda = \lambda_i$, let $\d_1,\d_2$ be defined as in \eqref{eq:PerturbedInitialization_ScalarDynamics} and $a,b,p,p_a,p_b$ as in \eqref{def:a_IdenticalInitialization}, \eqref{def:b_IdenticalInitialization} and \eqref{def:p_pa_pb}. Note that by Lemma \ref{lemma:PerturbedInitialization_MatrixDynamics}, $E_{ii}(\tdisc) = p(\tdisc) - \lambda_i$. We distinguish the following cases.

\begin{enumerate}[label=(\alph*)]

\item \label{case_a} Assume that $\lambda_i\geq (\alpha-\beta)\alpha^{N-1}$. By Lemma \ref{lemma:PerturbedInitialization_Convergence_PositiveLargeLambda}, we have that $p_a(\tdisc)\leq p(\tdisc) \leq \lambda_i$ for all $\tdisc\in\mathbb{N}_0$. Thus $|p(\tdisc)^{\frac{1}{N}}-\lambda_i^{\frac{1}{N}}|\leq  |p_a(\tdisc)^{\frac{1}{N}}-\lambda_i^{\frac{1}{N}}|<\epsilon$ for all $\tdisc\geq T_N^{\Id}(\lambda,\epsilon,\alpha,\eta)$ according to Theorem \ref{thm:IdenticalInitialization_ConvergenceRate}. 
This implies that $|E_{ii}(\tdisc)| \leq \epsilon N |\lambda_i|^{1-\frac{1}{N}}$ for all $\tdisc\geq T_N^{\Id}(\lambda,\epsilon,\alpha,\eta)$ as in the proof of Theorem~\ref{theorem:IdenticalInitialization_ContinuousNonAsymptoticRate}.

\item \label{case_b} Assume that $0\leq \lambda_i< (\alpha-\beta)\alpha^{N-1}$. By Lemma \ref{lemma:PerturbedInitialization_Convergence_PositiveSmallLambda}, we have that either $0\leq p_a(\tdisc)\leq p(\tdisc) \leq \lambda_i\leq\alpha^N$, for all $\tdisc\in\mathbb{N}_0$, or $0\leq\lambda_i\leq p(\tdisc) \leq p_b(\tdisc)\leq\alpha^N$, for all $\tdisc\in\mathbb{N}_0$. In either case we can deduce that $0\leq p(\tdisc) \leq \alpha^N$ and hence $|p(\tdisc)-\lambda_i| \leq \alpha^N$ for all $\tdisc\in\mathbb{N}_0$.

\item \label{case_c} Assume that $-(\alpha-\beta)\alpha^{N-1}<\lambda_i<0$. As analyzed in Lemma \ref{lemma:PerturbedInitialization_Convergence_AllLambda}, after the phase transition point $\tdisc_0$ defined in \eqref{def:PahseTransitionTime}, which is finite by Lemma \ref{lemma:PerturbedInitialization_PhaseTransition}, the dynamics effectively becomes the one with $|\lambda_i|$ (instead of $\lambda_i$) with initialization $\alpha_0 := \d_2(\tdisc_0) > 0$ and $\gamma_0 := - \d_1(\tdisc_0)$. Note that $0<\alpha_0<\alpha$ because $\d_2(\tdisc)$ is montonically decreasing for $\tdisc<\tdisc_0$ and $0<\gamma_0<\alpha_0$ by \eqref{ineq:alpha0geqgamma0}. Altogether we have $0<\gamma_0<\alpha_0<\alpha$. After $\tdisc_0$, we are back in either case \ref{case_a} (if $|\lambda_i|\geq \gamma_0\alpha_0^{N-1}$) or case \ref{case_b} (if $|\lambda_i|< \gamma_0\alpha_0^{N-1}$). In case \ref{case_a}, we have $0>p(\tdisc) \geq -\lambda_i > -\alpha^N$, for all $\tdisc\in\mathbb{N}_0$. In case \ref{case_b}, we have $0\geq p(\tdisc) \geq -\alpha^N$ for all $\tdisc\in\mathbb{N}_0$. In both cases we can deduce that $0\geq p(\tdisc) \geq -\alpha^N$ and hence $|p(\tdisc)-\lambda_i| \leq \alpha^N$ for all $\tdisc\in\mathbb{N}_0$.

\item \label{case_d} Assume that $\lambda_i\leq-(\alpha-\beta)\alpha^{N-1}$.
Let $\Delta_1(\tdisc)$, $\Delta_2(\tdisc)$ and $\kappa(\tdisc)$ be defined as in \eqref{def:Diff_SquareDiff_RateFactor} with $\lambda = \lambda_i$. Let $\tdisc_0$ be the phase transition time defined in \eqref{def:PahseTransitionTime} at which $d_1$ becomes negative. We start with some useful observations.

For $\tdisc<\tdisc_0$, the sequences $\d_1,\d_2$ are positive and decreasing by induction, i.e.,
\begin{equation}\label{d:ineq}
    \d_1(\tdisc),\d_2(\tdisc)>0,\quad
    \d_1(\tdisc + 1) < \d_1(\tdisc),\quad 
    \d_2(\tdisc + 1) < \d_2(\tdisc).
\end{equation}
Since $\d_1(0)<\d_2(0)<|\lambda_i|^{\frac{1}{N}}$, we have  $\d_1(\tdisc),\d_2(\tdisc)< |\lambda_i|^{\frac{1}{N}}$. Therefore
\begin{equation*}
    \kappa(\tdisc) = \d_2(\tdisc)^{N-2}(\d_1(\tdisc)\d_2(\tdisc)^{N-1}-\lambda) \leq |\lambda_i|^{\frac{N-2}{N}} 2 |\lambda_i| =  2|\lambda_i|^{\frac{2N-2}{N}}.
\end{equation*}

Note that $M = \max\{\alpha, \| \widehat{W}\|^\frac{1}{N}\} \geq |\lambda_i|^{\frac{1}{N}}$ together with Condition~\eqref{cond:stepsize:thm} on the stepsize implies that
\begin{equation}\label{kappa:bound}
    0 < \eta\kappa(\tdisc) < 2\eta |\lambda_i|^{\frac{2N-2}{2}} < \frac{2|\lambda_i|^{\frac{2N - 2}{N}}}{9N(cM)^{2N-2}}  \leq \frac{2}{9Nc^{2N-2}} < 1.
\end{equation}

for $\tdisc<\tdisc_0$. By Lemma~\ref{lemma:DeltaContraction}, $\Delta_1$ is positive and increasing, while $\Delta_2$ is positive and decreasing in $\tdisc$. Thus $\d_2(\tdisc) - \d_1(\tdisc) > \d_2(0) - \d_1(0) = \beta$ and $\d_2(\tdisc)^2 - \d_1(\tdisc)^2 > 0$ for $\tdisc<\tdisc_0$.

% Recall that both before and after the phase transition time $\tdisc_0$, $|\d_1|,|\d_2|\leq cM$ according to the proof of Lemma \ref{lemma:PerturbedInitialization_Convergence_PositiveLargeLambda}, \ref{lemma:PerturbedInitialization_Convergence_PositiveSmallLambda}, and \ref{lemma:PerturbedInitialization_Convergence_AllLambda}. Hence $0 < \kappa(k) \leq 2 (cM)^{2N - 2}$ and we get that
% \begin{equation}
% 0 < \eta\kappa(\tdisc) < \frac{2(c-1)^2(cM)^{2N - 2}}{8NM^{2N-2}} \leq \frac{1}{4N}
% \end{equation}
% for all $\tdisc\in\mathbb{N}_0$.

% As argued in the proof of Lemma~\ref{lemma:PerturbedInitialization_Convergence_AllLambda}, $\eta |\kappa(\tdisc)| < 1$ for all $\tdisc \in \mathbb{N}_0$. Hence, by Lemma~\ref{lemma:DeltaContraction} $\Delta_1(\tdisc)$ preserves its sign and $|\Delta_2(\tdisc)|$ is monotonically decreasing in $\tdisc$, implying that $\d_1(\tdisc) < \d_2(\tdisc)$ and $|\d_2^2(\tdisc) - \d_1^2(\tdisc)| \leq \alpha^2 - (\alpha - \beta)^2$ for all $\tdisc \in \mathbb{N}_0$. 

% Hence, Lemma~\ref{lemma:DeltaContraction} implies that $\Delta_1(\tdisc)$ preserves its sign and that
% $\Delta_1(\tdisc+1) \geq \Delta_1(\tdisc)$ for all $\tdisc \in \mathbb{N}_0$.
% In particular, since $\d_1(0) = \alpha-\beta <  \alpha = \d_2(0)$ we have that $\d_2(\tdisc) > \d_1(\tdisc) + \beta$ for all $\tdisc \in \mathbb{N}_0$. 

At the phase transition point $\tdisc_0$, since $\d_1(\tdisc_0)<0\leq\d_1(\tdisc_0-1)$ and $\d_2(\tdisc_0 - 1) > \d_1(\tdisc_0 - 1) + \beta > \beta$, relation \eqref{kappa:bound} yields
\begin{align}\label{d2:lowerbound}
    \d_2(\tdisc_0)
    & = \d_2(\tdisc_0 - 1) - \eta \d_1(\tdisc_0 - 1) \kappa(\tdisc_0 - 1)
    \geq \d_2(\tdisc_0 - 1)(1-\eta \kappa(\tdisc_0 - 1))\notag\\ 
    & > \beta\left(1-\frac{2}{9Nc^{2N-2}}\right)
    =\frac{9N - 2(c-1)^2}{9N} \beta.
\end{align}

Now consider $\tdisc \geq \tdisc_0$ and recall from case \ref{case_c} and the proof of Lemma~\ref{lemma:PerturbedInitialization_Convergence_AllLambda}, that the dynamics effectively becomes the one with $\lambda_i$ replaced by $|\lambda_i|$ by considering $(\d_0,\d_2):=(-\d_1,\d_2)$
and with initializations $|\d_1(\tdisc_0)|$ and $\d_2(\tdisc_0)$.

We now distinguish the subcases $\alpha^N \geq |\lambda_i|$ and $\alpha^N < |\lambda_i|$. Assume first that $\alpha^N \geq |\lambda_i|$. Then $0<-\d_1(\tdisc_0)<\d_2(\tdisc_0)<\alpha$ as in part (c). By Lemma~\ref{lemma:PerturbedInitialization_Convergence_PositiveLargeLambda} and Lemma~\ref{lemma:PerturbedInitialization_Convergence_PositiveSmallLambda}, $0 \geq \d_1(\tdisc)\d_2(\tdisc)^{N-1} \geq - \max(\alpha^N, |\lambda_i|) = - \alpha^N$ for all $\tdisc \geq \tdisc_0$. Moreover,
$0\leq \d_1(\tdisc) \leq \d_1(0) = (\alpha-\beta) < \alpha$ and $\d_2(\tdisc) \leq \d_2(0) \leq \alpha$ for all $\tdisc = 0,\hdots,\tdisc_0$ by \eqref{d:ineq}. Hence, $0 \leq \d_1(\tdisc)\d_2(\tdisc)^{N-1} \leq  \alpha^{N}$ for all $\tdisc = 0,\hdots,\tdisc_0$. It follows that altogether
$|E_{ii}(\tdisc)| = |\d_1(\tdisc)\d_2(\tdisc)^{N-1} - \lambda_i| \leq 2 \alpha^N$ for all $\tdisc \in \mathbb{N}_0$.

If $\alpha^N < |\lambda_i|$ then the same analysis as in case \ref{case_c} gives $0<-\d_1(\tdisc_0)<\d_2(\tdisc_0)<\alpha<|\lambda_i|^{\frac{1}{N}}$. Consequently, $|\d_1(\tdisc_0)|\d_2(\tdisc_0)^{N-1}\leq |\lambda_i|$ and we are back to the case \ref{case_a}.
By Lemma~\ref{lemma:PerturbedInitialization_Convergence_PositiveLargeLambda}, $\d_2(\tdisc) \leq c \max\{\alpha, |\lambda_i|^{\frac{1}{N}}\} = c|\lambda_i|^{\frac{1}{N}}$ 
%is bounded above by $\c\M$, 
for all $\tdisc\geq\tdisc_0$, and $\d_0, \d_2$ are monotonically increasing. %Thereofore,  
%By induction it follows directly from their definitions that $\d_2$ is positive and monotonically increasing
%while while $\d_1$ is negative and monotonically decreasing, i.e.,
%}
%
%After the phase transition ($\tdisc\geq\tdisc_0$), unless at some point $|\d_1\d_2^{N-1}|>|\lambda_i|$, which is impossible according to Lemma \ref{lemma:PerturbedInitialization_Convergence_PositiveLargeLambda} \textcolor{red}{(reformulate)}, by induction $\d_2$ is positive and increasing while $\d_1$ is negative and decreasing, i.e.,
%\begin{equation*}
%    \d_1(\tdisc) < 0 < \d_2(\tdisc),\quad
%    \d_1(\tdisc + 1) < \d_1(\tdisc),\quad 
%    \d_2(\tdisc + 1) > \d_2(\tdisc).
%\end{equation*}
Thus $\d_2$ attains its minimum at $\tdisc_0$, which implies by \eqref{d2:lowerbound} that
$\d_2(\tdisc) \geq \frac{9N-2(c-1)^2}{9N}\beta$, for all $\tdisc \in \mathbb{N}_0$. Summarizing, we obtain that 
%As analyzed in the proof of Lemma~\ref{lemma:PerturbedInitialization_Convergence_AllLambda}, cf.\ Equation~\eqref{eq:d0d2}, by defining $(\d_0,\d_2):=(-\d_1,\d_2)$, the dynamics effectively become the ones with $|\lambda_i|$ (instead of $\lambda_i$) with initialization $|\d_1(\tdisc_0)|$ and $\d_2(\tdisc_0)$. \textcolor{red}{By the same analysis in part (c), $0<-\d_1(\tdisc_0)<\d_2(\tdisc_0)<\alpha<|\lambda_i|^{\frac{1}{N}}$. Consequently, $|\d_1(\tdisc_0)|\d_2(\tdisc_0)^{N-1}\leq |\lambda_i|$ }
% Since $\d_2(\tdisc_0)<\d_2(0)=\alpha\leq |\lambda_i|^{\frac{1}{N}}$,
%and we are back in case \ref{case_a}. By Lemma \ref{lemma:PerturbedInitialization_Convergence_PositiveLargeLambda}, $\d_0$ and $\d_2$ are increasing and $\d_2$ is bounded above by \textcolor{red}{$\c\M$}, for all $\tdisc\geq\tdisc_0$. Thus together with analysis before phase transition, we get that
\begin{equation}\label{two-sided-d2}
    \frac{9N-2(c-1)^2}{9N}\beta\leq \d_2(\tdisc)\leq \c|\lambda_i|^{\frac{1}{N}}
    \quad \text{ for all }\tdisc\in\mathbb{N}_0.
\end{equation}
Note that since $|\d_1(\tdisc)| = \d_0(\tdisc) <\d_2(\tdisc)$ for $\tdisc\geq\tdisc_0$, \eqref{kappa:bound} also holds for $\tdisc\geq\tdisc_0$. Thus, Lemma~\ref{lemma:DeltaContraction} implies that $\d_2(\tdisc)^2 - \d_1(\tdisc)^2>0$ for all $\tdisc\in\mathbb{N}_0$.

To show our claim, it remains to characterize a time $T_0$ for which $d_1(T_0) < -\beta$ and apply Theorem \ref{thm:IdenticalInitialization_ConvergenceRate} as for the case \ref{case_a} with $\d_0(T_0) > \beta$
% (and $\d_2(T_0) > \beta$ because $\Delta_2 = \d_2^2-\d_0^2>0$)
as initial condition. This will give
$$
\left||d_1(T_0 + \tdisc) d_2(T_0 + \tdisc)^{N-1}|^{\frac{1}{N}}-|\lambda_i|^{\frac{1}{N}}\right| \leq \epsilon \quad \mbox{ for all } \tdisc \geq T^\Id_N\left(|\lambda_i|, \epsilon, \beta, \eta\right).
$$
implying that
$$
|E_{ii}(\tdisc)| \leq |\lambda_i| - (|\lambda_i|^{\frac{1}{N}} - \epsilon)^N
\quad \mbox{ for all }
\tdisc \geq T^{\operatorname{P}}_N(\lambda_i,\epsilon,\alpha,\beta,\eta) = T^\Id_N\left(|\lambda_i|, \epsilon, \beta, \eta\right) + T_0.
$$
The result will then follow from the same argument as in the proof of Theorem \ref{theorem:IdenticalInitialization_ContinuousNonAsymptoticRate}. %\\
So let us characterize $T_0$. 
Note that the assumption $0 < \frac{\beta}{c-1} < \alpha$ of Theorem~\ref{theorem:PerturbedInitialization_NonAsymptoticConvergence} together with $\alpha^N < |\lambda_i|$ implies that $|\lambda_i|^{\frac{1}{N}} \geq \frac{\beta}{c-1}$. Now assume that $\tdisc \in \mathbb{N}_0$ is such that $\d_1(\tdisc) \geq -\beta$. With $\d_2(\tdisc)\leq \c|\lambda_i|^{\frac{1}{N}}$ we then obtain
\begin{align*}
    \d_1(\tdisc) \d_2(\tdisc)^{N-1} - \lambda_i
    &= |\lambda_i| + \d_1(\tdisc) \d_2(\tdisc)^{N-1}
    \geq |\lambda_i| - \beta (c|\lambda_i|^{\frac{1}{N}})^{N-1}\\
    &= (|\lambda_i|^{\frac{1}{N}} - \beta c^{N-1})|\lambda_i|^{1-\frac{1}{N}}
    = \left(|\lambda_i|^{\frac{1}{N}} - \frac{\beta}{c-1}\right)|\lambda_i|^{1-\frac{1}{N}}
    > 0.
\end{align*}
%by assumptions that $|\lambda_i|^{\frac{1}{N}}>\frac{\beta}{c-1}$. 
Together with the lower bound in \eqref{two-sided-d2} this gives
\begin{align*}
    \d_1(\tdisc + 1)
    &= \d_1(\tdisc) - \eta \d_2(\tdisc)^{N-1}(\d_1(\tdisc) \d_2(\tdisc)^{N-1} - \lambda_i)\\
    &\leq \d_1(\tdisc) - \eta \left(\frac{9N-2(c-1)}{9N}\beta\right)^{N-1}\left(|\lambda_i|^{\frac{1}{N}} - \frac{\beta}{c-1}\right)|\lambda_i|^{1-\frac{1}{N}}.
\end{align*}
Since $\d_1(0) = \alpha - \beta$ it follows by induction that
\[
\d_1(k) \leq \alpha- \beta - \tdisc \eta \left(\frac{9N-2(c-1)}{9N}\beta|\lambda_i|^{\frac{1}{N}}\right)^{N-1}\left(|\lambda_i|^{\frac{1}{N}}-\frac{\beta}{c-1}\right)
\]
as long as $\d_1(\tdisc - 1) > - \beta$. Hence
\[
\d_1(\tdisc) \leq - \beta \quad \mbox{ for all } \tdisc \geq T_0 := \frac{\alpha}{\eta \left( \frac{9N-2(c-1)}{9N}\beta|\lambda_i|^{\frac{1}{N}}\right)^{N-1} \left(|\lambda_i|^{\frac{1}{N}}-\frac{\beta}{c-1}\right)}.
\]
This completes the proof.
\end{enumerate}

\section{Implicit Bias of Gradient Descent}
\label{sec:EffectiveRank}

%%%%%%%%%%%%%%%%%%%%%%%%%%%%%%%%%%%%%%%%%%%
%%% Implicit Bias Chapter
%%%%%%%%%%%%%%%%%%%%%%%%%%%%%%%%%%%%%%%%%%%

The explicit characterization of gradient flow and gradient descent dynamics derived in Section \ref{sec:AnalysisOfDynamics} may be used to shed some light on the phenomenon of \emph{implicit bias} resp.\ \emph{implicit regularization} of gradient descent. In fact, different convergence rates for different eigenvalues (depending on their respective signs and magnitudes) result in matrix iterates of low effective rank and accurately explain the implicit regularization observed when applying gradient descent to matrix factorization of symmetric matrices. We expect that a similar reasoning to be valid in more general contexts beyond matrix estimation. 

As a concrete illustration, Figures~\ref{fig:EffRankA} and \ref{fig:EffRankB} depict the outcome of a simple experiment in which a rank $3$ ground truth $\Wstar \in \mathbb{R}^{200\times 200}$ is approximated by factorized iterates $\W(\tdisc) = \W_2(\tdisc)\W_1(\tdisc)$ minimizing
\begin{align*}
    (\W_1,\W_2)~\mapsto~\frac{1}{2}~\| \W_2\W_1~-~\Wstar \|_F^2
\end{align*}
via gradient descent, i.e., matrix factorization with $N = 2$. Figure \ref{fig:EffRankA} shows the Frobenius approximation error over the iterates which decreases in a characteristic "waterfall behavior".
Related, yet more striking, is the dynamics of the effective rank of $\W(\tdisc)$ shown in Figure \ref{fig:EffRankB}. The effective rank\footnote{A related quantity called ``stable rank" is defined as $\widehat{r}(W) = \norm{W}{*}^2/ \norm{W}{F}^2$, for which it also holds that $1 \leq \widehat{r}(W) \leq \mathrm{rank}(W)$. For our purposes the effective rank $r(W)$ turned out to be more convenient.} of a matrix $W$ is defined as
\begin{align*}
	r(W) = \frac{\norm{W}{*}}{\norm{W}{}},
\end{align*}
i.e., the ratio of nuclear and operator norm of $W$, for which $1 \le r(W) \le \mathrm{rank}(W)$. The plateaus of Figure \ref{fig:EffRankB} are of height $1$, $1.5$, and $1.6$, corresponding to the values $r(\Wstar_1)$, $r(\Wstar_2)$, and $r(\Wstar_3)$, where we denote by $\Wstar_\L$ the rank-$\L$ best term approximation of $\Wstar$. Consequently, the gradient descent iterates $\W(\tdisc)$ approach the effective rank of $\Wstar$ step-by-step while each intermediate step is given by the effective rank of a rank $\L$ approximation of $\Wstar$. 

Building upon the analysis of the gradient descent iterates $\W(\tdisc)$ and the underlying time-continuous gradient flow $\tilde{\W}(\tcont)$ in Section \ref{sec:AnalysisOfDynamics}, we are going to present two results, Theorem \ref{prop:EffectiveRankContinuousPSD} (gradient flow, $N=2$) and Theorem \ref{prop:EffectiveRankDiscretePSD} (gradient descent, $N\ge2$), that precisely describe length and location of the plateaus in Figures \ref{fig:EffRankSection1} and \ref{fig:EffRankSection2}, for positive semi-definite ground truth $\Wstar$. Both results implicitly carry an important meta-statement: after few steps, the effective rank of the gradient flow and gradient descent iterates first decreases down to the multiplicity of the largest singular value of $\widehat{W}$ (i.e., to $1$ in the displayed experiments).
%of $\widehat{W}$
%and is monotonically increasing
%afterwards}.
%
%
%reduces to $r_{\text{min}} = \mu(\lambda_{\text{max}})$\footnote{Here $\mu(\lambda)$ denotes added multiplicity of both eigenvalues $\pm\lambda$ and $\lambda_{\text{max}}$ the in absolute value maximal eigenvalue of $\Wstar$.} and 
Afterwards the effective rank approaches plateaus of monotonically increasing height. Numerical simulations in more general settings including matrix sensing show a similar behavior. Therefore, we expect that this phenomenon holds in wider contexts
where explicit dynamics cannot be derived anymore. In particular,
proving monotonicity of the effective rank (or suitable replacements in the case of  nonlinear networks) after a short initial phase should be possible even under more general assumptions and will yield new insights into the implicit bias of gradient descent in deep learning.

%the effective rank of gradient flow and descent iterates \textcolor{blue}{first decreases down to the multiplicity of the largest eigenvalue $\lambda_1$ of $\widehat{W}$
%and is monotonically increasing
%afterwards.

%reduces to $r_{\text{min}} = \mu(\lambda_{\text{max}})$ and is monotonically increasing afterwards. \textcolor{red}{Reformulate a bit!}
%More importantly,
% we believe that while the observed implicit bias of gradient descent is hard to approach in a rigorous way,
%proving monotonicity of the effective rank of its iterates should be possible even under more general assumptions and yield new insights into the implicit bias phenomenon.

Let us start with describing the effective rank behavior of the gradient flow $\tilde{\W}(\tcont)$. While being sufficiently precise, the resulting statements are less involved and thus easier to appraise than the ones for the gradient descent iterates $\W(\tdisc)$.

%%%%%%%%%%%%%%%%%%%%%%%%%%%%%%%
%%% OLD Figure
\begin{comment}
\begin{figure}[!htb]
\begin{subfigure}[c]{0.45\textwidth}
%\includegraphics[width=\textwidth]{EffectiveRank_ApproxError.pdf}
\includegraphics[width=\textwidth]{FiguresPython/EffRankApproxError.png}
\subcaption{Approximation error.}
\label{fig:EffRankA}
\end{subfigure}
\quad
\begin{subfigure}[c]{0.45\textwidth}
\includegraphics[width=\textwidth]{FiguresPython/EffectiveRank.png}
\subcaption{Effective rank.}
\label{fig:EffRankB}
\end{subfigure}
\caption{(\textcolor{red}{Recommend COMBINE figure \ref{fig:EffRank}, \ref{fig:EffRankNoise}, and \ref{fig:EffRankNoise_GD}}) Matrix Factorization via GD without noise. The original matrix $\Wstar \in \mathbb{R}^{200\times 200}$ is of rank $3$ with $\lambda_1 = 10$, $\lambda_2 = 5$, and $\lambda_3 = 1$. The $x$-axis shows the number of iterations $\tdisc$. One can compare approximation of $\Wstar$ measured as $\norm{\Wstar - W(t)}{F}$ (a) and behavior of $r(W(t))$ (b). The dashed blue lines highlight certain values computed from Theorem \ref{prop:EffectiveRankContinuousPSD}, cf.\ Section \ref{subsubsec:SymmetricPSD}.} \label{fig:EffRank}
\end{figure}
\end{comment}
%%% OLD Figure - End
%%%%%%%%%%%%%%%%%%%%%%%%%%%%%%

\subsection{Gradient Flow}

As a proof of concept, the following theorem makes the above observations on effective rank approximation more precise for positive semi-definite $\Wstar$ and the continuous flow $\tilde{\W} = \tilde{\W}_2(\tcont) \tilde{\W}_1(\tcont)$ whose components' eigenvalue dynamics is characterized by \eqref{eq:IdenticalInitialization_ContinuousScalarDynamics} in Section \ref{sec:AnalysisOfDynamics}. In particular, the theorem illustrates the involved dependence between parameters of the problem and the stopping time necessary for gradient flow (and consequently for gradient descent as well) to produce solutions of a certain rank. The results in Section \ref{sec:AnalysisOfDynamics} allow to extend the setting by similar arguments to general symmetric ground truths. To keep the presentation simple, we refrain from giving statements in full generality and leave this to the reader.

Recall the distinction between discrete time $\tdisc$ and continuous time $\tcont = \eta \tdisc$ which are related by the step size $\eta$. We denote the eigenvalues of the symmetric ground truth $\Wstar$ by $\lambda_1\geq\cdots\geq \lambda_n\geq 0$. We define, for $\L \in [n]$, $\alpha > 0$ fixed, and $\L' = \max\curly{ \l \colon \lambda_\l > \alpha^2 }$, the three time intervals %\textcolor{red}{(Maybe we should express their dependence on the vaiable $\epsilon$ and $c$)}
\begin{align} \label{eq:Idef}
\begin{split}
		I_1 &= \curly{ \tcont \in \mathbb{R}_+ \colon \max_{\l \in [\L]}\curly{ \abs{ \frac{ g_{\lambda_1,\alpha}(\tcont) }{g_{\lambda_\l,\alpha}(\tcont)} - 1}{} } < \varepsilon }, \\
		I_2 &= \curly{ \tcont \in \mathbb{R}_+ \colon \abs{ \frac{g_{\lambda_1,\alpha}(\tcont)}{g_{\lambda_{\L+1},\alpha}(\tcont)} \frac{\lambda_{\L+1}}{\lambda_1} }{} < \frac{r(\Wstar_\L)}{\L'-\L} \varepsilon },
		\text{ and } \\
		I_3 &= \curly{ \tcont \in \mathbb{R}_+ \colon g_{\lambda_1,\alpha}(\tcont) < \cc },
\end{split}
\end{align}
where
\begin{align} \label{eq:gdef}
    g_{\lambda,\alpha} (\tcont) = 1 + \round{ \frac{\lambda}{\alpha^2} - 1} e^{-2\lambda \tcont}
\end{align}
and $\cc > 1$ may be chosen arbitrarily. Note, however, that $\cc$ influences Theorem \ref{prop:EffectiveRankContinuousPSD} since it controls the trade-off between the size of $I_3$ and tightness of the bound.
\begin{theorem} \label{prop:EffectiveRankContinuousPSD}
	 Let $\Wstar \in \mathbb{R}^{n\times n}$ be a symmetric ground truth with eigenvalues $\lambda_1 \ge \dots \ge \lambda_n \ge 0$ and $\tilde{\W}(\tcont) = \tilde{\W}_2(\tcont) \tilde{\W}_1(\tcont)$ be the %gradient flow whose component's eigenvalue dynamics are characterized by \eqref{eq:IdenticalInitialization_ContinuousScalarDynamics} in Section \ref{sec:AnalysisOfDynamics}%
	 solution of the differential equation $\tilde{\W}'_j(\tcont) = -\nabla_{\tilde{\W}_j(\tcont)} \mathcal{L}(\tilde{\W}(\tcont))$,
	 %with $\mathcal{L}(\tilde{\W}(\tcont)) = \frac{1}{2}\|\tilde{\W}-\Wstar\|_F^2$
	 a continuous analog of the gradient descent \eqref{eq:MatrixFactorization_GradientDescent}--\eqref{eq:MatrixFactorization_IdenticalInitialization}, for $N=2$. Let $\varepsilon > 0$ and $\L \in [n]$ be fixed and assume that $\alpha^2 < \lambda_\L$ in \eqref{eq:MatrixFactorization_IdenticalInitialization}. Define $\L' = \max\curly{ \l \colon \lambda_\l > \alpha^2 }$ and let the three time intervals $I_1$, $I_2$, and $I_3$ be given as in \eqref{eq:Idef}--\eqref{eq:gdef}. Then,
	\begin{align} \label{eq:EffectiverRankContinuousPSD}
		\abs{r(\Wstar_\L) - r(\tilde{\W}(\tcont))}{} \le 2 \varepsilon \; r(\Wstar_\L) + \cc \frac{n-\L'}{\|\Wstar\|} \alpha^2,
	\end{align}
	for all $\tcont \in I_1 \cap I_2 \cap I_3$. 
\end{theorem}
\begin{remark}
    Note that the effective rank of $\tilde{\W}$ and $\Wstar_\L$ approximately agree on $I_1 \cap I_2 \cap I_3$ if $\epsilon$ and $\alpha$ are small. If $\L = 1$, one has that $I_1 = \mathbb{R}_+$ and the first summand on the right-hand side of \eqref{eq:EffectiverRankContinuousPSD} becomes $\varepsilon \; r(\Wstar_\L)$. From the proof it will become clear that $I_1$ describes the time interval in which the leading $\L$ eigenvalues of $\Wstar$ are well approximated, $I_2$ describes the time interval in which the dominant tail eigenvalues $\lambda_{\L+1},\dots,\lambda_{\L'}$ are still small compared to the leading $\L$ eigenvalues, and $I_3$ controls $\lambda_1$ independently of $I_1$ in order to bound the error caused by the neglectable tail eigenvalues $\lambda_{\L'},\dots,\lambda_n$.\\
	Let us furthermore mention that the restriction on $N = 2$ in Theorem \ref{prop:EffectiveRankContinuousPSD} is due to the hardness of deriving from Lemma \ref{lemma:IdenticalInitialization_ContinuousScalarDynamics} explicit eigenvalue dynamics of $\tilde{\W}$ for larger $N$. (Note that the function $g_{\lambda,\alpha}$ in \eqref{eq:gdef} originates from the explicit dynamics for $N=2$ in Remark \ref{rem:ExplicitForm}.) The lemma, though, provides implicit characterizations of the eigenvalue dynamics for general $N$, which will be used in the gradient descent analysis in Section \ref{subsec:gradientdescent}. Finally, we emphasize that in order to guarantee $I_2$ is non-empty an eigenvalue gap $\lambda_1 > \lambda_{\L+1}$ is required whenever $\varepsilon$ is small. %\\ 
\end{remark}
\begin{remark}
    So far, one can hardly speak of regularization since the ground-truth $\Wstar$ is the unique global minimizer. Consequently, if $\Wstar$ is low-rank the unique global minimizer of $W \mapsto \frac{1}{2} \|\W - \Wstar \|_F^2$, which any optimization algorithm produces, is low-rank. A more sensible scenario would be given by considering bounded additive noise $\Xi \in \mathbb{R}^{n \times n}$ and using gradient descent for de-noising $\Wstar = \W_\text{LR} + \Xi$, where the ground-truth $\W_\text{LR}$ is of low rank. In this case, the global minimum $\Wstar$ of $W \mapsto \frac{1}{2} \|\W - \Wstar \|_{F}^2 $ is in general full-rank (although still effectively low-rank). Eventually, $W$ converges to $\Wstar$. However, the above theorem suggests that, before convergence, it is of lower effective rank for a certain period of time. These intermediate stages provide good approximations to $\W_\text{LR}$ with reduced noise influence. In fact, it is straight-forward to deduce a generalized version of Theorem \ref{prop:EffectiveRankContinuousPSD}, allowing noise that is bounded in operator norm: just apply\footnote{To apply the theory in Section \ref{sec:AnalysisOfDynamics} and hence Theorem \ref{prop:EffectiveRankContinuousPSD}, the matrix $\Wstar$ has to be symmetric which is not necessarily the case for general noise. However, when applying gradient descent, we can replace $\Wstar$ by its symmetrized version $\frac{1}{2}(\Wstar + \Wstar^\tT)$ since $\frac{1}{2}(\Wstar + \Wstar^\tT) = \W_\text{LR} + \frac{1}{2} (\Xi + \Xi^\tT)$ where $\norm{\frac{1}{2} (\Xi + \Xi^\tT)}{} \le \norm{\Xi}{}$. We may thus assume without loss of generality that the noise $\Xi$ and with it $\Wstar$ is symmetric.} the theorem to $\Wstar$ and control the eigenvalues of $\Wstar$ by the eigenvalues of $\W_\text{LR}$ and Weyl's inequality \cite[Theorem 4.3.1]{horn2012matrix}. Though conceptually simple, the modifications lead to involved statements and are thus omitted. In fact, even the bounds in Theorem \ref{prop:EffectiveRankContinuousPSD} yield good results for moderate noise as Figures \ref{fig:EffRankNoiseA} and \ref{fig:EffRankNoiseB} show.
\end{remark}
\begin{figure}[!htb]
\centering
\begin{subfigure}[c]{0.45\textwidth}
\includegraphics[width=\textwidth]{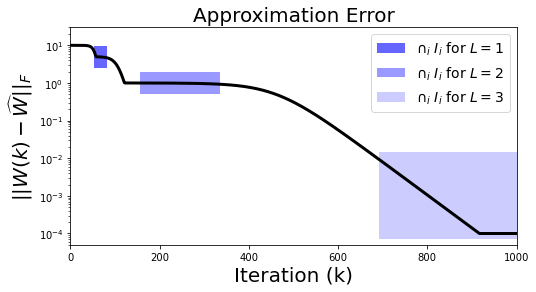}
\subcaption{Approximation error without noise.}
\label{fig:EffRankA}
\end{subfigure}
\quad
\begin{subfigure}[c]{0.45\textwidth}
\includegraphics[width=\textwidth]{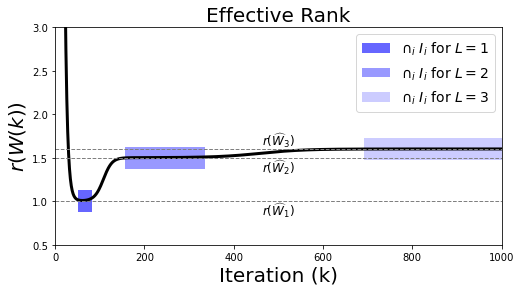}
\subcaption{Effective rank without noise.}
\label{fig:EffRankB}
\end{subfigure}
\quad
\begin{subfigure}[c]{0.45\textwidth}
\includegraphics[width=\textwidth]{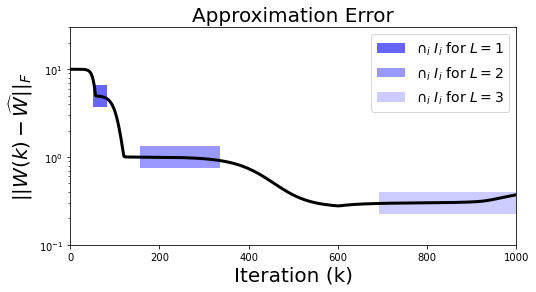}
\subcaption{Approximation error with noise.}
\label{fig:EffRankNoiseA}
\end{subfigure}
\quad
\begin{subfigure}[c]{0.45\textwidth}
\includegraphics[width=\textwidth]{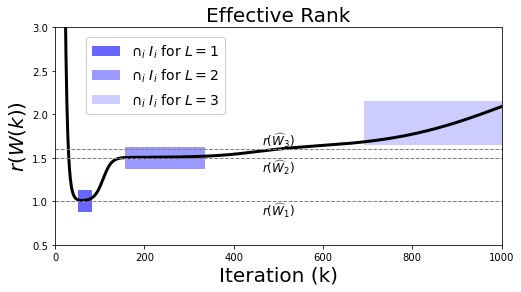}
\subcaption{Effective rank with noise.}
\label{fig:EffRankNoiseB}
\end{subfigure}
\caption{Above is the dynamics of gradient descent \eqref{eq:MatrixFactorization_GradientDescent}-\eqref{eq:MatrixFactorization_IdenticalInitialization}, where $N=2$, $n=200$, $\alpha = 10^{-2}$, $\eta = 10^{-2}$. $\W_\text{LR}\in\mathbb{R}^{n\times n}$ is a rank $3$ symmetric matrix with leading eigenvalues $(\lambda_1,\lambda_2,\lambda_3) = (10,5,1)$. In (a)-(b), $\Wstar = \W_\text{LR}$; in (c)-(d), $\Wstar = \W_\text{LR} + \Xi$, where $\Xi \in \mathbb{R}^{n \times n}$ is a scaled random Gaussian matrix such that $\norm{\Xi}{} = 0.05 \| \Wstar \|$. The shaded regions are the predictions from Theorem \ref{prop:EffectiveRankContinuousPSD} where the best rank $L$ approximation of $\Wstar$ lies, with the scaling $\tcont = \eta\tdisc$. Here we take $\epsilon=2.2\cdot 10^{-2}$, $L'=3$, $C=17$, which makes the prediction error (right hand side of \eqref{eq:EffectiverRankContinuousPSD}) less than $10^{-1}$. The empirical errors are also less than $10^{-1}$ in the shaded regions.}
\label{fig:EffRankSection1}
\end{figure}
The proof of Theorem \ref{prop:EffectiveRankContinuousPSD} relies on the evolution of eigenvalues of $\tilde{W}_j(\tcont)$, $j \in [N]$, which is characterized by the ODE
\begin{equation}  \label{eq:IdenticalInitialization_ContinuousScalarDynamics_Repeated}
    \tilde{\d}_\lambda'(\tcont) = - \tilde{\d}_\lambda(\tcont)^{N-1} (\tilde{\d}_\lambda(\tcont)^N-\lambda),
    \quad \tilde{\d}_\lambda(0) = \alpha.
\end{equation}
We start with a technical statement.
\begin{lemma}
\label{lem:Monotonicity}
    Let $\tilde{\d}_\lambda$ be defined as in \eqref{eq:IdenticalInitialization_ContinuousScalarDynamics_Repeated} \revision{and set $N \ge 2$. Then, $\tilde{\d}_\lambda(\tcont) \ge 0$, for any $\tcont \ge 0$ and $\lambda \in \mathbb R$, and} the function $\tilde{\d}_\lambda$ is monotonically increasing in $\lambda$, i.e., $\tilde{\d}_\lambda (\tcont) \ge \tilde{\d}_{\lambda'} (\tcont)$, for all $\tcont \ge 0$ and $\lambda \ge \lambda'$. %Moreover, for $N = 2$ and $\lambda > \alpha^2$, we can explicitly write
    %\begin{align*}
	%	\tilde{\d}_\lambda^2(\tcont) = \frac{\lambda}{1 + \round{\frac{\lambda}{\alpha^2} - 1} e^{-2\lambda \tcont}} = \frac{\lambda}{g_{\lambda,\alpha}(\tcont)},
	%\end{align*} 
	%where $g_{\lambda,\alpha}$ is defined in \eqref{eq:gdef}.
	%\red{Just for curiosity: Is it possible to extend the monotonicity to GD?}
\end{lemma}
\begin{proof}
   Fix $N \ge 2$. \revision{Since \eqref{eq:IdenticalInitialization_ContinuousScalarDynamics_Repeated} describes a continuous trajectory that is initialized with $\alpha > 0$ and satisfies $\tilde{\d}'_\lambda(\tcont) = 0$ for $\tilde{\d}_\lambda(\tcont) = 0$, one clearly has $\tilde{\d}_\lambda(\tcont) \ge 0$ for any $\tcont \ge 0$ and $\lambda \in \mathbb R$.}
   
   \revision{Let us now turn to the second statement.} If $\lambda = \lambda'$, we have that $\tilde{\d}_\lambda (\tcont) = \tilde{\d}_{\lambda'} (\tcont)$, for all $\tcont \ge 0$. We thus restrict ourselves to $\lambda > \lambda'$. Let us assume there exists a time $\tcont_< > 0$ such that $\tilde{\d}_\lambda (\tcont_<) \le \tilde{\d}_{\lambda'} (\tcont_<)$. Since $\tilde{\d}_\lambda' (0) > \tilde{\d}_{\lambda'}' (0)$ and $\tilde{\d}_\lambda (0)
   = \tilde{\d}_{\lambda'} (0)$, there exists $\delta > 0$ with $\tilde{\d}_\lambda (\tcont) > \tilde{\d}_{\lambda'} (\tcont)$, for $\tcont \in (0,\delta)$. By continuity of $\tilde{\d}_\lambda$ and $\tilde{\d}_{\lambda'}$, and the intermediate value theorem, the set $K = \{ \tcont \in (0,\tcont_<) \colon \tilde{\d}_\lambda (\tcont) = \tilde{\d}_{\lambda'}(\tcont) \} \subset (0,\tcont_<)$ of intersection points is thus compact and non-empty. \\ 
   If we define $\bar{\tcont} \in (0,\tcont_<)$ as the smallest element of $K$ such that, for some $\delta > 0$, $\tilde{\d}_\lambda (\tcont) > \tilde{\d}_{\lambda'} (\tcont)$, for $\tcont \in (\bar{\tcont}-\delta,\bar{\tcont})$ and $\tilde{\d}_\lambda (\tcont) \le \tilde{\d}_{\lambda'} (\tcont)$, for $\tcont \in (\bar{\tcont},\bar{\tcont}+\delta)$, we see that $\tilde{\d}_\lambda (\bar{\tcont}) = \tilde{\d}_{\lambda'} (\bar{\tcont})$ and $\tilde{\d}_\lambda' (\bar{\tcont}) \le \tilde{\d}_{\lambda'}' (\bar{\tcont})$. (Such a minimal element $\bar{\tcont}$ must exist for the following reason: if the smallest element $\tcont_{\text{min}}$ of $K$ does not satisfy the condition, then $\tilde{\d}_\lambda (\tcont) > \tilde{\d}_{\lambda'} (\tcont)$ on a neighborhood, i.e., $\tcont_{\text{min}}$ is isolated and $K-\curly{\tcont_{\text{min}}}$ is still closed. If repeating this process did not stop at some $\bar{\tcont}$, all elements of $K$ would be isolated points and $\tilde{\d}_\lambda (\tcont) > \tilde{\d}_{\lambda'} (\tcont)$, for all $\tcont \in (0,\tcont_<) \setminus K$ contradicting continuity of $\tilde{\d}_\lambda$ in $\tcont_<$.) But this implies
   \begin{align*}
        \tilde{\d}_\lambda(\bar{\tcont})^{N-1} (\lambda - \tilde{\d}_\lambda(\bar{\tcont})^N) 
        = \tilde{\d}_\lambda' (\bar{\tcont})
        \le \tilde{\d}_{\lambda'}' (\bar{\tcont})
        = \tilde{\d}_{\lambda'}(\bar{\tcont})^{N-1} (\lambda' - \tilde{\d}_{\lambda'}(\bar{\tcont})^N)
   \end{align*}
   and, hence, $\lambda \le \lambda'$ contradicting the initial assumption. \revision{Note that in the final deduction step we used $\tilde{\d}_\lambda(\bar{\tcont}) > 0$. We already showed $\tilde{\d}_\lambda(\bar{\tcont}) \ge 0$. The observation that $\tilde{\d}_\lambda(\bar{\tcont}) \neq 0$ follows along the same lines as in the proof of Lemma \ref{lemma:ODE_Increasing_ConvexSolution}. Indeed, one can easily check that $\tilde{\d}_\lambda(\bar{\tcont}) = 0$ would allow two possible explanations of the trajectory --- the one starting at $\tilde{\d}_\lambda(0) = \alpha$ and the constant one starting at $\tilde{\d}_\lambda(0) = 0$ --- contradicting the local uniqueness of the solution of \eqref{eq:IdenticalInitialization_ContinuousScalarDynamics_Repeated} guaranteed by Picard--Lindelöf.}
   %To obtain the explicit solution of $\tilde{\d}_\lambda^2$, note that, for $N=2$, Lemma \ref{lemma:IdenticalInitialization_ContinuousScalarDynamics} states
	%\begin{align*}
	%	-\lambda \tcont + c 
	%	= \frac{1}{2}\log|\tilde{\d}_\lambda^2 - \lambda| -\log|\tilde{\d}_\lambda| 
	%	= \frac{1}{2}\log \abs{1 - \frac{\lambda}{\tilde{\d}_\lambda^2}}{},
	%\end{align*}
	%such that we may compute the presented form, for $\tilde{W}_j(0) = \alpha > 0$ and $\alpha^2 < \lambda$.
\end{proof}
\begin{proof}[Proof of Theorem \ref{prop:EffectiveRankContinuousPSD}]
	We first decompose the difference as
	\begin{align*} %\label{eq:RankDecomposition}
	\begin{split}{}
	\abs{r(\Wstar_\L) - r(\tilde{W}(\tcont))}{} &\le \underbrace{ \abs{r(\Wstar_\L) - r(\tilde{W}(\tcont)_\L)}{} }_{=:A_1} + \underbrace{ \abs{r(\tilde{W}(\tcont)_\L) - r(\tilde{W}(\tcont)_{\L'})}{}}_{=:A_2} + \underbrace{ \abs{r(\tilde{W}(\tcont)_{\L'}) - r(\tilde{W}(\tcont))}{} }_{=:A_3}. \\
	%&= (I) + (II) + (III).
	\end{split}
	\end{align*}
	By using the explicit form of $\tilde{\d}_\lambda^2$ given in Remark~\ref{rem:ExplicitForm}
	\begin{align*}
		\tilde{\d}_\lambda^2(\tcont) = \frac{\lambda}{1 + \round{\frac{\lambda}{\alpha^2} - 1} e^{-2\lambda \tcont}} = \frac{\lambda}{g_{\lambda,\alpha}(\tcont)},
	\end{align*} 
	where $g_{\lambda,\alpha}$ is defined in \eqref{eq:gdef},
	we obtain, for $\tcont\in I_1$, that
	\begin{align*}
		A_1 = \abs{ \sum_{\l=1}^\L \revision{\left( \frac{ \abs{\tilde{d}_{\lambda_\l}^2(\tcont)}{} }{ \abs{\tilde{d}_{\lambda_1}^2(\tcont)}{} } - \frac{ \abs{\lambda_\l}{} }{ \abs{\lambda_1}{} } \right)} }{}
		\le \sum_{\l=1}^\L \round{ \abs{ \frac{ g_{\lambda_1,\alpha}(\tcont)}{ g_{\lambda_\l,\alpha}(\tcont) } - 1}{} \frac{\lambda_\l}{\lambda_1} }
		\le \epsilon r(\Wstar_\L).
	\end{align*}
	\revision{By monotonicity of $\tilde{\d}_\lambda(\tcont) \ge 0$ in $\lambda$, cf.\ Lemma~\ref{lem:Monotonicity}, which implies monotonicity of $\tilde{\d}_\lambda^2(\tcont)$,} we have, for
	$\tcont\in I_2$, 
	\begin{align*}
		A_2
		= \sum_{\l=\L+1}^{\L'} \frac{ \abs{\tilde{d}_{\lambda_\l}^2(\tcont)}{} }{ \abs{\tilde{d}_{\lambda_1}^2(\tcont)}{} } 
		\le (\L' - \L) \frac{ \abs{\tilde{d}_{\lambda_{\L+1}}^2(\tcont)}{} }{ \abs{\tilde{d}_{\lambda_1}^2(\tcont)}{} } 
		= (\L' - \L)  \abs{ \frac{ g_{\lambda_1,\alpha}(\tcont)}{g_{\lambda_{\L+1},\alpha}(\tcont)} \frac{\lambda_{\L+1}}{\lambda_1} }{}
		\leq \epsilon r(\Wstar_\L).
	\end{align*}
	Finally, for $\tcont\in I_3$,
	\begin{align*}
	    A_3
	    = \sum_{\l=\L'+1}^{n} \frac{ \abs{\tilde{d}_{\lambda_\l}^2(\tcont)}{} }{ \abs{\tilde{d}_{\lambda_1}^2(\tcont)}{} }
	    \leq \sum_{\l=\L'+1}^{n}\abs{ \frac{ g_{\lambda_1,\alpha}(\tcont)}{g_{\lambda_{\ell},\alpha}(\tcont)} \frac{\lambda_{\ell}}{\lambda_1} }{}
	    \le \cc \frac{n-\L'}{\lambda_1} \alpha^2,
	\end{align*}
	where we used that $\tilde{d}_{\lambda_1}^2(\tcont) \ge \frac{\lambda_1}{\cc}$, for all $\tcont \in I_3$, and the definition of $\L'$. The claim follows by assuming $\tcont \in I_1 \cap I_2 \cap I_3$. For $\L = 1$, it is clear that $I_1 = \mathbb{R}_+$ and \revision{$A_1 = 0$}.
\end{proof}

\subsection{Gradient Descent}
\label{subsec:gradientdescent}

%\textcolor{red}{(better?)} 
After the simpler analysis of gradient flow, we now deduce  a corresponding statement for gradient descent from the results in Section~\ref{sec:AnalysisOfDynamics}. In contrast to Theorem~\ref{prop:EffectiveRankContinuousPSD}, it holds for a general number of layers $N \ge 2$. As before, we remark that the statement can be extended in a straight-forward but tedious way to handle general symmetric ground truths and additive noise.
%
% \begin{comment}
% Let us define
% \begin{align} \label{eq:Tdef}
% \begin{split}
%     T_0(\lambda) &:= \frac{1}{\eta(-\lambda)}\cdot
%     \begin{cases}
%     \ln(\alpha) - \ln(\epsilon) 
%     &\text{ if } N=2\\
%     \frac{1}{N-2}\left(\frac{1}{\epsilon^{N-2}}-\frac{1}{\alpha^{N-2}}\right)
%     &\text{ if } N\geq 3.
%     \end{cases} \\
%     T_1(\lambda,\epsilon) &:=
%     \frac{1}{\eta\lambda}\cdot\begin{cases}
%     c - \frac{1}{2}\log|\epsilon^2 + \lambda| + \log|\epsilon|
%     &\text{ if } N=2\\
%     c - \frac{1}{N}\sum_{\ell = 1}^N c_{N, \ell} \log|\epsilon- r_{N, \ell}| - \frac{1}{(N-2)\epsilon^{N-2}}
%     &\text{ if } N \geq 3
%     \end{cases}\\
%     T_2(\lambda,\epsilon) &:=\frac{\ln(\epsilon) - \ln(\lambda^{\frac{1}{N}})}{\ln(1 - \eta \lambda^{2-\frac{2}{N}})}\\
%     s(\lambda) &= \frac{\lambda^{2-\frac{1}{N}}}{ \alpha^{N-1}(-\alpha^N+\lambda)}\\
%     \gamma(\lambda) &= \frac{\ln(1 - \eta N\lambda^{2-\frac{2}{N}})}{\ln\left(1 - \eta c_2 N\lambda^{2-\frac{2}{N}}\right)},
% \end{split}
% \end{align}
% where $c_2 = N((N-1)/(2N-1))^{2 - \frac{2}{N}}$, the constants $c,c_{N,l},r_{N,l}$ are defined in Lemma \ref{lemma:IdenticalInitialization_ContinuousScalarDynamics}, and we suppress for simplicity any dependence on $\eta$, $N$, and $\alpha$.
% \end{comment}
%
\begin{theorem} \label{prop:EffectiveRankDiscretePSD}
    Let $\Wstar \in \mathbb{R}^{n\times n}$ be a symmetric ground-truth with eigenvalues $\lambda_1 \ge \cdots \ge \lambda_n \ge 0$ and assume that $\W_1(\tdisc),\dots,\W_N(\tdisc) \in \mathbb{R}^{n\times n}$ and $\W(\tdisc) = \W_N(\tdisc) \cdots \W_1(\tdisc)$ follow the gradient descent defined in \eqref{eq:MatrixFactorization_GradientDescent}, \eqref{eq:MatrixFactorization_IdenticalInitialization}, for $N \ge 2$. Let
    $\L \in [n]$ be fixed and assume $\lambda_{L+1} > 0$. Let
    $\varepsilon \in (0,1)$ and $\varepsilon' \in (0,c_N)$, where $c_N = \frac{N-1}{2N-1}$. Assume that the initialization parameter $\alpha$ in \eqref{eq:MatrixFactorization_IdenticalInitialization} satisfies $\alpha^N \leq \varepsilon' \lambda_{L+1}$ and that the stepsize satisfies
    $\eta < \left( (3N-2) \max\left\{\alpha^{N-2}, \lambda_1^{2-\frac{2}{N}}\right\}\right)^{-1}$.
    %, and set $\epsilon > \alpha^N$. 
    Define $\L' = \max\curly{ \l \in [n] \colon \varepsilon' \lambda_\l > \alpha^N }$, $L'' = \max\curly{ \l \in [n] \colon \lambda_\l > \alpha^N }$ and recall the definition of $T_N^{\text{Id}}$ in \eqref{def:TId} and $T_N^+$ in \eqref{def:T+-}.
    %
    % \begin{comment}
    %     Recall the quantities defined in \eqref{eq:Tdef}. Define $K' = \max\curly{ k \colon \lambda_k > \alpha^2 }$ and 
    %     \begin{equation} \label{eq:IdenticalInitialization_ContinuousNonAsymptoticRate}
    %         T_{\text{Id}}(\lambda,\epsilon):=
    %         \begin{cases}
    %         T_0(\lambda)
    %         &\text{if } \lambda \leq 0\\
    %         T_1(\lambda,\zeta) + s(\lambda) + \gamma(\lambda) T_2(\lambda,\epsilon) 
    %         &\text{if } \lambda > 0
    %         \end{cases},
    %     \end{equation}
    %     for $\zeta = (2N-1)^{-\frac{1}{N}}((N-1)\lambda)^{\frac{1}{N}}$ (a thorough discussion of the specific shape and meaning of $T_{\text{Id}}$ is deferred to Theorem \ref{theorem:IdenticalInitialization_ContinuousNonAsymptoticRate}).
    % \end{comment}
    %
    Let us abbreviate the time when the leading $L$ eigenvalues of $\Wstar$ are approximated up to $\varepsilon$ (relative to their respective magnitude) as
    \begin{align*}
        T_\mathrm{max}(\L,\varepsilon,\alpha,\eta) = \max_{\l \in [\L]} \; T_N^{\text{Id}} \round{\lambda_\l, \frac{\lambda_\l^{\frac{1}{N}}}{4 N} \epsilon ,\alpha,\eta}.
    \end{align*}
    Then, for $\tdisc$ satisfying
    \begin{align} \label{eq:EffectiveRankDiscretePSD}
        \max\curly{ T_N^{\text{Id}} \round{ \lambda_1,\frac{\lambda_1}{2},\alpha,\eta }, T_\mathrm{max} (\L,\varepsilon,\alpha,\eta) } \le \tdisc \le \frac{1}{\eta} T_N^+ \round{ \lambda_{\L+1}, (\epsilon' \lambda_{L+1})^{1/N}, %\round{r(\Wstar_\L) \varepsilon}^{\frac{1}{N}},
        \alpha} %,\eta },
    \end{align}
    we have
    \begin{align}
		\abs{r(\Wstar_\L) - r(W(\tdisc))}{} & \le
		\varepsilon r(\Wstar_\L) + \frac{2(L'-L)}{c_N}
        \frac{\lambda_{L+1}}{\lambda_1} \varepsilon'  + 
        \frac{2}{\lambda_1}\left( \sum_{\l=L'+1}^{L''} \lambda_\l + (n-L'') \alpha^N\right) \notag\\
        & \leq
        \varepsilon r(\Wstar_\L) + \frac{2(L'-L)}{c_N}
        \frac{\lambda_{L+1}}{\lambda_1} \varepsilon'
        +        (n-L')\frac{2\alpha^N}{ \varepsilon'\lambda_1}. \label{eq:bound-effective-rank}
		%\round{1 + \frac{6(\L'-\L)}{\norm{\Wstar}{}} } \varepsilon \; r(\Wstar_\L) + 2 \frac{n-\L'}{\norm{\Wstar}{}} \alpha^N.
	\end{align}
\end{theorem}
\begin{remark}
    In order to obtain a simplified estimate, we can choose $\varepsilon'$
    and $\alpha$ depending on $\epsilon$ and some of the eigenvalues of $\widehat{W}$  such that all three terms in \eqref{eq:bound-effective-rank} are bounded by $\varepsilon r(\widehat{W}_L)$. Additionally, replacing $L'-L$ by its upper bound $n-L$ and $n-L'$ by its upper bound $n-L$ (so that $L'$ and $L''$ are not needed for the estimate), this gives the choices
    \begin{align}
    \varepsilon' &= c_N \min\left\{ \frac{\lambda_1}{\lambda_{L+1}} \frac{\varepsilon r(\widehat{W}_L) }{2(n-L)}, 1 \right\}, \notag\\
    \alpha & =  \min\left\{(\varepsilon'\lambda_1)^{\frac{1}{N}} \left(\frac{\varepsilon r(\widehat{W}_L)}{2(n-L)}\right)^{\frac{1}{N}}, \left(\varepsilon' \lambda_{L+1}\right)^{\frac{1}{N}}\right\} \label{choice:alpha}  
    \end{align}
    and the estimate
    \[
    \abs{r(\Wstar_\L) - r(W(\tdisc))}{} \le
		3\, \varepsilon\, r(\Wstar_\L). 
    \]
    Of course, we may additionally set $\varepsilon = \delta/3$ for some desired accuracy $\delta \in (0,1)$ to obtain $$\abs{r(\Wstar_\L) - r(W(\tdisc))}{} \leq \delta\, r(\Wstar_\L).$$
    Note in particular that \eqref{choice:alpha} gives an indication on how to choose $\alpha$ (smaller values of $\alpha$ may also be fine). In particular, we require small enough initialization in comparison with the spectral norm $\|\widehat{W}\|= \lambda_1$. 
    
    Let us also remark that the lower bound (time needed to approximate the leading $L$ eigenvalues) in \eqref{eq:EffectiveRankDiscretePSD} may become larger than the upper bound (time in which the remaining eigenvalues of $\W(\tdisc)$ stay small) for certain parameter choices, i.e., the time interval of values $\tdisc$ for which the theorem can make a statement becomes empty. This is to be expected if there is no gap between $\lambda_L$ and $\lambda_{L+1}$ since then $r(W(\tdisc))$ never comes arbitrarily close to $r(\widehat{W}_L)$ and the time interval is indeed empty, for small $\varepsilon$. Figure~\ref{fig:EffRank_Intro}, however, shows that the theorem does provide non-empty time-intervals in relevant situations.
\end{remark}
\begin{figure}[!htb]
\centering
\begin{subfigure}[c]{0.45\textwidth}
\includegraphics[width=\textwidth]{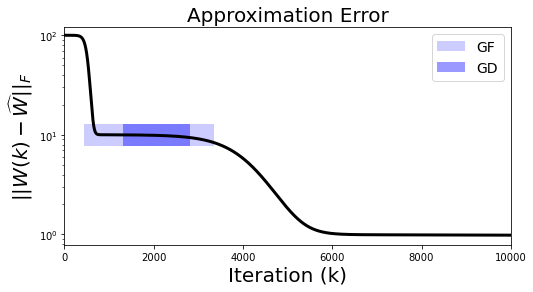}
\subcaption{Approximation error with noise.}
\label{fig:EffRankNoise_GDA}
\end{subfigure}
\quad
\begin{subfigure}[c]{0.45\textwidth}
\includegraphics[width=\textwidth]{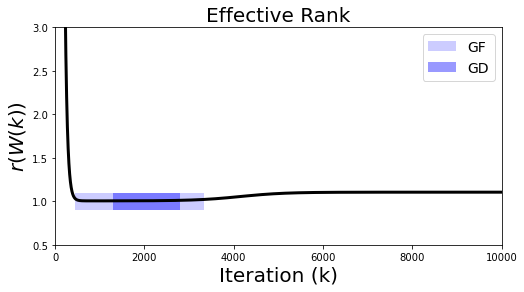}
\subcaption{Effective rank with noise.}
\label{fig:EffRankNoise_GDB}
\end{subfigure}
\caption{Above is the dynamics of gradient descent \eqref{eq:MatrixFactorization_GradientDescent}-\eqref{eq:MatrixFactorization_IdenticalInitialization}, where $N=2$, $n=200$, $\alpha = 5\cdot 10^{-2}$, $\eta = 10^{-4}$. $\Wstar\in\mathbb{R}^{n\times n}$ is a rank $3$ symmetric matrix with leading eigenvalues $(\lambda_1,\lambda_2,\lambda_3) = (100,10,1)$. The light (GF) and dark (GD) shaded regions are the predictions from Theorem \ref{prop:EffectiveRankContinuousPSD} and Theorem \ref{prop:EffectiveRankDiscretePSD}, respectively, where the best rank $L$ approximation of $\Wstar$ lies. Here we take $\epsilon = 3\cdot 10^{-2}$, $\epsilon' = 6.4\cdot 10^{-2}$, $L'=3$, $C=17$, which makes the prediction error (both the right hand side of \eqref{eq:EffectiverRankContinuousPSD} and \eqref{eq:bound-effective-rank}) less than $10^{-1}$. We can see that Theorem \ref{prop:EffectiveRankDiscretePSD} gives a more restrictive but arguably more accurate prediction.}
\label{fig:EffRankSection2}
\end{figure}
\begin{remark}
    The estimate \eqref{eq:EffectiveRankDiscretePSD} in Theorem \ref{prop:EffectiveRankDiscretePSD} for the time interval where the effective rank of $W(\tdisc)$ is close to the one of $\widehat{W}_L$ is slightly weaker than the ones in Theorem \ref{prop:EffectiveRankContinuousPSD} and depends in terms of quality on the step-size $\eta$.
    %, cf.\ Figures \ref{fig:EffRankSection1} and \ref{fig:EffRankSection2}. 
    Since gradient flow is at the core of the gradient descent analysis, the bounds on gradient descent are more accurate if gradient descent stays close to its continuous flow, which is rather the case for small choices of $\eta$ than for large ones. To make this more precise, for $\eta$ small, the gap between necessary and sufficient iteration bounds in Theorem \ref{thm:IdenticalInitialization_ConvergenceRate} shrinks and the prediction accuracy improves.
    %, cf.\ Remark \ref{rem:IdenticalInitialization_ConvergenceRate}. 
    Figure \ref{fig:EffRankNoise_GDB} shows that Theorem \ref{prop:EffectiveRankDiscretePSD} yields an accurate description of the effective rank behavior as long as the step-size is chosen sufficiently small and the eigenvalue gap between $\lambda_\L$ and $\lambda_{\L+1}$ is sufficiently large to guarantee that the feasible region in \eqref{eq:EffectiveRankDiscretePSD} is non-empty.
\end{remark}
To prove Theorem \ref{prop:EffectiveRankDiscretePSD}, we need the following lemma which is a direct consequence of the considerations in Theorem \ref{thm:IdenticalInitialization_ConvergenceRate}. Recall $\zeta_{\lambda} = (c_N \lambda)^{\frac{1}{N}} = \sqrt[N]{\lambda \frac{N-1}{2N-1}}$ appearing in the proof of Theorem \ref{thm:IdenticalInitialization_ConvergenceRate} (inflection point of eigenvalue dynamics) and define $\d_\lambda$ as the discrete dynamic following \eqref{eq:IdenticalInitialization_ScalarDynamics}, i.e.,
\begin{equation} \label{eq:IdenticalInitialization_ScalarDynamics_Repeated}
    \d_\lambda(\tdisc+1) = \d_\lambda(\tdisc) - \eta \d_\lambda(\tdisc)^{N-1}(\d_\lambda(\tdisc)^N-\lambda),
    \quad \d_\lambda(0) = \alpha >0.
\end{equation}
\begin{lemma} \label{lemma:EffectiveRankEpsilonBound}
    Under the assumptions of Theorem \ref{prop:EffectiveRankDiscretePSD}, if $\tdisc \le \frac{1}{\eta} T_N^+(\lambda_{\L+1},(\varepsilon' \lambda_{L+1})^\frac{1}{N},\alpha)$, then, for $\ell \geq L+1$, the eigenvalues $\d_{\lambda_\ell}(\tdisc)^N$ of $W(\tdisc)$, whose $N$-th root dynamics are described by \eqref{eq:IdenticalInitialization_ScalarDynamics_Repeated}, satisfy $$d_{\lambda_\l}(\tdisc)^N \le c_N^{-1} \varepsilon' \lambda_{L+1} \leq 3 \varepsilon' \lambda_{L+1}.$$
    %, for all $\l \ge \L+1$.
\end{lemma}
\begin{proof}
   We consider the continuous version of $d_\lambda$, denoted in this section by $\tilde{\d}_\lambda$, see \eqref{eq:IdenticalInitialization_ContinuousScalarDynamics_Repeated}. Note that $\alpha < (\varepsilon' \lambda_{L+1})^{1/N} \leq (c_N \lambda_{L+1})^{1/N} = \zeta_{\lambda_{L+1}}$. As in the proof of Theorem~\ref{thm:IdenticalInitialization_ConvergenceRate} for the case $\alpha < \zeta$ it follows from Lemma~\ref{lemma:ODE_Increasing_ConvexSolution} with $I = [\alpha,\zeta]$ that
   $d_{\lambda_{L+1}}(\tdisc) \leq \tilde{\d}_{\lambda_{L+1}}(\eta \tdisc)$. Since $\frac{1}{\eta} T_N^+(\lambda,\epsilon,\alpha)$ in \eqref{def:T+-} defines the time $\tdisc$ when $\tilde{\d}_{\lambda}(\eta \tdisc)$ first exceeds $\epsilon$, we have $d_{\lambda_{L+1}} \leq \tilde{\d}_{\lambda_{L+1}}(\eta \tdisc) \leq (\varepsilon' \lambda_{L+1})^{1/N}$ as long as $\tdisc \leq \frac{1}{\eta} T_N^+(\lambda_{L+1},(\varepsilon \lambda_{L+1})^{1/N},\alpha)$.  
   
   Now consider $\ell \geq L+2$ (if such $\ell$ exists). First observe that Lemma \ref{lem:Monotonicity} implies by $\lambda_{\ell} \leq \lambda_{L+1}$ that 
   \begin{equation}\label{eq:lambda-monotone}
   \tilde{d}_{\lambda_\ell}(\eta \tdisc) \leq \tilde{d}_{\lambda_{L+1}}(\eta \tdisc) \quad \mbox{  for all }\tdisc \in \mathbb{N}_0. \end{equation}
   Let us first consider the case that $(\epsilon' \lambda_{L+1})^{1/N} \leq (c_N \lambda_{\ell})^{1/N}= \zeta_{\lambda_{\ell}}$.
   This means by \eqref{eq:lambda-monotone}
   (and $\tilde{\d}_{\lambda_\ell}(0) = \alpha$ and monotonicity of $\tilde{\d}$) that $\tilde{\d}_{\lambda_\ell}(\eta k) \in [\alpha, \zeta_{\lambda_\ell}]$
   for all $\tdisc \leq \frac{1}{\eta} T_N^+(\lambda_{L+1},(\varepsilon \lambda_{L+1})^{1/N},\alpha)$ so that using again Lemma~\ref{lemma:ODE_Increasing_ConvexSolution} as in the proof of Theorem~\ref{thm:IdenticalInitialization_ConvergenceRate} gives $d_{\lambda_\ell}(\tdisc) \leq \tilde{\d}_{\lambda_\ell}(\eta k) \leq \tilde{\d}_{\lambda_{L+1}}(\eta k) \leq (\varepsilon' \lambda_{L+1})^{1/N}$ for all $\tdisc \leq \frac{1}{\eta} T_N^+(\lambda_{L+1},(\varepsilon \lambda_{L+1})^{1/N},\alpha)$. 
   
   In the case that $(\varepsilon' \lambda_{L+1})^{1/N} > (c_N \lambda_\ell)^{1/N}$  Lemma~\ref{lemma:IdenticalInitialization_Convergence} gives
   $$0 < d_{\lambda_{\ell}}(\tdisc) \leq \max\{\alpha, \lambda_\ell^{1/N}\} \leq \max\left\{(\varepsilon' \lambda_{L+1})^{1/N}, \left(\frac{\varepsilon'}{c_N} \lambda_{L+1}\right)^{1/N}\right\} = \left(\frac{\varepsilon'}{c_N} \lambda_{L+1}\right)^{1/N}.$$
   This concludes the proof.
   %If $\lambda_\l \le 3 \varepsilon$, the claim trivially holds since $\d_{\lambda_\l}^N \le \lambda_\l$, cf.\ Lemma \ref{lemma:IdenticalInitialization_Convergence}. Assume now the contrary, implying $\alpha < \varepsilon^\frac{1}{N} \le \round{\frac{\lambda_\l}{3}}^\frac{1}{N} \le \zeta_{\lambda_\l}$. Consequently, as in the proof of Theorem \ref{thm:IdenticalInitialization_ConvergenceRate} (case $\lambda > \alpha^N$, $\alpha < \zeta$), we may apply Lemma \ref{lemma:ODE_Increasing_ConvexSolution} for $I = [\alpha,\zeta_{\lambda_\l}]$ to show that $\d_{\lambda_\l}(\tdisc) \le \tilde{\d}_{\lambda_\l}(\eta \tdisc)$ and conclude with Lemma \ref{lemma:IdenticalInitialization_ContinuousScalarDynamics} that $\d_{\lambda_\l}(\tdisc) \le \varepsilon^\frac{1}{N}$, for $\tdisc \le \frac{1}{\eta} T_N^+(\lambda_\l,\varepsilon^\frac{1}{N},\alpha)$ (recall that $\frac{1}{\eta} T_N^+(\lambda,\epsilon,\alpha)$ in \eqref{def:T+-} defines the time $\tdisc$ when $\tilde{\d}_{\lambda}(\eta \tdisc)$ first exceeds $\epsilon$). The claim follows by monotonicity of $T_N^+$ in the first argument, i.e., $T_N^+(\lambda_\l,\varepsilon,\alpha) \le T_N^+(\lambda_{\l'},\varepsilon,\alpha)$, for $\lambda_\l \ge \lambda_{\l'}$, where the monotonicity of $T_N^+$ in $\lambda$ is implied by monotonicity of $\tilde{d}_\lambda$ in $\lambda$, cf.\ Lemma \ref{lem:Monotonicity}.
\end{proof}
\begin{proof}[Proof of Theorem \ref{prop:EffectiveRankDiscretePSD}]
   The proof mainly relies on Theorem \ref{thm:IdenticalInitialization_ConvergenceRate} characterizing the evolution of eigenvalues of $\W(\tdisc) = \W_N(\tdisc) \cdots \W_1(\tdisc)$ by $\d_\lambda$ in \eqref{eq:IdenticalInitialization_ScalarDynamics_Repeated}.
   As in the proof of Theorem \ref{prop:EffectiveRankContinuousPSD}, we decompose the difference as
   \begin{align*}
       \abs{r(\Wstar_\L) - r(W(\tdisc))}{} 
       &\le 
       \underbrace{\abs{r(\Wstar_\L) - r(W(\tdisc)_\L)}{}}_{=:A_1} + \underbrace{\abs{r(W(\tdisc)_\L) - r(W(\tdisc)_{\L'})}{}}_{=:A_2} + \underbrace{\abs{r(W(\tdisc)_{\L'}) - r(W(\tdisc))}{}}_{=:A_3}. % \\
	   %&= (I) + (II) + (III).
   \end{align*}
   First let $\ell \in [L]$.
   Since $\alpha^N < \varepsilon' \lambda_{L+1} < \lambda_\ell$, Lemma~\ref{lemma:IdenticalInitialization_Convergence} yields $0 < \d_{\lambda_\l}^N(\tdisc) < \lambda_\l$ for all $\tdisc \in \mathbb{N}_0$,
   so that %for all $\ell \in [n]$,
   \begin{align*}
       \abs{ \lambda_\l - d_{\lambda_\l}^N(\tdisc) }{}
       &= \abs{ \lambda_\l^\frac{1}{N} - d_{\lambda_\l}(\tdisc) }{} \abs{ \sum_{\l = 1}^{N}\lambda_\l^{1-\frac{\l}{N}}d_{\lambda_\l}^{\l-1} }{} 
        %\\
       %&
       < N %\max\left\{\alpha^{N-1},\lambda_\ell^{1-\frac{1}{N}}\right\}
       \lambda_\ell^{1-\frac{1}{N}} \abs{ \lambda_\l^\frac{1}{N} - d_{\lambda_\l}(\tdisc) }{}. %\quad \mbox{ for all } \ell \in [n].
   \end{align*}
   Theorem \ref{thm:IdenticalInitialization_ConvergenceRate} %we hence have, for $\l \in [n]$ and 
   implies that for $\tdisc \ge T_N^\text{Id} (\lambda_\l,\tilde{\epsilon}_\l,\alpha,\eta)$ with $\tilde{\epsilon}_\l = \epsilon \lambda_\l^{\frac{1}{N}}/(4N)$ we have $\abs{ \lambda_\l^\frac{1}{N} - d_{\lambda_\l}(\tdisc) }{} \leq \tilde{\epsilon}_\l$, so that
   $\abs{ \lambda_\l - d_{\lambda_\l}^N(\tdisc) }{} < N \lambda_\l^{1-\frac{1}{N}} \tilde{\epsilon}_\ell = \varepsilon \lambda_\l/4$. Hence, for $\tdisc \ge \max\curly{ T_N^\text{Id} (\lambda_1, \lambda_1/2,\alpha,\eta), T_\mathrm{max} (\L, \epsilon,\alpha,\eta) }$ %and $\l \in [\L]$, we get
   we obtain
   \begin{align*}
       \abs{ \frac{ \lambda_\l }{ \lambda_1 } - \frac{ d_{\lambda_\l}^N(\tdisc) }{ d_{\lambda_1}^N(\tdisc) } }{}
       &= \abs{ \frac{ \lambda_\l ( d_{\lambda_1}^N(\tdisc) - \lambda_1 ) + ( \lambda_\l - d_{\lambda_\l}^N(\tdisc) ) \lambda_1 }{ \lambda_1 d_{\lambda_1}^N(\tdisc) } }{} \\
       &\le \frac{\lambda_\l}{\lambda_1} \round{ \abs{d_{\lambda_1}^N(\tdisc) - \lambda_1}{} \frac{1}{ d_{\lambda_1}^N(\tdisc) } + \abs{ \lambda_\l - d_{\lambda_\l}^N(\tdisc) }{} \frac{\lambda_1}{\lambda_\l d_{\lambda_1}^N(\tdisc)} } \\
       &\le \frac{\lambda_\l}{\lambda_1} \round{ \abs{d_{\lambda_1}^N(\tdisc) - \lambda_1}{} \frac{2}{ \lambda_1 } + \abs{ \lambda_\l - d_{\lambda_\l}^N(\tdisc) }{} \frac{2}{\lambda_\l} } 
       \le \frac{\lambda_\l}{\lambda_1} \varepsilon,
   \end{align*}
   where we used $\tdisc \ge T_N^\text{Id} (\lambda_1, \lambda_1/2,\alpha,\eta)$ in the second inequality and $\tdisc \ge T_\mathrm{max} (\L, \epsilon,\alpha,\eta)$ in the last one. This yields
   \begin{align*}
       A_1 \le \sum_{\l=1}^\L \abs{ \frac{ \lambda_\l }{ \lambda_1 } - \frac{ d_{\lambda_\l}^N(\tdisc) }{ d_{\lambda_1}^N(\tdisc) } }{} \leq \frac{\epsilon}{\lambda_1}\sum_{\l=1}^L \lambda_\l = 
        \varepsilon \; r(\Wstar_\L).
   \end{align*}
   Let us now consider $A_2$. %let $L+1 \leq \l \leq L'$.  Then, 
   For $T_N^\text{Id} (\lambda_1, \lambda_1/2,\alpha,\eta) \le \tdisc \le \frac{1}{\eta}T_N^+(\lambda_{\L+1}, (\varepsilon' \lambda_{L+1})^{\frac{1}{N}},\alpha)$ 
   %with $\tilde{\varepsilon} = r(\Wstar_\L) \varepsilon > \alpha^N$, 
   Lemma \ref{lemma:EffectiveRankEpsilonBound} yields
   \begin{align*}
       A_2 \le \sum_{\l=\L+1}^{\L'} \frac{ d_{\lambda_\l}^N(\tdisc) }{ d_{\lambda_1}^N(\tdisc) } 
       \le
       \frac{2}{\lambda_1} \sum_{\l=\L+1}^{\L'} d_{\lambda_\l}^N(\tdisc) 
       \le
        \frac{2(L'-L)}{c_N}
        \frac{\lambda_{L+1}}{\lambda_1} \varepsilon' .
       %\frac{6(\L'-\L)}{\lambda_1} \varepsilon \; r(\Wstar_\L).
   \end{align*}
   Finally, assume $\tdisc \ge T_N^\text{Id} (\lambda_1, \lambda_1/2,\alpha,\eta)$ (implying $d_{\lambda_1}(\tdisc)^N \geq \lambda_1/2$) and
   $\l > L'$.
   Lemma~\ref{lemma:IdenticalInitialization_Convergence} gives $0 < d_{\lambda_\l} \leq \max\{\alpha,\lambda_\ell^{1/N}\}$, hence for $L' < \l \leq L''$ (only applicable if $L'' > L$) we have $\varepsilon' \lambda_\l < \alpha^N \leq \lambda_\l$ %\textcolor{red}{(Shouldn't it be $\varepsilon' \lambda_\l < \alpha^N \leq \lambda_\l$?)} 
   so that $d_{\lambda_\l}(\tdisc)^N \leq \lambda_\l \leq \alpha^N/\varepsilon'$ and for
   $\l > L''$ (only applicable if $\lambda_n < \alpha^N$) we have
   $d_{\lambda_\l}(\tdisc)^N \leq \alpha^N$. This yields
	\begin{align*}
	    A_3 = \sum_{\l=\L'+1}^{n} \frac{ \d_{\lambda_\l}^N(\tdisc)}{ \d_{\lambda_1}^N(\tdisc)} \le 
	  \frac{2}{\lambda_1}\left( \sum_{\l=L'+1}^{L''} \lambda_\l + (n-L'') \alpha^N\right) 
	    \leq \frac{2}{\lambda_1}(n-L') \frac{\alpha^N}{\varepsilon'}.
	\end{align*}
   This concludes the proof.
\end{proof}

\subsection{Our work in light of \cite{Gissin2019Implicit}}
\label{sec:Comparison}

Theorems \ref{prop:EffectiveRankContinuousPSD} and \ref{prop:EffectiveRankDiscretePSD} only make a non-trivial claim if $\alpha$, $\varepsilon$, and $\eta$ are chosen in a way such that $I_1 \cap I_2 \cap I_3 \neq \emptyset$ (resp.\ the set of valid choices for $\tdisc$ in \eqref{eq:EffectiveRankDiscretePSD} is non-empty). In \cite[Theorems 2 \& 3]{Gissin2019Implicit} the authors characterize, for any \HR{pair of eigenvalues} $\lambda_i,\lambda_j$ of a positive semi-definite $\Wstar$ with $i,j \in [n]$, the maximal choice of $\alpha$ such that $\lambda_i$ and $\lambda_j$ are well-approximated at distinguishable times. Applying these results to $\lambda_L$ and $\lambda_{L+1}$, we thus can get a priori a necessary condition on $\alpha$ for the existence of the $\L$-th effective rank plateau. Note, however, that the result on gradient descent \cite[Theorem 3]{Gissin2019Implicit} only holds for the case $N=2$. Although there is a partial overlap of theory between \cite{Gissin2019Implicit} and our work, the main difference of \cite[Theorems 2 \& 3]{Gissin2019Implicit} and our Theorems \ref{prop:EffectiveRankContinuousPSD} \& \ref{prop:EffectiveRankDiscretePSD} is that the former answer the question \emph{whether} a plateau exists, whereas the latter characterize \emph{the time at which} the plateaus occur if they exist.

\section{Numerical Simulations}
\label{sec:Numerics}
We have already demonstrated numerical results for our exact setting in Fig.~\ref{fig:IdenticalVSPerturbed} (effects of perturbation), Fig.~\ref{fig:IdenticalwithDifferentN} (effects of number of layers), Fig.~\ref{fig:Tfigure} (accuracy of the prediction of a single eigenvalue), Fig.~\ref{fig:EffRankSection1} (accuracy of low rank approximation), and Fig.~\ref{fig:EffRankSection2} (difference between gradient flow and gradient descent).

In this section, we would like to numerically explore whether our findings also hold in more general situations. We demonstrate the impact of implicit bias and the "waterfall" behavior of gradient descent on de-noising of real data. It should be noted that this experiment does not fully lie in the scope of the theory presented in the paper since the setting is not symmetric and the initialization is random. It shall illustrate generalizability of our findings. \\
We consider an example from the MNIST-dataset \cite{mnist}. MNIST consists of images of handwritten digits from one to nine. All images have a resolution of $28\times28$ and each of the $28^2 = 784$ pixels takes values in $\{0,\dots, 255\}$. For our simulation, we take 100 pictures of ones from the MNIST-dataset and create a matrix 
\begin{equation*}
    \W_\text{LR} \in \{0,..., 255\}^{100 \times 784}
\end{equation*}
in which each row is a vectorized MNIST-one. We run gradient descent with factorization depth $N=1, 2, 3, 4$ on a noisy version $\Wstar = \W_\text{LR} + \Xi$ of the ground truth. We consider uniform noise, i.e., the matrix $\Xi$ satisfies
\begin{equation*}
    \Xi_{k,l} \stackrel{\text{iid}}{\sim} \mathcal{U}\big(\{-255,255\}\big)\;.
\end{equation*}
Moreover, the factorizations are initialized by zero mean Gaussian random matrices
\begin{equation*}
    \W_1(0)\in\mathbb{R}^{100\times 784}\;\text{ and }\;\W_j(0)\in\mathbb{R}^{784\times 784} \; j\ge 2
\end{equation*}
with small standard deviation $\sigma = \frac{1}{\sqrt{n}}$ (so that the expected norm of initialization $\mathbb{E}\norm{\W(0)}{F}
^2= k n^N \sigma^{2N} = k$ remains constant for all $N$). The step-size is chosen as $\eta = 10^{-7}$. Gradient descent is run on each factor matrix $\W_j$ as described in Section \ref{sec:Introduction} until a loss smaller than $10^{-5}\|\Wstar\|_F^2$ is reached.

Figures \ref{fig: MNIST: Noise: SingularValues}-\ref{fig: MNIST-Ones N=3} illustrate setting and outcome of the experiment. Selected rows of $\W_\text{LR}$ and $\Wstar = \W_\text{LR} + \Xi$ (reshaped to $28\times 28$-pixel images) are depicted in Figure \ref{fig: MNIST: OriginalAndNoisyOne}. Figure \ref{fig: MNIST: Noise: SingularValues} shows that the singular values of the end-to-end iterates $\W(\tdisc)$ show several properties derived in the theory of this paper. We observe that deeper factorization indeed provokes sharper transition. Here deeper factorization converges faster because the leading eigenvalues are very large.
% However, deeper factorization also converges faster, which is different from the simplified dynamics we have analyzed. One possible explanation is that although the initialization is of the same size, deeper factorization introduces more randomness and hence accelerates the convergence, a phenomenon also discussed in others' work, e.g., \cite{arora2018optimization}.

The behavior of the training error (Figure \ref{fig: MNIST: Noise: Error}), generalization error (Figure \ref{fig: MNIST: Noise: GenError}) and the effective rank (Figure \ref{fig: MNIST: Noise: Effective Rank}) support the theory we provided in a more restricted setting in Section \ref{sec:EffectiveRank}. Note that due to very different scale, the case $N=1$ is illustrated with an individual axis in each of the plots. As Figure \ref{fig: MNIST: Noise: Effective Rank} shows, using factorizations (i.e. $N=2,3,4$) yields low-rank iterates at an early stage of optimization, while for $N=1$ all iterates have a high rank. Consequently, and in accordance with Section \ref{sec:EffectiveRank} the factorized versions allow de-noising when stopped at an appropriate time.
% (in this context it is worth mentioning that though higher $N$ earlier leads to low-rank iterates, the range of feasible stopping times is much larger for $N$ small \revision{(I am not sure what this means either. Any idea?)})
This is not the case for the non-factorized version $N=1$.\\
The last point becomes particularly clear when comparing Figures \ref{fig: MNIST-Ones N=1}-\ref{fig: MNIST-Ones N=3} depicting one row of $\W(\tdisc)$ (reshaped to $28\times 28$-pixels) for different $\tdisc\in\mathbb{N}$ and $N$. The values of $\tdisc$ are chosen according to Figure \ref{fig: MNIST: Noise: Effective Rank} such that the different regimes w.r.t to the effective rank are represented.

%Singular Values during Optimization
\begin{figure}
%\centering
\begin{subfigure}{.5\textwidth}
  \centering
  \includegraphics[width=1\linewidth]{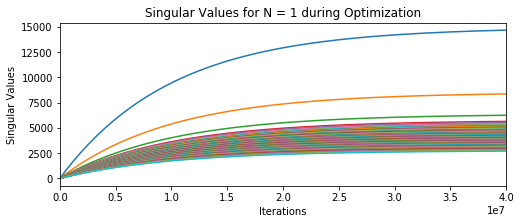}
  \caption{Singular Value Dynamics for $N=1$.}
  \label{fig: MNIST: Noise: SingularValues: N=1}
\end{subfigure}
\begin{subfigure}{.5\textwidth}
  \centering
  \includegraphics[width=1\linewidth]{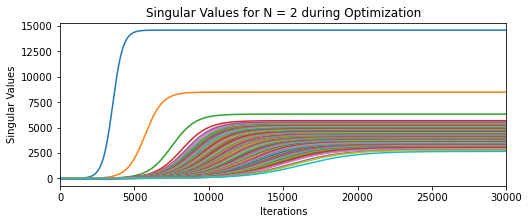}
  \caption{Singular Value Dynamics for $N=2$.}
  \label{fig: MNIST: Noise: SingularValues: N=2}
\end{subfigure}
\begin{subfigure}{.5\textwidth}
  \centering
  \includegraphics[width=1\linewidth]{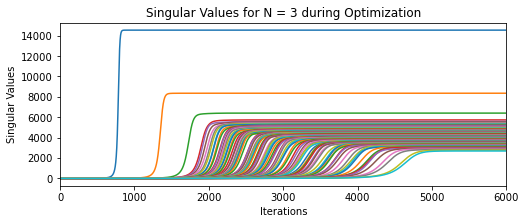}
  \caption{Singular Value Dynamics for $N=3$.}
  \label{fig: MNIST: Noise: SingularValues: N=3}
\end{subfigure}
\begin{subfigure}{.5\textwidth}
  \centering
  \includegraphics[width=1\linewidth]{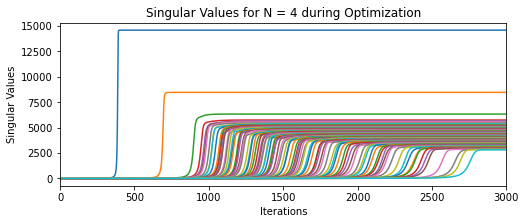}
  \caption{Singular Value Dynamics for $N=4$.}
  \label{fig: MNIST: Noise: SingularValues: N=4}
\end{subfigure}
\caption{MNIST: Illustration of singular calues of $\W(\tdisc)$ during optimization. Note that deeper factorization makes the convergence of each eigenvalue sharper, resulting in more distinguishable dynamics between eigenvalues.}
\label{fig: MNIST: Noise: SingularValues}
\end{figure}

% Illustrating Errors and Ranks
\begin{figure}
\centering
\begin{subfigure}{.5\textwidth}
  \centering
\includegraphics[width=1\linewidth]{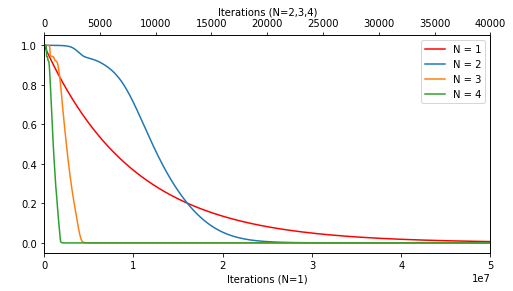}
  \caption{ Relative Loss $\| \W(\tdisc)-\Wstar \|_{F} / \| \Wstar \|_{F}$.}
  \label{fig: MNIST: Noise: Error}
\end{subfigure}%
\begin{subfigure}{.5\textwidth}
  \centering
\includegraphics[width=1\linewidth]{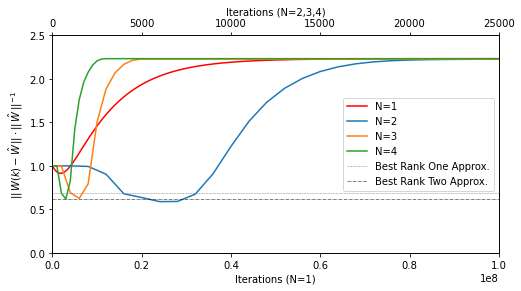}
  \caption{ Rel. Approx. Error $\norm{\W(\tdisc)-\W_\text{LR}}{F}/ \norm{\W_\text{LR}}{F} $.}
  \label{fig: MNIST: Noise: GenError}
\end{subfigure}%

\begin{subfigure}{.5\textwidth}
  \centering
  \includegraphics[width=1\linewidth]{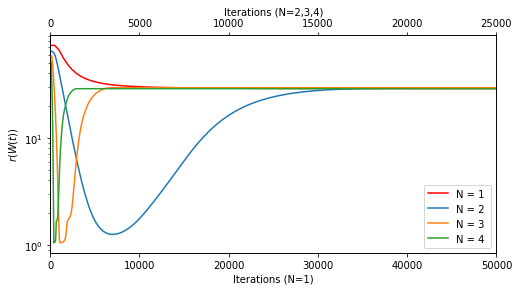}
  \caption{ Effective Rank $r\big(\W(\tdisc)\big)$, c.f.\ Section \ref{sec:EffectiveRank}. }
  \label{fig: MNIST: Noise: Effective Rank}
\end{subfigure}
\begin{subfigure}{.45\textwidth}
  \centering
  \includegraphics[width=1\linewidth]{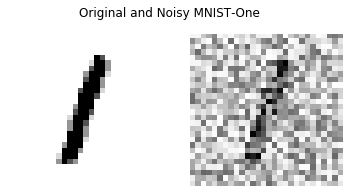}
  \caption{Original and noisy MNIST-One.}
  \label{fig: MNIST: OriginalAndNoisyOne}
\end{subfigure}

\caption{MNIST: Different properties during Optimization} \label{fig: MNIST: Different properties during Optimization} {\small 
The figures (a)-(c) illustrate different properties of the gradient descent iterates during optimization. Note that the properties for the case $N=1$ use a different axis than the cases $N=2,3,4$. Sub-figure (b) includes two different de-noising approaches as benchmarks: the best rank-one and rank-two approximation of $\Wstar$ (under all best rank approximations of $\Wstar$, the rank-two approximation proved to be closest to $\Wstar$ in Frobenius metric). Sub-figure (d) illustrates a MNIST-One and a noisy version.}
% , corresponding to reshaped rows of the matrices in Figure \ref{fig: MNIST: Ground Truth}
\end{figure}

\begin{figure}
    \centering
    \includegraphics[width=1\linewidth]{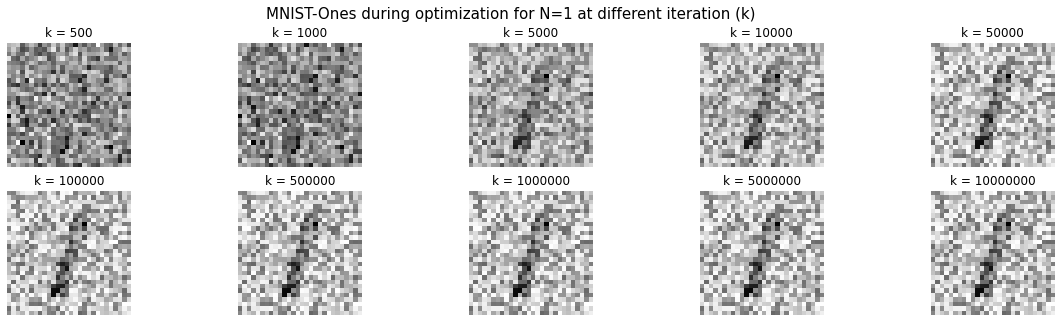}
    \caption{A row of $W(\tdisc)$ reshaped into $28\times 28$ for different $\tdisc$ and $N=1$. Compared to the case of $N=2$ and $N=3$, as shown in figure \ref{fig: MNIST-Ones N=2} and \ref{fig: MNIST-Ones N=3}, the noise remain large throughout iterations. Note that it converges much slower than $N=2$ and $N=3$.}
    \label{fig: MNIST-Ones N=1}
\end{figure}

\begin{figure}
    \centering
    \includegraphics[width=1\linewidth]{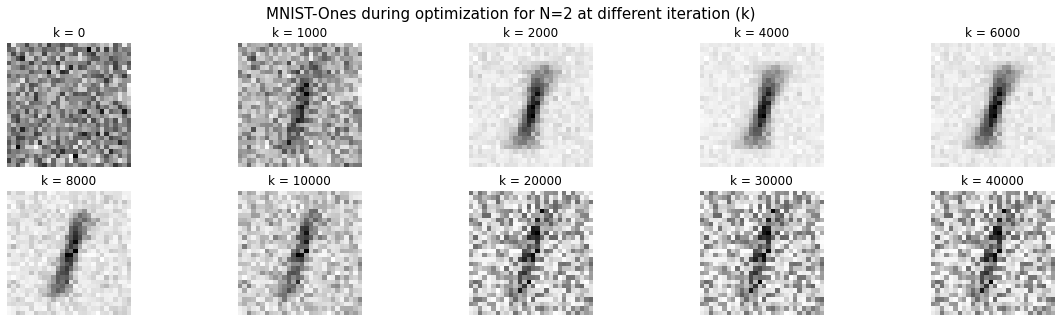}
    \caption{A row of $W(\tdisc)$ reshaped into $28\times 28$ for different $\tdisc$ and $N=2$. Compared to figure \ref{fig: MNIST-Ones N=1}, the noise is reduced significantly in the middle of training, but then increases later on. Note that the convergence rate is faster than $N=1$ but slower than $N=3$.}
    \label{fig: MNIST-Ones N=2}
\end{figure}

\begin{figure}
    \centering
    \includegraphics[width=1\linewidth]{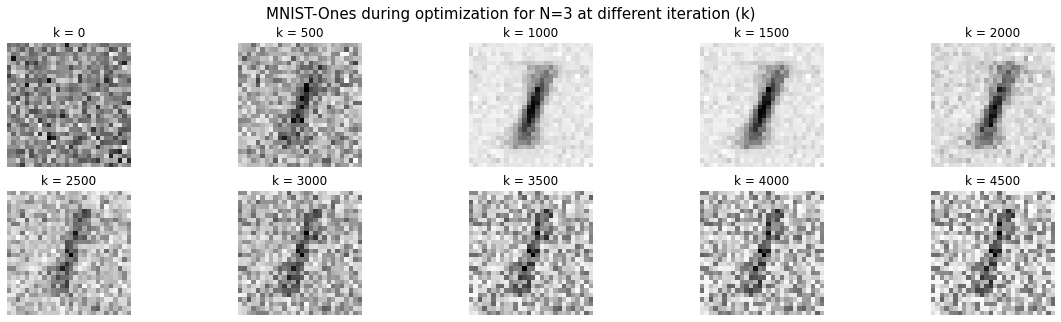}
    \caption{A row of $W(\tdisc)$ reshaped into $28\times 28$ for different $\tdisc$ and $N=3$. It exhibit similar behavior as in figure \ref{fig: MNIST-Ones N=2}, but at different intervals. Here the image is clear for $\tdisc\in[1000,2000]$, while in figure \ref{fig: MNIST-Ones N=2} the image is clear $\tdisc\in[2000,8000]$.}
    \label{fig: MNIST-Ones N=3}
\end{figure}

% Illustrating MNIST-Ones during Optimization

%-------------------------------------------------------------

\section{Discussion}
\label{sec:Discussion}

In this paper we approached in a simplified setting the self-regularizing effect of gradient descent in multi-layer matrix factorization problems. For symmetric ground-truths, we analyzed the dynamics of gradient descent and its underlying continuous flow, and explicitly characterized the effective rank of gradient descent/flow iterates in dependence of model parameters like the spectrum of the ground truth matrix and number of layers of the factorization. In particular, we proved that early stopping of gradient descent produces effectively low-rank solutions. Numerical simulations both on toy and real data validated our theory. Viewing matrix factorization as training of a linear neural network, we believe that our results yield valuable insights in the implicit low-rank regularization of gradient descent observed in recent deep learning research. Extending the theory to more general settings should help to enlighten the implicit bias phenomenon of gradient descent.

We envision several directions for potential future work. First, we expect similar results for non-symmetric and rectangular ground-truths by using the singular value instead of the eigenvalue decomposition. Extending the theory accordingly, however, requires additional work on a technical level.

Second, a widely used initialization for gradient descent is to randomly and independently draw the entries of the initial gradient descent estimate from a Gaussian distribution, cf.\ Xavier initialization in deep neural networks \cite{glorot2010understanding}. Compared to the initializations we considered, numerical simulations suggest that random initialization converges faster but exhibits a more complicated behavior which heavily depends on the initialization variance $\alpha^2$, i.e. $\W_j(0)=\alpha_j I$ where $\alpha_j$ are i.i.d. with $\mathbb{E}(\alpha_j)= 0$ and $\mathbb{E}(\alpha_j^2)= \alpha^2$. In our experiment we take $\alpha_j$ to be Gaussian. For instance, Figure \ref{fig:RandInt} shows that, for $\sigma$ large, the order of eigenvalue approximation is not determined by the eigenvalue magnitudes, thus indicating that the above described phenomena of (effective) rank approximation do not hold in this case; in contrast, for $\sigma$ small, we recognize the dynamics to be similar to the one of our perturbed initialization. We assume that those observations can be rigorously stated and proved in appropriate probabilistic settings, but for now we will leave it to future work.

Finally, it would be desirable to generalize our explicit effective rank analysis to low rank matrix sensing when we do not have full information of the ground truth. In this underdetermined setting additional ambiguities appear and regularization becomes even more meaningful. Nevertheless, the analysis is more challenging due to additional coupling between the variables.
\begin{figure}
\begin{subfigure}[c]{0.45\textwidth}
\includegraphics[width=\textwidth]{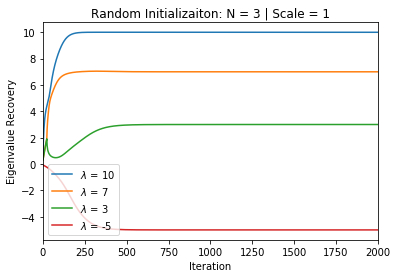}
\subcaption{$\alpha = 1$.}
\label{fig:RandInt_1}
\end{subfigure} \quad
\begin{subfigure}[c]{0.45\textwidth}
\includegraphics[width=\textwidth]{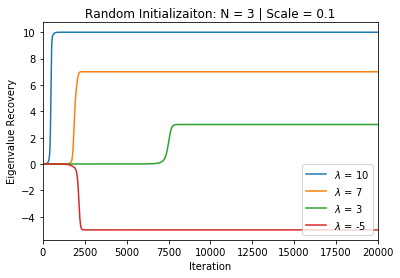}
\subcaption{$\alpha = 0.1$.}
\label{fig:RandInt_01}
\end{subfigure}
\caption{Random initialization: $\W_j(0)=\alpha_j I$ where $\alpha_j\sim\mathcal{N}(0,\alpha^2)$ are independent. When $\alpha$ is large, the eigenvalues converge faster, but the implicit bias phenomenon is less significant. Here $N=3$.}
\label{fig:RandInt}
\end{figure}
%

%-------------------------------------------------------------

\section*{Acknowledgements}
\label{sec:Acknowledgements}
HHC and HR acknowledge funding by the DAAD through the project {\it{Understanding stochastic gradient descent in deep learning}} (project no.\ 57417829).
JM and HR acknowledges funding by the Deutsche Forschungsgemeinschaft (DFG, German Research Foundation) through the project CoCoMIMO funded within the priority program SPP 1798 {\it{Compressed Sensing in Information Processing}} (COSIP).
HR acknowledges funding by the Federal Ministry of Education and Research (BMBF) and the Ministry of Culture and Science of the German State of North Rhine-Westphalia (MKW) under the Excellence Strategy of the Federal Government and the L{\"a}nder.
We wish to sincerely thank our colleagues Le Thang Huynh, Hans Christian Jung, and Ulrich Terstiege for the numerous joint discussions on the topic.

%-------------------------------------------------------------

\bibliography{references}{}
\bibliographystyle{abbrv}

%-------------------------------------------------------------

\appendix

\section{Supplement to Remark \ref{rem:main-thm}}
\label{sec:Remark_main-thm}

In this section, we provide a detailed derivation of \eqref{TID-form} in Remark \ref{rem:main-thm}. Recall that we restrict ourselves to the case $N \geq 3$, $0 < \epsilon^N \ll \lambda_i \leq \lambda_1$, and $0 < \alpha^N \ll \lambda_i$, so that the initial matrix $W(0) = \alpha^N \Id$ has small enough spectral norm compared to the $i$-th eigenvalue of the ground truth, which in turn is larger than the desired accuracy $\epsilon^N$.
As mentioned in Remark \ref{rem:main-thm}, we have to assume that
\[
\eta = \frac{\kappa}{N \lambda_1^{2-\frac{2}{N}}}
\]
for some $\kappa \leq \frac{1}{3}$ so that \eqref{eta:cond:mthm} is satisfied. The quantity $T^\Id_N$ then takes the form (see the fourth case in \eqref{def:TId})
\begin{equation}\label{TID-form-appendix}
T^\text{Id}_N(\lambda_i,\epsilon,\alpha,\eta) = A(\lambda_i,\alpha,\eta) + B(\lambda_i,\epsilon,\eta) + s_N(\lambda_i,\alpha).
\end{equation}
The proof of Theorem~\ref{thm:IdenticalInitialization_ConvergenceRate} reveals that $A(\lambda_i,\alpha,\eta)$ is related to the time the corresponding $i$-th eigenvalue of the continuous dynamics needs to reach its ``inflection point'', $B(\lambda_i,\epsilon,\eta)$ refers to the number of iterations required to reach an $\epsilon$-accuracy approximation of the $i$-th eigenvalue starting from the ``inflection'' point and $s_N(\lambda_i,\alpha_i)$ is a term arising from comparing the discrete with the continuous dynamics in the phase before it reaches the ``inflection point''. This last term may be an artefact of the proof; a lower bound for the convergence time does not require $s_N(\lambda_i,\alpha)$, but it does (essentially) require the other two terms.

The additional time to reach $\epsilon$ accuracy once the ``inflection point'' is reached is very small. Indeed, an accuracy  $|E_{ii}(\tdisc) | \leq \epsilon' \lambda_i$ (meaning relative accuracy $\epsilon' \in (0,1)$) is reached for
\[
\epsilon = \epsilon' \lambda_i^{1/N}/N
\]
by \eqref{eq:error-estimate} and for this choice of $\epsilon$
the quantity $B$ is given by
\[
B(\lambda_i,\epsilon,\eta) = \frac{\ln(N/\epsilon') - a_N}{\left|\ln\left(1-\eta N (c_N \lambda_i)^{2-\frac{2}{N}}\right)\right|} \leq C_N \ln(1/\epsilon').
\]
(where $a_N$ and $c_N$ are defined in the next section).
Ignoring the term $s_N(\lambda_i,\alpha)$ for the moment (which may be a proof artefact) shows that the convergence time is basically determined by $A(\lambda_i, \alpha,\eta)$, i.e., the time to reach the ``inflection point''.

An analysis of the exact expression for $A(\lambda_i, \alpha,\eta)$, see \eqref{def:T+-} and Lemma~\ref{lem:approximate:TId}, shows that, for $\alpha^N \ll \lambda_i$,
\begin{align*}
A(\lambda_i, \alpha,\eta) = \frac{1}{\eta} T_N^+(\lambda_i,(c_N\lambda_i)^{\frac{1}{N}},\alpha) & = \kappa^{-1}\left(\frac{\lambda_1}{\lambda_i}\right)^{2-\frac{2}{N}}\left( \frac{1}{N-2}\left( \frac{\lambda_i}{\alpha^N}\right)^{1-\frac{2}{N}} + N \mathcal{O}(1) \right)\\
& \sim \frac{1}{\kappa(N-2)} \left(\frac{\lambda_1}{\alpha^N}\right)^{2-\frac{2}{N}} \frac{\alpha^N}{\lambda_i}. 
%=  \frac{1}{\kappa(N-2)} \frac{\lambda_1^{2-\frac{2}{N}}}{(\alpha^N)^{1-\frac{2}{N}}} \frac{1}{\lambda_i}. 
\end{align*}
%so it is omitted here. 
This means that the larger an eigenvalue $\lambda_i$ is in relation to $\lambda_1$ and $\alpha^N$, the smaller $A(\lambda_i,\alpha,\eta)$ is %$T^\Id_N(\lambda_i,\epsilon,\alpha,\eta)$ 
and the faster it is approximated by gradient descent, see also Figure \ref{fig:IdenticalVSPerturbedA} for an illustration. 
Moreover, the exponent $2-\frac{2}{N}$ at $\lambda_1/\alpha^N$ in the approximate expression for $A(\lambda_i,\alpha,\eta)$ leads to the fact that the differences of consecutive ``relative inverse eigenvalues'' $\frac{\alpha^N}{\lambda_i}$, $i \in [n]$, are ``stretched out'' more with increasing $N$. Since the dynamics of the eigenvalues stays close to zero for a long time before reaching the inflection point (for small $\alpha^N$), this has the effect that different eigenvalues can be distinguished by their convergence time more easily for larger $N$, see again Figure~\ref{fig:IdenticalVSPerturbedA}. In turn this leads to a dynamics for the matrix $W(\tdisc)$ with low rank approximations in the initial phase and plateaulike increasing effective rank, see also Figure~\ref{fig:EffRank_Intro}.

Let us finally discuss the third term $s_N(\lambda_i,\alpha)$ in \eqref{TID-form-appendix} given by
\[
s_N(\lambda_i,\alpha) = 
\left\lceil c_N^{1-\frac{1}{N}} \left(\frac{\lambda_i}{\alpha^N}\right)^{1-\frac{1}{N}} \right\rceil \sim C_N \left(\frac{\lambda_i}{\alpha^N} \right)^{1-\frac{1}{N}}.
\]
In order to judge on the influence of $s_N(\lambda_i,\alpha)$ on the above discussion, let us compare it to $A(\lambda_i,\alpha,\eta)$ by forming the fraction
\[
\frac{s_N(\lambda_i,\alpha)}{A(\lambda_i,\alpha,\eta)} \sim  D_N \kappa \left(\frac{\lambda_i}{\lambda_1}\right)^{2-\frac{2}{N}} \left( \frac{\lambda_i}{\alpha^N}\right)^\frac{1}{N}.
\]
This means that while there is a non-negligible contribution of $s_N(\lambda_i,\alpha)$ to $T^\Id_N$ for eigenvalues $\lambda_i$ close to $\lambda_1$ (and in particular, for $\lambda_i = \lambda_1$), the  contribution does become negligible for relatively small $\lambda_i$. Moreover, larger $N$ again helps. Also note, that a small constant $\kappa$ can also reduce the influence of $s_N(\lambda, \alpha_i)$. In particular, in the situation where we would like to distinguish significantly different eigenvalues (i.e., different from $\lambda_1$) from their dynamics, it is valid to ignore $s_N$ and the above discussion taking into account only $A$ and $B$ applies.

\section{Optimality of Lemma \ref{lemma:IdenticalInitialization_Convergence}}
\label{sec:Appendix_Optimality}

The following lemma shows, that the condition on $\eta$ in Lemma \ref{lemma:IdenticalInitialization_Convergence} is necessary up to a constant, since the fixed point $\lambda^{\frac{1}{N}}$ becomes unstable otherwise. We say that a fixed point $a$ of the iteration $x_{n+1} = g(x_n)$ is unstable if $g(a) = a$ and $|g'(a)|>1$. 
In particular, this means that iterations move away from the fix point $a$ once they are in a small enough neighborhood of $a$, but do not reach $a$ exactly.

\begin{lemma}\label{lemma:IdenticalInitialization_Convergence_Necessary}
    Let $\d$ be the solution of \eqref{eq:IdenticalInitialization_ScalarDynamics} for some $\lambda \in \RR$. If $N=1$ and $\eta >2$, then %$\d=\lambda$ is an unstable equilibrium. 
    the sequence $\d(\tdisc)$ diverges as $\tdisc \to \infty$ unless $\d(0) = \alpha = \lambda$. If $N\geq 2$, $\lambda>0$, and $\eta>2(N\lambda^{2-\frac{2}{N}})^{-1}$, then $\lambda^{\frac{1}{N}}$ is an unstable equilibrium of the iterates $\d(\tdisc)$. If $N\geq 2$, $\lambda< 0$, and 
    \begin{equation*}
        \eta > \frac{(1+2^{\frac{1}{N}})\max(\alpha, |\lambda|^{\frac{1}{N}})}{\min(\alpha,|\lambda|^{\frac{1}{N}})^{2N-1}},
    \end{equation*}
    then $d$ diverges.
\end{lemma}
\begin{proof}
    Let $g(x) = x - \eta x^{N-1}(x^N-\lambda)$ so that $\d(\tdisc+1) = g(\d(\tdisc))$. 
    %If $g(x)= x$ and $|g'(x)|>1$, the point $x$ is an unstable equilibrium. 
    For $N=1$ we have $g'(x) = 1-\eta$ so that $\eta>2$, $|g'(x)| = |1-\eta|>1$
    for all $x$. It follows that the iterates $\d$ defined by $\d(\tdisc+1) = g(\d(\tdisc)) $ form a diverging sequence unless $\d(0) = \lambda$.
    %\textcolor{red}{and hence $d$ diverges}. 
    For $N\geq 2$, $\lambda>0$, and $\eta > 2(N\lambda^{2-\frac{2}{N}})^{-1}$, we obtain
    \begin{align*}
        %g(\lambda^{\frac{1}{N}})
        %&= \lambda^{\frac{1}{N}} - \eta \lambda^{1-\frac{1}{N}}(\lambda - \lambda)
        %= \lambda^{\frac{1}{N}}\\
        |g'(\lambda^{\frac{1}{N}})| &= |1-\eta ((N-1)\lambda^{1-\frac{2}{N}}(\lambda-\lambda)+N\lambda^{2-\frac{2}{N}})|
        =|1-\eta N\lambda^{2-\frac{2}{N}}|
        >1.
    \end{align*}
    Hence, $\lambda^{\frac{1}{N}}$ is an unstable equilibrium of the dynamics $\d(\tdisc+1) = g(\d(\tdisc))$.
    %\textcolor{red}{$d$ diverges.} 
    For $N\geq 2$, $\lambda< 0$, the analysis is slightly more complicated because $g(0) = 0$ but $g'(0)=1$. However, if
    \begin{equation*}
        \eta>\frac{(1+2^{\frac{1}{N}})\max(\alpha, |\lambda|^{\frac{1}{N}})}{\min(\alpha,|\lambda|^{\frac{1}{N}})^{2N-1}},
    \end{equation*}
    then
    \begin{align*}
        \d(1)
        &= g(\d(0))
        = \alpha - \eta\alpha^{N-1}(\alpha^N-\lambda)
        \leq \alpha - \eta\alpha^{2N-1}\\ 
        &<  \alpha - (1+2^{\frac{1}{N}})\max(\alpha, |\lambda|^{\frac{1}{N}})
        \leq - 2^{\frac{1}{N}}\max(\alpha, |\lambda|^{\frac{1}{N}}).
    \end{align*}
    In particular note that $|\d(1)^N-\lambda| \geq |\d(1)|^N - |\lambda| > |\lambda|$. We will now show that $|\d(\tdisc)|$ is increasing by at last constant factor that is larger than one in each iteration and that $|d(\tdisc)| \geq (2|\lambda|)^{\frac{1}{N}}$ for all $\tdisc \in \mathbb{N}$. The latter inequality has just been shown for $\tdisc = 1$.
    So assume that it holds for some $\tdisc \in \mathbb{N}$. 
    %Suppose $|\d(\tdisc)| \geq (2|\lambda|)^{\frac{1}{N}}$. 
    Since $\d(\tdisc+1) = \d(\tdisc)(1 - \eta\d(\tdisc)^{N-2}(\d(\tdisc)^N-\lambda))$, it suffices to show that $|1 - \eta\d(\tdisc)^{N-2}(\d(\tdisc)^N-\lambda)| >1$. We have
    \begin{align*}
        |\eta\d(\tdisc)^{N-2}(\d(\tdisc)^N-\lambda)|
        &\geq \eta (2|\lambda|)^{1-\frac{2}{N}}|\lambda|
        = 2^{1-\frac{2}{N}} \eta |\lambda|^{2-\frac{2}{N}}
        > 2^{1-\frac{2}{N}} (1+2^{\frac{1}{N}})
        >2,
    \end{align*}
    which implies that
    $|\d(\tdisc+1)| >(1+c)|\d(\tdisc)|$ 
    with $c= 2^{1-\frac{2}{N}} (1+2^{\frac{1}{N}})-2>0$; in particular $\d(\tdisc + 1) >(2|\lambda|)^{\frac{1}{N}} $. Hence, $\d(\tdisc)$ diverges as $\tdisc \to \infty$. %and hence $0$ is an unstable equilibrium.
    %\textcolor{red}{H: Should we also discuss the case $\lambda \leq 0$?} \textcolor{blue}{Ed: For $\lambda \leq 0$ the argument might be a bit more complicated, because we will have to analyze the point $\d=0$, but $g'(0)=1$ always. We could take higher order derivatives to get some analysis, but I am not sure if it is worthwhile. Also, from simulation we see that for $\lambda <0$ our bounds are not so sharp, meaning that one could take much larger $\eta$ than we have in lemma 3.2.}
\end{proof}

\section{On Solutions of the Continuous Dynamics}
\label{sec:Appendix_ContinuousDynamics}

As claimed in Remark \ref{rem:ExplicitForm}, the solution $\y(t)$ of \eqref{eq:IdenticalInitialization_ContinuousScalarDynamics} can, for $N \geq 3$, be expressed in an alternative way that avoids complex logarithms. Introduce the function
\begin{equation*}
    h_N(\alpha,\y) :=
    \begin{cases}
        \ln\left|\frac{\lambda^{\frac{2}{N}} - \alpha^2}{\lambda^{\frac{2}{N}} - \y^2}\right| & \mbox{ if } $N$ \mbox{ is even},\\
         \ln\left|\frac{\lambda^{\frac{1}{N}} - \alpha}{\lambda^{\frac{1}{N}} - \y}\right| & \mbox{ if } $N$ \mbox{ is odd}.
    \end{cases}
\end{equation*}   
%   $$
%   h_N(\alpha,\y) := \left\{ \begin{array}{ll} \ln\left|\frac{\lambda^{\frac{2}{N}} - \alpha^2}{\lambda^{\frac{2}{N}} - \y^2}\right| & \mbox{ if } $N$ \mbox{ is even},\\
%   \ln\left|\frac{\lambda^{\frac{1}{N}} - \alpha}{\lambda^{\frac{1}{N}} - \y}\right| & \mbox{ if } $N$ \mbox{ is odd},
%   \end{array} \right.
%   $$
Then $\y(t)$ satisfies
\begin{align*}
    t & = \frac{\lambda^{\frac{2}{N}-2}}{N} \left[h_N(\alpha,\y(t)) + \sum_{\ell=1}^{\lfloor N/2-1 \rfloor} \cos\left(\frac{4\pi \ell}{N}\right) \ln\left|\frac{\alpha - \lambda^{\frac{1}{N}}e^{\frac{2\pi i \ell}{N}}}{\y - \lambda^{\frac{1}{N}} e^{\frac{2\pi i \ell}{N}}}\right|^2 \right. \\
    & \left. \phantom{=\frac{\lambda^{\frac{2}{N}-2}}{N} ()}
    + 2 \sum_{\ell=1}^{\lfloor N/2-1 \rfloor} \sin\left(\frac{4\pi \ell}{N}\right) \left( \arctan\left(\frac{\y - \lambda^{\frac{1}{N}} \cos(\frac{2\pi \ell}{N})}{\lambda^{\frac{2}{N}} \sin^2(\frac{2 \pi \ell}{N})} \right) - \arctan\left(\frac{\alpha - \lambda^{\frac{1}{N}} \cos(\frac{2\pi \ell}{N})}{\lambda^{\frac{1}{N}} \sin(\frac{2 \pi \ell}{N})} \right)\right)\right] \\
    & \quad + \frac{1}{\lambda(N-2)}\left(\frac{1}{\alpha^{N-2}} - \frac{1}{\y^{N-2}}\right).
\end{align*} 

The formula can be deduced from \eqref{def:Uplus} by splitting the complex logarithm into real and imaginary parts.

\section{Supplement to Section \ref{subsec:PerturbedInitialization}}
\label{sec:Appendix_Supplement}

We provide here Lemma \ref{lemma:PerturbedInitialization_Convergence_PositiveLargeLambda} and \ref{lemma:PerturbedInitialization_Convergence_PositiveSmallLambda}, which show the claim of Lemma \ref{lemma:PerturbedInitialization_Convergence_AllLambda} for non-negative $\lambda \ge \alpha^N$ and non-negative $\lambda < \alpha^N$, respectively.
In the following we repeatedly use the two auxiliary sequences
\begin{align}
    a(\tdisc+1) &= a(\tdisc) - \eta a(\tdisc)^{N-1}(a(\tdisc)^N-\lambda),
    \quad a(0) = \alpha -\beta >0, \tag{\ref{def:a_IdenticalInitialization}}\\
    b(\tdisc+1) &= b(\tdisc) - \eta b(\tdisc)^{N-1}(b(\tdisc)^N-\lambda),
    \quad b(0) = \alpha >0. \tag{\ref{def:b_IdenticalInitialization}}
\end{align}
already defined in Section~\ref{subsec:PerturbedInitialization} to control the trajectory of $d_1(\tdisc) d_2(\tdisc)^{N-1}$. We, furthermore, abbreviate
\begin{align}\label{def:p_pa_pb}
    p(\tdisc) = d_1(\tdisc) d_2(\tdisc)^{N-1},
    \quad
    p_a(\tdisc) = a(\tdisc)^N,
    \quad
    \text{ and}
    \quad
    p_b(\tdisc) = b(\tdisc)^N.
\end{align}
We observe that the product dynamics $p$ satisfies the following relation.
\begin{lemma} \label{lem:fp}
    Let $N \geq 2$ and $\lambda \geq 0$. Consider $p$ defined in \eqref{def:p_pa_pb} and a fixed $\tdisc \in \mathbb{N}_0$. Suppose
    that $p(\tdisc) \neq 0$ and $0<\eta < p(\tdisc)^{-2+\frac{2}{N}}$. 
    %For $p$ defined in \eqref{def:p_pa_pb}, we have
    Then
    \begin{align*}
        p(\tdisc+1) = p(\tdisc) f_{p(\tdisc)}(d_2(\tdisc)),
    \end{align*}
    where \revision{
    \begin{align*}
        f_{p}(x) := (1+c_{p,1} x^{2N-2})(1 + c_{p,2}x^{-2})^{N-1},
        \quad
        c_{p,1} := -\eta p^{-1}(p - \lambda),
        \quad\text{and }
        c_{p,2} := -\eta p(p - \lambda).
    \end{align*}}
    Moreover, $f_{p}'(x) = 0$ for $x = p^{\frac{1}{N}}$. If $c_{p,1},c_{p,2} \ge 0$, then on $(0,\infty)$, $f_{p}$ is convex and $f_{p}'$ has a unique zero. If $c_{p,1},c_{p,2} \le 0$ and $p > 0$, then $f_{p}$ is concave on $[p^{\frac{1}{N}},\infty)$.
\end{lemma}
\begin{proof}
    For simplicity we write $p = p(\tdisc), \d_1 = \d_1(\tdisc), \d_2 = \d_2(\tdisc)$ below.
    %Unless otherwise specified, all the functions are evaluated at time $\tdisc$, i.e., $p=p(\tdisc)$.
    By the definition of the dynamics in \eqref{eq:PerturbedInitialization_ScalarDynamics}, we have
    \begin{align} \label{eq:productReduction} 
    \begin{split}
        p(\tdisc+1)
        &= \d_1(\tdisc+1)\d_2(\tdisc+1)^{N-1}\\
        &= \left(\d_1 - \eta \d_2^{N-1}(\d_1\d_2^{N-1}-\lambda)\right)\left(\d_2 - \eta \d_1\d_2^{N-2}(\d_1\d_2^{N-1}-\lambda)\right)^{N-1}\\
        &= \d_1\left(1 - \eta \frac{\d_2^{2N-2}}{\d_1\d_2^{N-1}}(\d_1\d_2^{N-1}-\lambda)\right) \d_2^{N-1}\left(1 - \eta\frac{\d_1\d_2^{N-1}}{\d_2^2}(\d_1\d_2^{N-1}-\lambda)\right)^{N-1}\\
        &= p\left(1 - \eta p^{-1}(p-\lambda)\d_2^{2N-2}\right) \left(1 - \eta p(p-\lambda)\d_2^{-2}\right)^{N-1}
        %&= p(1+c_{p,1}\d_2^{2N-2})(1+c_{p,2}\d_2^{-2})^{N-1}\\
        = pf_p(\d_2).
    \end{split}
    \end{align}
    By differentiation, we obtain
    \begin{align}
        f'_p(x)
        &= (2N-2)c_{p,1}x^{2N-3}(1+c_{p,2}x^{-2})^{N-1} + (1+c_{p,1}x^{2N-2})(N-1)(1+c_{p,2}x^{-2})^{N-2}(-2c_{p,2}x^{-3})\nonumber \\
        &= (2N-2)(1+c_{p,2}x^{-2})^{N-2}[c_{p,1}x^{2N-3}(1+c_{p,2}x^{-2})-(1+c_{p,1}x^{2N-2})c_{p,2}x^{-3}]\label{fp_FirstDerivative}\\
        &= (2N-2)(1+c_{p,2}x^{-2})^{N-2}(c_{p,1}x^{2N-3}-c_{p,2}x^{-3}) \nonumber
    \end{align}
    and
    \begin{align}
        \frac{f''_p(x)}{2N-2} 
        &= (N-2)(1+c_{p,2}x^{-2})^{N-3}(-2c_{p,2}x^{-3})(c_{p,1}x^{2N-3}-c_{p,2}x^{-3})\nonumber\\
        &\quad + (1+c_{p,2}x^{-2})^{N-2}((2N-3)c_{p,1}x^{2N-4}+3c_{p,2}x^{-4})\nonumber\\
        &= (1+c_{p,2}x^{-2})^{N-3} [-2(N-2)(c_{p,1}c_{p,2}x^{2N-6}-c_{p,2}^2x^{-6})) \label{fp_SecondDerivative}\\
        &\quad + (2N-3)(c_{p,1}x^{2N-4}+c_{p,1}c_{p,2}x^{2N-6}) +3(c_{p,2}x^{-4}+c_{p,2}^2x^{-6})]\nonumber\\
        &= (1+c_{p,2}x^{-2})^{N-3}[3c_{p,2}x^{-4} + (2N-1)c_{p,2}^2x^{-6} + c_{p,1}c_{p,2}x^{2N-6} + (2N-3)c_{p,1}x^{2N-4} ]\nonumber.
    \end{align}
    It follows from \eqref{fp_FirstDerivative} that $f_p'(x) = 0$ if $x = \left(\frac{c_{p,2}}{c_{p,1}}\right)^{\frac{1}{2N}} = p^{\frac{1}{N}}$. If $x > 0$ and $c_{p,1},c_{p,2} \ge 0$, this zero of $f_p'$ is unique. Moreover, in this case the last line in $\eqref{fp_SecondDerivative}$ is clearly positive for $N\geq 2$, which implies that $f_p''(x) > 0$ so that $f_p$ is convex on $(0,\infty)$. For the last claim, note that under the assumptions on $\eta$, $p$, $c_{p,1}$, $c_{p,2}$, and $x$, implying that $\lambda \le p$,
    \begin{align*}
        &(1 + c_{p,2}x^{-2})
        \ge 1 - \eta p(p-\lambda)p^{-\frac{2}{N}}
        > 1 - p^{-2+\frac{2}{N}} p^{1-\frac{2}{N}}(p-\lambda)
        \geq 1 - p^{-1}(p-\lambda)
        > 0,\\
        &(-2c_{p,2}x^{-3}) = 
        2p(p-\lambda)x^{-3} > 0,\\
        &(c_{p,1}x^{2N-3}-c_{p,2}x^{-3})
        = x^{-3}(c_{p,1}x^{2N}-c_{p,1}p^2) < 0.
    \end{align*}
    It follows that the the expression after the first equality sign in \eqref{fp_SecondDerivative} is negative, that is, $f_p''<0$ and $f_p$ is concave on $[p^{\frac{1}{N}},\infty)$.
\end{proof}

\begin{lemma}\label{lemma:PerturbedInitialization_Convergence_PositiveLargeLambda}
    Let $N \geq 2$ and $\lambda > 0$. Let $\d_1,\d_2$ be defined by \eqref{eq:PerturbedInitialization_ScalarDynamics} with the perturbed identical initialization \eqref{eq:PerturbedInitialization_ScalarInitialization} for $N \geq 2$. Assume that $(\alpha-\beta)\alpha^{N-1}\leq\lambda$. Let $\M=\max\{\alpha, \lambda^{\frac{1}{N}}\}$ and $\c\in(1,2)$ be the maximal real solution to the polynomial equation $1=(\c-1)\c^{N-1}$. If
    \begin{equation}\label{C2:cond:stepsize}
        %\textcolor{blue}{
        0 < \eta < %\min\left\{\frac{2^{\frac{1}{N}}-1}{2} , \frac{1}{8N} \right\} \frac{1}{(c\M)^{2N-2}} =
        \frac{1}{8N \revision{(c\M)^{2N-2}}},%},
    \end{equation}
    then $0 < \d_1(\tdisc)\d_2(\tdisc)^{N-1}\leq \lambda$, $\d_2(\tdisc)\leq \c\M$ for all $\tdisc\in\mathbb{N}_0$, and $\lim_{\tdisc\to\infty}\d_1(\tdisc)\d_2(\tdisc)^{N-1} = \lambda$.
\end{lemma}
\begin{proof}
    We first note that $\c\in(1,2)$ because $h(x):=(x-1)x^{N-1}$ is continuous and satisfies $h(1)=0$, and $h(x)\geq 2^{N-1}$ for $x\geq 2$.
    
    Recall the sequences $a$, $p$, and $p_a$ defined in \eqref{def:a_IdenticalInitialization} and \eqref{def:p_pa_pb}. 
    %Unless specified otherwise, all \textcolor{blue}{sequences} are evaluated at time $\tdisc$, i.e., $p=p(\tdisc)$. 
    We will prove the claim by inductively showing that 
    \begin{equation}\label{C2:proof:claims}
    p_a(\tdisc) \leq p(\tdisc) \leq \lambda
    \quad \mbox{ and } \quad 0 < \d_1(\tdisc) < \d_2(\tdisc) \leq \c\M \quad \mbox{ for all } 
    \tdisc \in \mathbb{N}_0.
    \end{equation}
    Note that $p_a(\tdisc) = a(\tdisc)^N > 0$ for all $\tdisc$ by Lemma~\ref{lemma:IdenticalInitialization_Convergence}. %Since $p_a>0$, 
    Hence, $p_a(\tdisc) \leq p(\tdisc)$ will imply that $p(\tdisc)>0$, while $p(\tdisc)\leq \lambda\leq \M^N$ together with Condition~\eqref{C2:cond:stepsize} will lead to 
    \begin{equation}\label{eta:p:cond}
        \eta< (8N (cM)^{2N-2})^{-1} < p(\tdisc)^{-2+\frac{2}{N}}.
    \end{equation}%
    The claimed inequalities clearly hold for $\tdisc=0$, i.e., $p_a(0) = a(\tdisc)^N = (\alpha-\beta)^N < (\alpha-\beta) \alpha^{N-1} = p(0) \leq \lambda$ by assumption and $0 < \alpha = \d_2(0) \leq c M$. Assume now that the inequalities \eqref{C2:proof:claims} hold for some $\tdisc\in\mathbb{N}_0$. First, by \eqref{eq:PerturbedInitialization_ScalarDynamics}, $p(\tdisc) \le \lambda$ and $0 < \d_1(\tdisc) < \d_2(\tdisc)$, we have $d_2(\tdisc+1) \ge \d_1(\tdisc + 1) > d_1(\tdisc) > 0$. To show the remaining claims, note that if $p(\tdisc)=\lambda$, then $d_2(\tdisc+1) = d_2(\tdisc)$, \revision{$\d_1(\tdisc+1) = \d_1(\tdisc)$} and $p(\tdisc+1) = p(\tdisc)$ such that the claims trivially hold for $\tdisc+1$. Hence, it suffices to consider the case $p(\tdisc)<\lambda$. Let us start with some useful observations, where we often write $p = p(\tdisc)$, $\d_2=\d_2(\tdisc)$ etc.\ for simplicity. By Lemma~\ref{lem:fp} together with \eqref{eta:p:cond}, we have
    \begin{align*}
        p(\tdisc+1) = p f_p(\d_2),
    \end{align*}
    where $c_{p,1} = -\eta p^{-1}(p - \lambda)$ and $c_{p,2} = -\eta p(p - \lambda)$ are positive for $p \in (0,\lambda)$ so that $f_p$ is convex on the $(0,\infty)$ with unique global minimizer $p^{\frac{1}{N}}$.
    %cf.\ Lemma \ref{lem:fp}).
    Let $\Delta_1(\tdisc)=\d_2(\tdisc)-\d_1(\tdisc)$ and $\kappa(\tdisc)=\d_2^{N-2}(\tdisc)(p(\tdisc)-\lambda)$ be as in \eqref{def:Diff_SquareDiff_RateFactor} and note that $\eta\kappa(\tdisc) < 0$ is negative, while $\Delta_1(\tdisc) > 0$ by the induction hypothesis \eqref{C2:proof:claims}. Using the induction hypothesis another time, i.e., the last inequality in \eqref{C2:proof:claims}, together with \eqref{C2:cond:stepsize} it holds
    \[
    |\eta\kappa(\tdisc)|<\eta (\c\M)^{N-2} |\lambda| \leq \eta \c^{N-2} \M^{2N-2} < 1.
    \]%
    Lemma~\ref{lemma:DeltaContraction} implies that $\d_2(\tdisc+1)-\d_1(\tdisc+1)=\Delta_1(\tdisc+1)\geq 0$, so that \revision{$\d_1(\tdisc+1) < \d_2(\tdisc+1)$}. Since $\d_1(\tdisc)\d_2^{N-1}(\tdisc)=p(\tdisc)$, the induction hypothesis $0 < \d_1(\tdisc) < \d_2(\tdisc) \leq cM$
    gives $p^{\frac{1}{N}}(\tdisc)\leq\d_2(\tdisc)\leq \c\M$.
    %, where the second inequality is the induction hypothesis. 
    Because $f_{p(\tdisc)}$ is increasing on $[p(\tdisc)^{\frac{1}{N}},\infty)$,
    \begin{equation}\label{eq:fpCompare}
        f_{p(\tdisc)}(p(\tdisc)^{\frac{1}{N}})
        \leq f_{p(\tdisc)}(\d_2(\tdisc))
        \leq f_{p(\tdisc)}(\c\M).
    \end{equation}
    
    We now have all necessary tools to prove that $p_a(\tdisc+1)\leq p(\tdisc+1)\leq\lambda$. First, we show that $p_a(\tdisc+1)\leq p(\tdisc+1)$. By \eqref{eq:fpCompare},
    \begin{align*}
        p(\tdisc+1)
        &= pf_{p}(\d_2)
        \geq pf_{p}(p^{\frac{1}{N}})
        = p(1 - \eta p^{-1}(p-\lambda)p^{2-\frac{2}{N}}) (1 - \eta p(p-\lambda)p^{-\frac{2}{N}})^{N-1}\\
        &= p(1 - \eta p^{1-\frac{2}{N}}(p-\lambda))^{N}.
    \end{align*}
    Recall that
    \begin{align*}
        p_a(\tdisc+1)
        &= (a-\eta a^{N-1}(a^N-\lambda))^N
        = a^N(1-\eta a^{N-2}(a^N-\lambda))^N
        = p_a(1-\eta p_a^{1-\frac{2}{N}}(p_a-\lambda))^N.
    \end{align*}
    %and the assumption that $0<p_a\leq p\leq\lambda$. 
    Using the induction hypothesis $0 < p_a(\tdisc) \leq p(\tdisc) \leq \lambda$ and $\eta \leq \left(N \lambda^{2-\frac{2}{N}}\right)^{-1}$ 
    by \eqref{C2:cond:stepsize} we obtain
    \begin{align*}
        &p(\tdisc+1)^{\frac{1}{N}} - p_a(\tdisc+1)^{\frac{1}{N}}
        \geq p^\frac{1}{N}(1 - \eta p^{1-\frac{2}{N}}(p-\lambda)) - p_a^{\frac{1}{N}}(1-\eta p_a^{1-\frac{2}{N}}(p_a-\lambda))\\
        &= p^\frac{1}{N} - p_a^{\frac{1}{N}}+ \eta( p^{1-\frac{1}{N}}(\lambda-p) - p_a^{1-\frac{1}{N}}(\lambda-p_a))
        \geq p^\frac{1}{N} - p_a^{\frac{1}{N}} + \eta( p^{1-\frac{1}{N}}(\lambda-p) - p^{1-\frac{1}{N}}(\lambda-p_a))\\
        &= p^\frac{1}{N} - p_a^{\frac{1}{N}} + \eta p^{1-\frac{1}{N}}( p_a - p)
        = (p^\frac{1}{N} - p_a^{\frac{1}{N}})\left(1-\eta p^{1-\frac{1}{N}}\sum_{\ell=0}^{N-1}p^{\frac{\ell}{N}}p_a^{\frac{N-1-\ell}{N}}\right)\\
        &\geq (p^\frac{1}{N} - p_a^{\frac{1}{N}})(1-\eta N \lambda^{2-\frac{2}{N}})
        \geq 0,
    \end{align*}
    and we arrive at the induction step $p_a(\tdisc+1)\leq p(\tdisc+1)$.
    
    As a next step, we show that $p(\tdisc+1)\leq\lambda$. We distinguish two cases: either $2p(\tdisc)<\lambda$ or $2p(\tdisc)\geq\lambda$. Suppose first $2p(\tdisc)<\lambda$. Using \eqref{eq:fpCompare} another time together with $\lambda \leq (c M)^N$, we obtain
    \begin{align*}
        &p(\tdisc+1) \leq pf_p(\c\M)
        = \left(p - \eta (p-\lambda)(\c\M)^{2N-2}\right) \left(1 - \eta p(p-\lambda)(\c\M)^{-2}\right)^{N-1} \\
        &\leq \left(\frac{\lambda}{2} + \eta\lambda (\c\M)^{2N-2}\right) \left(1 + \eta\frac{\lambda^2}{2} (\c\M)^{-2}\right)^{N-1} 
        \leq \frac{\lambda}{2} \left(1 + 2\eta (\c\M)^{2N-2}\right) \left(1 + \frac{1}{2}\eta(\c\M)^{2N-2}\right)^{N-1} \\
        &\leq \frac{\lambda}{2} \left(1 + 2\eta (\c\M)^{2N-2}\right)^{N}
        \leq \lambda
    \end{align*}
    if $\eta \leq \frac{2^{\frac{1}{N}}-1}{2 (\c\M)^{2N-2}}$.
    Since $8\log(2)> 2$ and $x\geq \log(1+x)$ for $x\in(0,1)$, it holds 
    \begin{equation}\label{aux:eta}
    \log(2^{1/N}) = \frac{\log(2)}{N}
        > \frac{2}{8N}
        \geq \log\left(\frac{2}{8N} + 1\right), 
    \end{equation}
    so that $(2^{\frac{1}{N}}-1)/2 > 1/(8N)$ and \eqref{C2:cond:stepsize} implies the required condition on $\eta$.
    % \textcolor{red}{Since $9\log(2)\geq 2$ and $x\geq \log(1+x)$ for $x\in(0,1)$, we have that}
    %\begin{align*}
     %   \frac{\log(2)}{N}
      %  &\geq \frac{2}{9N}
       % \geq \log(\frac{2}{9N} + 1)\\
    %    \implies
    %    2^{\frac{1}{N}}
     %   &\geq \frac{2}{9N} + 1.
    %\end{align*}
    Now suppose $\lambda< 2p(\tdisc)\leq 2\lambda$. Using $\lambda < (cM)^{N}$ another time gives
    \begin{align*}
        \log(pf_p(\c\M))
        &= \log(p) +\log\left(1 - \eta p^{-1}(p-\lambda)(\c\M)^{2N-2}\right) +(N-1)\log \left(1 - \eta p(p-\lambda)(\c\M)^{-2}\right)\\
        &\leq \log(p) +\log\left(1 + 2 \eta \lambda^{-1}(\c\M)^{2N-2}(\lambda-p)\right) +(N-1)\log\left(1 + \eta \lambda(\c\M)^{-2}(\lambda-p)\right)\\
        &\leq \log(p) + N\log\left(1 + 2 \eta \lambda^{-1}(\c\M)^{2N-2}(\lambda-p)\right).
    \end{align*}
    %\textcolor{blue}{The above inequality seems to use that $\lambda (cM)^{-2} \leq \lambda^{-1}(cM)^{2N-2}$. But this seems to require that $\lambda \lesssim 1$. It seems that we also need to fix this point. Perhaps work with $p/\lambda$ instead of $p$?} \textcolor{red}{We just need $\lambda\leq (cM)^N$ here.}
    Since $x\geq \log(1+x)\geq x-\frac{1}{2}x^2$ by Taylor expansion, we obtain, using again the induction hypothesis that $\d_2(\tdisc) \leq \c\M$,
    \begin{align*}
        \log (\lambda) - \log(p(\tdisc+1))
        &=\log (\lambda) - \log(p(\tdisc)f_{p(\tdisc)}(\d_2(\tdisc))) 
        \geq \log (\lambda) - \log(p f_p(\c\M))\\
        &= \log(\lambda) - \log(p) -  N\log\left(1 + 2 \eta \lambda^{-1}(\c\M)^{2N-2}(\lambda-p)\right)\\
        &= \log\left(1+\frac{\lambda-p}{p}\right) - N\log\left(1 + 2 \eta \lambda^{-1}(\c\M)^{2N-2}(\lambda-p)\right)\\
        &\geq \frac{\lambda-p}{p} - \frac{1}{2}\left(\frac{\lambda-p}{p}\right)^2 - 2\eta N\lambda^{-1}(\c\M)^{2N-2}(\lambda-p)\\
        &= \frac{\lambda-p}{p^2} \left(p- \frac{1}{2}(\lambda-p) - 2\eta N\lambda^{-1}(\c\M)^{2N-2}p^2 \right)
        \geq \frac{\lambda-p}{p^2} \left( \frac{\lambda}{4} - 2\eta N\lambda(\c\M)^{2N-2} \right)\\
        & \geq 0
    \end{align*}
    since $\eta \leq \frac{1}{8N(\c\M)^{2N-2}}$. Thus $p(\tdisc+1)\leq\lambda$.
    
    It remains to show that $\d_2(\tdisc+1) \leq \c\M$. 
    Using the induction hypothesis $\d_2(\tdisc') \leq \c\M$ and $0 < p(\tdisc') \leq \lambda$ for all $\tdisc' = 0, \hdots, \tdisc$, it follows that
    \revision{
    \begin{equation*}
        0>\eta \kappa(\tdisc')= \eta \d_2^{N-2}(\tdisc')(p(\tdisc')-\lambda) \ge -\eta (\c\M)^{2N - 2} > -1
    \end{equation*}
    for all $\tdisc' = 0,\hdots,\tdisc$. Therefore, Lemma~\ref{lemma:DeltaContraction}} together with $\Delta_1(0) = \beta > 0$ implies that $\Delta_1(\tdisc') = \d_2(\tdisc')-\d_1(\tdisc') > 0$ and $\Delta_1(\tdisc'+1) > \Delta_1(\tdisc')$ for all $\tdisc' = 0,\hdots,\tdisc$. This gives
    \begin{align*}
        \d_2(\tdisc+1)
        \geq\d_1(\tdisc+1)
        \geq\d_2(\tdisc+1)-\d_2(0)+\d_1(0)
        = \d_2(\tdisc+1) - \beta.
    \end{align*}
    Now assume that $\d_2(\tdisc+1)> c\M$. Then
    \begin{align*}
        p(\tdisc+1)
        =\d_1(\tdisc+1)\d_2(\tdisc+1)^{N-1}
        &> (\c\M-\beta)(\c\M)^{N-1}
        > (\c\M-\alpha)(\c\M)^{N-1}\\
        &\geq (\c-1)\max\{\alpha, \lambda^{\frac{1}{N}}\} \c^{N-1}\max\{\alpha, \lambda^{\frac{1}{N}}\}^{N-1}
        = \max\{\alpha^N, \lambda\} \geq \lambda,
    \end{align*}
    which contradicts $p(\tdisc+1)\leq\lambda$ as shown above. Hence $\d_2(\tdisc+1) \leq \c\M$.  
    
    In particular, we have shown that $p_a(\tdisc)\leq p(\tdisc) \leq \lambda$ for all $\tdisc \in \mathbb{N}_0$. Since $\lim_{\tdisc\to\infty}p_a(\tdisc) = \lambda$ by Lemma~\ref{lemma:IdenticalInitialization_Convergence}, we obtain $\lim_{\tdisc\to\infty}\d_1(\tdisc)\d_2(\tdisc)^{N-1} = \lambda$.
\end{proof}
\begin{lemma}\label{lemma:PerturbedInitialization_Convergence_PositiveSmallLambda}
    Let $N \geq 2$ and $\lambda \geq 0$. Let $\d_1,\d_2$ be defined by \eqref{eq:PerturbedInitialization_ScalarDynamics} with the perturbed identical initialization \eqref{eq:PerturbedInitialization_ScalarInitialization}. Assume that $(\alpha-\beta)\alpha^{N-1}>\lambda$. If
    \begin{equation}\label{C3:cond:stepsize}
       % \textcolor{blue}{
        0 < \eta <
        %\min\left\{\frac{1-2^{-\frac{1}{N}}}{2}, \frac{1}{9N}\right\} \frac{1}{\alpha^{2N-2}} = 
        \frac{1}{9N
        \alpha^{2N-2}}, 
        %\textcolor{blue}{\max\{\alpha,\lambda^{\frac{1}{N}}\}^{2N-2}}},
        %}
    \end{equation}
    then $0 < \d_1(\tdisc)\d_2(\tdisc)^{N-1}\leq (\alpha-\beta)\alpha^{N-1}$, $\d_2(\tdisc)\leq \alpha$ for all $\tdisc\in\mathbb{N}_0$, and $\lim_{\tdisc\to\infty}\d_1(\tdisc)\d_2(\tdisc)^{N-1}=\lambda$.
\end{lemma}
\begin{proof}
    We first consider the case $\lambda=0$. 
    \revision{We first show by induction that $0\leq\d_1(\tdisc)\leq\d_2(\tdisc)\leq\alpha$ for all $\tdisc\in\mathbb{N}_0$. From the initialization \eqref{eq:PerturbedInitialization_ScalarInitialization} we have $0<d_1(0) = \alpha- \beta < \alpha = d_2(0)$, which is the claim for $\tdisc=0$. Furthermore, if}
    $0\leq\d_1(\tdisc)\leq\d_2(\tdisc)\leq\alpha$ for some $\tdisc\in\mathbb{N}_0$, then
    \begin{align*}
        \d_1(\tdisc+1)
        &= \d_1(\tdisc)(1-\eta\d_2(\tdisc)^{2N-2})
        \geq \d_1(\tdisc)(1-\eta\alpha^{2N-2} )
        \geq 0\\
        \d_1(\tdisc+1)
        &= \d_1(\tdisc)(1-\eta\d_2(\tdisc)^{2N-2})
        \leq \d_2(\tdisc)(1-\eta\d_1(\tdisc)^2\d_2(\tdisc)^{2N-4})
        =\d_2(\tdisc+1)
        \leq \d_2(\tdisc).
    \end{align*}
    \revision{This concludes the induction argument.} Moreover, it also follows with the above relations that $\d_1(\tdisc+1) \leq \d_1(\tdisc)$. In particular, $\d_1$ and $\d_2$ are monotonically decreasing and bounded from below by $0$. Hence, $p(\tdisc) = \d_1(\tdisc) \d_2(\tdisc)^{N-1} \geq 0$ for all $\tdisc \in \mathbb{N}_0$ and $p$ forms a monotonically decreasing sequence. Hence, $p(\tdisc)$, $\d_1(\tdisc)$, $\d_2(\tdisc)$ converge as $\tdisc \to \infty$. The limits $p^* = \lim_{\tdisc \to \infty} p(\tdisc)$, $\d_1^* = \lim_{\tdisc \to \infty} \d_1(\tdisc)$, $\d_2^* = \lim_{\tdisc \to \infty} \d_2(\tdisc)$ satisfy the fixed point equation
    \revision{
    \begin{equation*}
        p^* = p^* (1-\eta (\d_2^*)^{2N-2})(1-\eta (\d_1^*)^2 (\d_2^*)^{2N-2}).
    \end{equation*}
    }
    Since $\d_1^*, \d_2^* \leq \alpha$ and $\eta \alpha^{2N-2} < 1$ by \eqref{C3:cond:stepsize}, the only solution to the fixed-point equation is $p^* = 0 = \lambda$, which proves the claim for $\lambda = 0$. %Hence, either $p^* = 0$ or $1 = f_{p^*}(\d_2^*) = (1 - \eta p^*$}
    %\textcolor{red}{Define $p=\d_1\d_2^{N-1}$. Note that since $\d_1,\d_2$ are both decreasing and non-negative, $p$ is strictly monotonically decreasing (unless $p=0$) and bounded below by the fixed point zero. By the monotone convergence theorem, $p$ converges to zero.}
    
    For $\lambda>0$, the proof strategy is essentially the same as in Lemma \ref{lemma:PerturbedInitialization_Convergence_PositiveLargeLambda}.
    Recall the sequences $b,p,p_b$ defined in \eqref{def:b_IdenticalInitialization} and \eqref{def:p_pa_pb}.  If $p(\tdisc)=\lambda$, then $p(\tdisc + 1) = p(\tdisc)$ and the claim trivially holds. Hence it suffices to consider $p(\tdisc)\neq\lambda$.
    
    %Unless otherwise specified, all the variables are evaluated at time $\tdisc$, i.e., $p=p(\tdisc)$.
    We will show that $\lambda\leq p(\tdisc) \leq p_b(\tdisc)$ and $0 < \d_1(\tdisc) < \d_2(\tdisc)\leq \alpha$ for all $\tdisc \in \mathbb{N}_0$ by induction. Since $p(0) = (\alpha-\beta)\alpha^{N-1} > \lambda$, $p_b(0) = \alpha^N > (\alpha-\beta)\alpha^{N-1} = p(0)$ and $0 < \alpha - \beta = \d_1(0) < \d_2(0) = \alpha$ by assumption, the claim holds for $\tdisc=0$.
    Assume that $\lambda \leq p(\tdisc) \leq p_b(\tdisc)$ and $0 < \d_1(\tdisc) < \d_2(\tdisc) \leq \alpha$ holds up to some $\tdisc\in\mathbb{N}_0$. For simplicity, we will often abbreviate $p=p(\tdisc)$, $\d_1=\d_1(\tdisc)$, $\d_2=\d_2(\tdisc)$ below. 
    %Since %the quantities 
    The induction hypotheses $0< \d_1(\tdisc) < \d_2(\tdisc) < \alpha$ and $p(\tdisc) \geq \lambda$ imply that
    \[
    0 \leq \eta \kappa(\tdisc) = \eta\d_2(\tdisc)^{N-2}(p(\tdisc) - \lambda) \leq \eta \alpha^{N-2}(\alpha^{N} - \lambda) \leq \eta \alpha^{2N-2} < 1,  
    \]
    using also \eqref{C3:cond:stepsize} in the last step. 
    %Hence, 
    %by \eqref{d1:kappa}
    %\[
    %\d_1(\tdisc+1) = \d_1(\tdisc)- \eta \d_2(\tdisc) \kappa(\tdisc)) \leq \d_1(\tdisc)
    %\]
    Hence, by \eqref{d2:kappa}
    \[
    \d_2(\tdisc+1) = \d_2(\tdisc) -\eta \d_1(\tdisc) \kappa(\tdisc) \leq \d_2(\tdisc) \leq \alpha \quad \mbox{ and } \quad \d_2(\tdisc+1) = \d_2(\tdisc) -\eta \d_1(\tdisc) \kappa(\tdisc) \geq \d_2(\tdisc)(1- \eta \kappa(\tdisc)) > 0.
    \]
    Lemma~\ref{lemma:DeltaContraction} implies that $\Delta_1(\tdisc+1) > \Delta_1(\tdisc)$ so that inductively $\d_2(\tdisc+1) - \d_1(\tdisc+1) = \Delta_1(\tdisc+1) > \Delta_1(0) = \beta > 0$, i.e., $\d_1(\tdisc+1) < \d_2(\tdisc+1)$. \revision{Further note that due to the induction hypothesis, which implies $\d_2(\tdisc)^N \ge p(\tdisc) \geq \lambda$, and our assumption \eqref{C3:cond:stepsize} on $\eta$, we have
    \begin{align*}
        |c_{p,1} d_2(\tdisc)^{2N-2}|
        &= \eta p(\tdisc)^{-1}(p(\tdisc)-\lambda)\d_2(\tdisc)^{2N-2}
        = \eta(1- \lambda/p(\tdisc)) \d_2(\tdisc)^{2N-2}
        \leq \eta \alpha^{2N-2}
        <1\\
        |c_{p,2} \d_2(\tdisc)^{-2}|
        &= p(\tdisc)^2|c_{p,1}\d_2(\tdisc)^{-2}|
        \leq |c_{p,1}\d_2(\tdisc)^{2N-2}|
        <1
    \end{align*}
    }
    % \[
    % 1+c_{p,1} d_2(\tdisc)^{2N-2} = 1 - \eta p(\tdisc)^{-1}(p(\tdisc)-\lambda)\d_2(\tdisc)^{2N-2} = 1 - \eta(1- \lambda/p(\tdisc)) \d_2(\tdisc)^{2N-2} \geq 1 - \eta \alpha^{2N-2} > 0  
    % \]
    Hence, Lemma~\ref{lem:fp} gives
    \[
    p(\tdisc+1) = p(\tdisc) f_p(\d_2(\tdisc)) = p(\tdisc)(1+ c_{p,1}\d_2(\tdisc)^{2N-2})\revision{(1+c_{p,2} \d_2(\tdisc)^{-2})^{N-1}} > 0.
    \]
    Since also $\d_2(\tdisc+1) > 0$, $p(\tdisc+1) = \d_1(\tdisc+1) \d_2(\tdisc + 1)^{N-1}$ implies that $\d_1(\tdisc+1) > 0$.
    %Since $d_1(0) < d_2(0) \le \alpha$ and $\d_2$ is decreasing, Lemma \ref{lemma:DeltaContraction} yields that $\d_1 < \d_2$ \textcolor{red}{because $0>\eta\kappa>\eta\alpha^{2N-2}>-1$.} 
    %In particular, $\d_2(\tdisc+1)<\d_2(\tdisc)\leq \alpha$. 
    Using $c_{p,1} < 0$ another time in combination with  $c_{p,2}=-\eta p(\tdisc)(p(\tdisc) - \lambda)<0$, 
    Lemma~\ref{lem:fp} implies that $f_p$ is concave 
    on $[p^{\frac{1}{N}},\infty)$, and $f_p'(p^{\frac{1}{N}})=0$. Hence, $f_p$ is monotonically decreasing on $[p^{\frac{1}{N}},\infty)$.
    Together with the induction hypothesis $0<\d_1(\tdisc) < \d_2(\tdisc) \leq \alpha$ this gives
    \begin{equation}\label{eq:fpCompare2}
        f_{p}(\alpha)
        \leq f_{p}(\d_2)
        \leq f_{p}(p^{\frac{1}{N}}).
    \end{equation}
    We will use this to prove that $\lambda\leq p(\tdisc+1)\leq p_b(\tdisc+1)$. Equation \eqref{eq:fpCompare2} implies that
    \begin{align*}
        p(\tdisc+1)
        &= pf_{p}(\d_2)
        \leq pf_{p}(p^{\frac{1}{N}})
        = p(1 - \eta p^{-1}(p-\lambda)p^{2-\frac{2}{N}}) (1 - \eta p(p-\lambda)p^{-\frac{2}{N}})^{N-1}\\
        &= p(1 - \eta p^{1-\frac{2}{N}}(p-\lambda))^{N}.
        % p(\tdisc+1)
        % &= pf_{p}(\d_2)
        % \geq pf_{p}(\alpha)\\
        % &= p(1 - \eta p^{-1}(p-\lambda)\alpha^{2N-2}) (1 - \eta p(p-\lambda)\alpha^{-2})^{N-1}\\
        % &\geq (p - \eta(p-\lambda) \alpha^{2N-2}) (1 - \eta (p-\lambda)\alpha^{N-2})^{N-1}.
    \end{align*}
    We further note that by Lemma~\ref{lemma:IdenticalInitialization_Convergence} in combination with \eqref{C3:cond:stepsize} (noting that $\alpha^N > \lambda$ by assumption on $\lambda$) the sequence $p_b(\tdisc) = b(\tdisc)^{N}$ satisfies $p_b(\tdisc) \leq \alpha^N$. 
    By a similar calculation as in the proof of Lemma \ref{lemma:PerturbedInitialization_Convergence_PositiveLargeLambda}, we obtain, using the induction hypothesis $p(\tdisc) \leq p_b(\tdisc)$, %b(\tdisc)^N \leq b(0)^N = \alpha^N$, where the last inequality follows from Lemma \ref{}} 
    \begin{align*}
        p_b(\tdisc+1)^{\frac{1}{N}} - p(\tdisc+1)^{\frac{1}{N}}
        &\geq p_b^{\frac{1}{N}}(1-\eta p_b^{1-\frac{2}{N}}(p_b-\lambda)) - p^\frac{1}{N}(1 - \eta p^{1-\frac{2}{N}}(p-\lambda))\\
        &= p_b^\frac{1}{N} - p^{\frac{1}{N}} + \eta\lambda (p_b^{1-\frac{1}{N}}- p^{1-\frac{1}{N}}) - \eta(p_b^{2-\frac{1}{N}}- p^{2-\frac{1}{N}})\\
        &\geq p_b^\frac{1}{N} - p^{\frac{1}{N}} - \eta(p_b^{2-\frac{1}{N}}- p^{2-\frac{1}{N}})= (p_b^\frac{1}{N} - p^{\frac{1}{N}})(1-\eta\sum_{\ell=0}^{2N-2}p^{\frac{\ell}{N}}p_b^{\frac{2N-2-\ell}{N}})\\
        &\geq (p_b^\frac{1}{N} - p^{\frac{1}{N}})(1-2N\eta \alpha^{2-\frac{2}{N}})
        \geq 0
    \end{align*}
    since $\eta\leq \frac{1}{2N \alpha^{2N-2}}$ by \eqref{C3:cond:stepsize}. Hence $p_b(\tdisc+1) \geq p(\tdisc+1)$.
    
    Next we show that $p(\tdisc+1)\geq \lambda$. Similarly to the proof of  Lemma~\ref{lemma:PerturbedInitialization_Convergence_PositiveLargeLambda}, we distinguish two cases: either $p(\tdisc)>2\lambda$ or $p(\tdisc)\leq 2\lambda$. Suppose first that $p(\tdisc) >2\lambda$. Another application of  \eqref{eq:fpCompare2} together with $p(\tdisc) \leq \alpha^N$ yields
    \begin{align*}
        p(\tdisc+1) \geq pf_p(\alpha)
        &= p\left(1 - \eta p^{-1}(p-\lambda)\alpha^{2N-2}\right) \left(1 - \eta p(p-\lambda)\alpha^{-2}\right)^{N-1} \\
        &\geq p\left(1 - \eta\alpha^{2N-2}\right) \left(1 - \eta p^2 \alpha^{-2}\right)^{N-1} \\
        &\geq 2\lambda \left(1 - \eta \alpha^{2N-2}\right)^N 
        \geq \lambda
    \end{align*}
    since $\eta \leq \frac{1-2^{-\frac{1}{N}}}{2\alpha^{2N-2}}$, where the latter is implied by \eqref{C3:cond:stepsize} with the fact that $1-2^{-\frac{1}{N}} > \frac{2}{9N}$ for $N \geq 2$, which follows from an elementary analysis.
    Now suppose $\lambda\leq p(\tdisc)\leq 2\lambda$. Using $\lambda < \alpha^N$ we obtain
    \begin{align*}
        \log(pf_p(\alpha))
        &= \log(p) +\log\left(1 - \eta p^{-1}(p-\lambda)\alpha^{2N-2}\right) +(N-1)\log \left(1 - \eta p(p-\lambda)\alpha^{-2}\right)\\
        &\geq \log(p) +\log\left(1 - \eta \lambda^{-1}\alpha^{2N-2}(p-\lambda)\right) +(N-1)\log\left(1 - 2\eta \lambda\alpha^{-2}(p-\lambda)\right)\\
        &\geq \log(p) + N\log\left(1 - 2 \eta \lambda^{-1}\alpha^{2N-2}(p-\lambda)\right).
    \end{align*}
    The inequalities $\log(1-x)\geq \frac{-x}{1-x}$ and $\log(1+x)\geq x-\frac{1}{2}x^2$, valid for $x \in (0,1)$, then lead to 
    \begin{align*}
        \log(p(\tdisc+1))-\log(\lambda)
        &=\log(pf_p(\alpha)) - \log (\lambda)
        \geq \log(p) - \log (\lambda) + N\log\left(1 - 2 \eta \lambda^{-1}\alpha^{2N-2}(p-\lambda)\right)\\
        &= \log\left(1+\frac{p-\lambda}{p}\right) + N\log\left(1 - 2 \eta \lambda^{-1}\alpha^{2N-2}(p-\lambda)\right)\\
        &\geq \frac{p-\lambda}{p} - \frac{1}{2}\left(\frac{p-\lambda}{p}\right)^2 - N\frac{2 \eta \lambda^{-1}\alpha^{2N-2}(p-\lambda)}{1 - 2 \eta \lambda^{-1}\alpha^{2N-2}(p-\lambda)}\\
        &= \frac{p-\lambda}{p^2} \left(p- \frac{1}{2}(p-\lambda) - N\frac{2 \eta p^2 \lambda^{-1}\alpha^{2N-2}}{1 - 2 \eta \lambda^{-1}\alpha^{2N-2}(p-\lambda)} \right)\\
        &\geq \frac{p-\lambda}{p^2} \left( \lambda - N\frac{8 \eta \lambda\alpha^{2N-2}}{1 - 2 \eta \alpha^{2N-2}} \right)
        \geq 0
    \end{align*}
    %if $\eta \leq \frac{1}{2\alpha^{2N-2}(1+4N)}$, which is implied by the condition $\eta \leq \frac{1}{8N\alpha^{2N-2}}$.
    since $\eta \leq \frac{1}{9N \alpha^{2N-2}}$ and $N\geq 2$.
    Thus $p(\tdisc+1)\leq\lambda$. 
    
    In particular, we derived $\lambda\leq p(\tdisc) \leq p_b(\tdisc)$ for all $\tdisc \in \mathbb{N}_0$ and since  $\lim_{\tdisc\to\infty} p_b(\tdisc) = \lambda$ by Lemma~\ref{lemma:IdenticalInitialization_Convergence}, we conclude that $\lim_{\tdisc\to\infty} p(\tdisc) = \lambda$. The proof is thus complete.
\end{proof}
\section{Simplified expression for convergence time}

%In order to provide a better intuitive understanding of the dynamics, let us simplify the expression for $T^\Id_N$ in the relevant regime that $N \geq 3$, $\lambda_j > 0$ for at least one $j$, $\alpha, \epsilon \ll \lambda_k$ for some $k \in [n]$, and
%$\eta = \frac{\kappa}{N \lambda_1^{2-\frac{2}{N}}}$ for some $\kappa < 1/3$ so that the condition \eqref{} on the stepsize is satisfied.

Since especially the exact term for $T_N^+$ defined in \eqref{def:T+-} and appearing in the bounds for the convergence times is hard to interpret, we give a simplified approximate expression in the following lemma.

\begin{lemma}\label{lem:approximate:TId} 
    Let $N \geq 3$ and $k \in [n]$ such that $\lambda_1 \geq \lambda_k > 0$. Assume that $\alpha^N < \lambda_k$, 
    and 
    \[
    \eta = \frac{\kappa}{N \lambda_1^{2-\frac{2}{N}}}
    \]
    for some $\kappa < \frac{1}{3}$ (so that \eqref{eta:cond:mthm} is satisfied). 
    %Set $\xi = \alpha \lambda_k^{-\frac{1}{N}} < 1$. 
    Then
    \[
    \frac{1}{\eta} T_N^+(\lambda_k,(c_N \lambda_k)^{\frac{1}{N}}, \alpha) = \kappa^{-1}\left(\frac{\lambda_1}{\lambda_k}\right)^{2-\frac{2}{N}}\left( \frac{1}{N-2}\left( \frac{\lambda_k}{\alpha^N}\right)^{1-\frac{2}{N}} + C_N - \frac{N}{2} \left(\frac{\alpha^N}{\lambda_k}\right)^{\frac{2}{N}} + G\left(\frac{\alpha}{\lambda_k^{\frac{1}{N}}}  \right) \right), 
    \]
    where $G$ is a function satisfying $|G(t)| \leq \frac{Nt^3}{3(1-t)^3}$ for $t \in [0,1)$ and 
    \begin{align*}
    C_N & = \sum_{\ell=1}^N \Re\left(e^{\frac{4\pi i \ell}{N}}\left( \ln \left(e^{\frac{2\pi i \ell}{N}} - c_N^{\frac{1}{N}}\right)\right)\right)
     - Q_N - \frac{1}{(N-2) c_N^{1-\frac{2}{N}}} \\
     \mbox{ with } \quad Q_N & = \begin{cases}
         \pi \cot\left(\frac{2\pi}{N}\right) &  \mbox{ for } N \mbox{ even},\\
         \frac{\pi}{\sin\left(\frac{2\pi}{N}\right)} & \mbox{ for } N \mbox{ odd. }  
     \end{cases}
    \end{align*}
    %{\textcolor{blue}{Give some values for $C_N$: Here $C_3=..$, $C_4 = ..$, $C_5 = ...$. But perhaps not necessary.}}
\end{lemma}
\begin{proof}
    Recalling the definition of $T_N^+$ in \eqref{def:T+-} we obtain
    \begin{align}
    \frac{1}{\eta} & T_N^+(\lambda_k,(c_N \lambda_k)^{\frac{1}{N}}, \alpha)  = \frac{N \lambda_1^{2-\frac{2}{N}}}{\kappa} \frac{\lambda_k^{\frac{2}{N}-2}}{N} \left( \sum_{\ell=1}^N \Re\left(e^{\frac{4\pi i \ell}{N}}\left( \ln \left(e^{\frac{2\pi i \ell}{N}} - c_N^{\frac{1}{N}}\right) - \ln\left(e^{\frac{2\pi i\ell}{N}} - \frac{\alpha}{\lambda_{k}^{\frac{1}{N}}}\right) \right) \right)\right) \notag\\
    & \quad - \frac{N \lambda_1^{2-\frac{2}{N}}}{\kappa}  \left(\frac{1}{\lambda_k (N-2) (c_N \lambda_k^{1/N})^{N-2}} - \frac{1}{\lambda_k (N-2) \alpha^{N-2}}\right)\notag \\
    & = \kappa^{-1} \left( \frac{\lambda_1}{\lambda_k} \right)^{2-\frac{2}{N}}\left( \sum_{\ell=1}^N \Re\left(e^{\frac{4\pi i \ell}{N}}\left( \ln \left(e^{\frac{2\pi i \ell}{N}} - c_N^{\frac{1}{N}}\right)\right)\right)
     - \frac{1}{(N-2) c_N^{1-\frac{2}{N}}}
     \right. \notag\\
    & \phantom{\kappa^{-1} \left( \frac{\lambda_1}{\lambda_k}\right)^{2-\frac{2}{N}}()()}\left. - \sum_{\ell=1}^N \Re\left(e^{\frac{4\pi i \ell}{N}}\left( \ln \left(e^{\frac{2\pi i \ell}{N}} -\frac{\alpha}{\lambda_k^{\frac{1}{N}}}\right)\right) \right)
     + \frac{1}{N-2}\left(\frac{\lambda_k^{\frac{1}{N}}}{\alpha}\right)^{N-2} \right) \notag\\
     & = \kappa^{-1} \left( \frac{\lambda_1}{\lambda_k} \right)^{2-\frac{2}{N}}\left( D_N - \sum_{\ell=1}^N \Re\left(e^{\frac{4\pi i \ell}{N}} \ln \left(e^{\frac{2\pi i \ell}{N}} -\frac{\alpha}{\lambda_k^{\frac{1}{N}}}\right) \right)
     + \frac{1}{N-2}\left(\frac{\lambda_k^{\frac{1}{N}}}{\alpha}\right)^{N-2} \right), \notag
    \end{align}
    where $D_N = \sum_{\ell=1}^N \Re\left(e^{\frac{4\pi i \ell}{N}}\left( \ln \left(e^{\frac{2\pi i \ell}{N}} - c_N^{\frac{1}{N}}\right)\right)\right)
     - \frac{1}{(N-2) c_N^{1-\frac{2}{N}}}$.
    Applying Taylor's theorem to the function $h_\ell : \R \to \mathbb{C}$, $h_\ell(t) = \ln\left(e^{\frac{2\pi i \ell}{N}}- t\right)$ yields, for some $t \in [0,1)$ and some $\zeta \in (0,t)$,
    \[
    h_\ell(t) = h_\ell(0) - t h_\ell'(0) + \frac{t^2}{2} h_\ell''(0) - \frac{t^3}{3!} h_\ell'''(\zeta)
    = \ln\left(e^{\frac{2\pi i\ell}{N}}\right) - t \frac{1}{e^{2\pi i \ell/N}} + \frac{t^2}{2} \frac{1}{e^{4\pi i\ell/N}} - g_\ell(t),
    \]
    where $g_\ell$ satisfies
    \[
    |g_\ell(t)| \leq \frac{t^3}{3!}\frac{2}{|e^{2 \pi i \ell/N} - \zeta|^3} \leq \frac{t^3}{3(1-t)^3} = \frac{1}{3} \left(\frac{t}{1-t} \right)^3. 
    \]
    With $\xi = \alpha/\lambda_k^{\frac{1}{N}}$ this gives
    \begin{align}
    \sum_{\ell=1}^N e^{\frac{4\pi i \ell}{N}}\ln\left(e^{\frac{2\pi i \ell}{N}} - \xi \right) & = 
    \sum_{\ell=1}^N e^{\frac{4\pi i \ell}{N}} \ln(e^{\frac{2 \pi i \ell}{N}}) - \xi \sum_{\ell=1}^N e^{\frac{2\pi i \ell}{N}} + \frac{\xi^2}{2} N - G(\xi)\notag \\
    & =  \frac{1}{N}\sum_{1 \leq \ell \leq N/2} e^{\frac{4\pi i \ell}{N}}(2 \pi i \ell) + \frac{1}{N} \sum_{N/2 < \ell \leq N} e^{\frac{4\pi i \ell}{N}}(2 \pi i (\ell-N)) + \frac{N}{2} \xi^2 + G(\xi)\notag\\
    & = \frac{2\pi i}{N}\left( \sum_{1 \leq \ell \leq N/2} \ell e^{\frac{4\pi i \ell}{N}} - \sum_{0 \leq j < N/2} j e^{-\frac{4 \pi i j}{N}} \right) + \frac{N}{2} \xi^2 - G(\xi),
    \label{Taylor-hl-sum}
    \end{align}
    where $G$ is a function satisfying $|G(\xi)| \leq \frac{N}{3} \frac{\xi^3}{(1-\xi)^3}$. Hereby, we have taken into account that the principal branch of the complex logarithm is used. 
    Some analysis using the geometric sum formula shows that $\sum_{j=1}^n j z^j = z(1-(n+1)z^n + n z^{n+1})/(1-z)^2$ for $z \in \mathbb{C} \setminus \{1\}$. %Using this for the two sums above yields after some algebra
    %\[
    %\sum_{1 \leq \ell \leq N/2} \ell e^{\frac{4\pi i \ell}{N}} - \sum_{0 \leq j < N/2} j e^{-\frac{4 \pi i j}{N}}  = 2 N e^{2\pi i/N}\sin\left(\frac{2 \pi}{N}\right) \cos\left(\frac{2 \pi}{N}\right).
    %\quad \mbox{\textcolor{blue}{DOUBLE CHECK!}}
    %\]
    %\textcolor{red}{Below is my calculation, which result in different answer. First consider even $N$.}
    For $N$ even this gives
    \begin{align*}
        & \sum_{1 \leq \ell \leq N/2} \ell e^{\frac{4\pi i \ell}{N}} - \sum_{0 \leq j < N/2} j e^{-\frac{4 \pi i j}{N}}
        = \sum_{\ell = 1}^{N/2} \ell e^{\frac{4\pi i \ell}{N}} - \sum_{j = 1}^{N/2} j e^{-\frac{4 \pi i j}{N}} + \frac{N}{2}\\
        &= \frac{e^{\frac{4\pi i}{N}}}{(1- e^{\frac{4\pi i}{N}})^2}\left(1 - (\frac{N}{2}+1) + \frac{N}{2}e^{\frac{4\pi i}{N}}\right)- \frac{e^{-\frac{4\pi i}{N}}}{(1- e^{-\frac{4\pi i}{N}})^2}\left(1 - (\frac{N}{2}+1) + \frac{N}{2}e^{-\frac{4\pi i}{N}}\right) + \frac{N}{2}\\
        &= %\frac{e^{\frac{4\pi i}{N}}}{(1- e^{\frac{4\pi i}{N}})^2}
        \frac{1}{\left(e^{\frac{2\pi i}{N}} - e^{-\frac{2\pi i}{N}}\right)^2} \left(\frac{N}{2}e^{\frac{4\pi i}{N}} - \frac{N}{2}e^{-\frac{4\pi i}{N}}\right) + \frac{N}{2}
        %= \frac{N}{2} \left(\frac{e^{\frac{4\pi i}{N}}- e^{-\frac{4\pi i}{N}}}{(e^{-\frac{2\pi i}{N}}- e^{\frac{2\pi i}{N}})^2 }+ 1\right)\\
        = \frac{N}{2}\left(\frac{\sin(4\pi/N)}{2i\sin(2\pi/N)^2}+1\right)\\ 
        &= \frac{N}{2}
        \left( 1 - i \cot\left(\frac{2\pi}{N}\right)\right).
      %\left(\frac{-i\cos(2\pi/N)}{\sin(2\pi/N)}+1\right)\\
        %&= \frac{N}{2}\frac{\sin(2\pi/N)- i\cos(2\pi/N)}{\sin(2\pi/N)}
    \end{align*}
    For $N$ being odd a similar computation gives 
    %\textcolor{red}{Now consider odd $N$.}
    \begin{align*}
        &\sum_{1 \leq \ell \leq N/2} \ell e^{\frac{4\pi i \ell}{N}} - \sum_{0 \leq j < N/2} j e^{-\frac{4 \pi i j}{N}}
        = \sum_{\ell=1}^{\frac{N-1}{2}} \ell e^{\frac{4\pi i \ell}{N}} - \sum_{j=1}^{\frac{N-1}{2}} j e^{-\frac{4 \pi i j}{N}} \\
        &= \frac{e^{\frac{4\pi i}{N}}}{(1- e^{\frac{4\pi i}{N}})^2}\left(1 - \left(\frac{N-1}{2}+1\right)e^{\frac{4\pi i (N-1)}{2N}} + \frac{N-1}{2}e^{\frac{4\pi i(N+1)}{2N}}\right) \\
        & - \frac{e^{\frac{-4\pi i}{N}}}{(1- e^{\frac{-4\pi i}{N}})^2}\left(1 - \left(\frac{N-1}{2}+1\right)e^{-\frac{4\pi(N-1) i}{2N}} + \frac{N}{2}e^{\frac{-4\pi i (N+1) }{2N}}\right)\\
        &= \frac{1}{\left(e^{\frac{2\pi i}{N}}-e^{-\frac{2\pi i}{N}}\right)^2}\left( - \frac{N+1}{2}\left(e^{-\frac{2\pi i}{N}} - e^{\frac{2 \pi i}{N}}\right) + \frac{N-1}{2}\left(e^{\frac{2 \pi i}{N}} - e^{-\frac{2\pi i}{N}}\right)\right)  \\
       & = \frac{N \left(e^{\frac{2\pi i}{N}} - e^{-\frac{2\pi i}{N}}\right)}{\left(e^{\frac{2\pi i}{N}}-e^{-\frac{2\pi i}{N}}\right)^2} = \frac{i N}{2 \sin(2\pi/N)}.
       %\frac{(N+1)e^{\frac{4\pi i}{N}}}{(1- e^{\frac{4\pi i}{N}})^2}(e^{\frac{2\pi i}{N}} - e^{-\frac{2\pi i}{N}}) \\
        %&= (N+1)\frac{e^{\frac{2\pi i}{N}} - e^{-\frac{2\pi i}{N}}}{(e^{-\frac{2\pi i}{N}}- e^{\frac{2\pi i}{N}})^2 }
        %= \frac{(N+1)}{e^{\frac{2\pi i}{N}}- e^{-\frac{2\pi i}{N}} }\\
        %&= \frac{(N+1)}{2i\sin(2\pi/N)}.
    \end{align*}
    %{\textcolor{blue}{DOUBLE CHECK!}}
    %Noting that $\Re(i e^{2\pi i/N}) = - \sin(2\pi/N)$, this gives
    %Denoting $Q_N = \pi  \cot(2\pi/N)$ for $N$ even and $Q_N = \frac{\pi}{\sin(2\pi/N)}$ for $N$ odd, and 
    Plugging the above computations into \eqref{Taylor-hl-sum} and using the definition of $Q_N$ gives
    \begin{align*}
    \sum_{\ell=1}^N \Re\left(e^{\frac{4\pi i \ell}{N}}\ln\left(e^{\frac{2\pi i \ell}{N}} - \xi \right)\right) & = 
    %-4 \pi \sin^2\left(\frac{2\pi}{N}\right) \cos\left(\frac{2\pi}{N}\right) + 
    Q_N + \frac{N}{2} \xi^2 - G(\xi).
    \end{align*}
    Setting $C_N := D_N - Q_N$ we obtain
    \begin{align*}
    \frac{1}{\eta}  T_N^+(\lambda_k,(c_N \lambda_k)^{\frac{1}{N}}, \alpha) =  \kappa^{-1} \left(\frac{\lambda_1}{\lambda_k} \right)^{2 - \frac{2}{N}}\left( \frac{1}{(N-2) \xi^{N-2}} + C_N - \frac{N}{2} \xi^2 + G(\xi) \right).
    \end{align*}
    This completes the proof.
\end{proof}

\end{document}